%% file: SyncQlearning.tex
\documentclass[10pt,english]{article}
\usepackage[T1]{fontenc}
\usepackage[latin9]{inputenc}
\usepackage{geometry}
\usepackage[active]{srcltx}
\usepackage{babel}
\usepackage{verbatim}
\usepackage{mathtools}
\usepackage{bm}
\usepackage{amsmath}
\usepackage{amssymb}
\usepackage[unicode=true,
 bookmarks=false,
 breaklinks=false,pdfborder={0 0 1},colorlinks=false]
 {hyperref}
\hypersetup{
 colorlinks,linkcolor=blue,anchorcolor=blue,citecolor=blue}

\makeatletter

\usepackage{microtype,comment}
\usepackage{graphicx}
\usepackage{subfigure}
\usepackage{booktabs} 
\usepackage{hyperref}
\usepackage{verbatim}
\usepackage{mathtools}
\usepackage{bm}
\usepackage{amsmath}
\usepackage{amssymb} 
\usepackage{amsthm}
\usepackage{babel}
\usepackage{cite}
\usepackage{amsthm}
\usepackage{dsfont}
\usepackage{array}
\usepackage{mathrsfs}
\usepackage{color}

\allowdisplaybreaks

\usepackage{natbib}
\usepackage{xfrac}
\usepackage{enumitem}

\usepackage{babel}
\usepackage{algorithm}
\usepackage{algorithmic}
\usepackage{arydshln}

\theoremstyle{plain} 
\newtheorem{lemma}{\textbf{Lemma}} 

\newtheorem{theorem}{\textbf{Theorem}}
 
\newtheorem{assumption}{\textbf{Assumption}}

\theoremstyle{definition}
\theoremstyle{remark}

\newtheorem{remark}{\textbf{Remark}}

\newcommand{\pib}{\pi_{\mathsf{b}}}
\newcommand{\tmix}{t_{\mathsf{mix}}}

\newcommand{\cS}{\mathcal{S}}
\newcommand{\cA}{\mathcal{A}}

\newcommand{\mymid}{\,|\,}

\newcommand{\ceta}{c_\eta}

\newcommand{\mumin}{\mu_{\mathsf{min}}}

\newcommand{\defn}{\coloneqq}

\usepackage{color}
\definecolor{yxc}{RGB}{255,0,0}
\definecolor{yjc}{RGB}{0,155,55}
\definecolor{ytw}{RGB}{255,69,0}
\definecolor{gen}{RGB}{0,0,200}
\definecolor{cxc}{RGB}{0,140,0}
\definecolor{red}{RGB}{255,0,0}

\begin{document}

\title{Is Q-Learning Minimax Optimal? \\ A Tight Sample Complexity Analysis} 

\author{Gen Li\thanks{Department of Statistics and Data Science,  Wharton School, University of Pennsylvania, Philadelphia, PA 19104, USA.} \\
UPenn   \\
	\and
	Changxiao Cai\thanks{Department of Biostatistics, University of Pennsylvania, Philadelphia, PA 19104, USA.} \\
 UPenn \\
 	\and
	Yuxin Chen\footnotemark[1] \\
 UPenn \\
  	\and 
	Yuting Wei\footnotemark[1]\\
	UPenn  \\
	\and
	Yuejie Chi\thanks{Department of Electrical and Computer Engineering, Carnegie Mellon University, Pittsburgh, PA 15213, USA.}\\
	CMU  \\
	}
	
\date{February 2021; ~~ Revised: October 2022}	

\maketitle

\begin{abstract}
Q-learning, which seeks to learn the optimal Q-function of a Markov decision process (MDP) in a model-free fashion, lies at the heart of reinforcement learning. When it comes to the synchronous setting (such that independent samples for all state-action pairs are drawn from a generative model in each iteration), substantial progress has been made towards understanding the sample efficiency of Q-learning. Consider a $\gamma$-discounted infinite-horizon MDP with state space $\cS$ and action space $\cA$: to yield an entrywise $\varepsilon$-approximation of the optimal Q-function,  state-of-the-art theory for Q-learning requires a sample size exceeding
	the order of $\frac{|\mathcal{S}||\mathcal{A}|}{(1-\gamma)^5\varepsilon^{2}}$, which fails to match existing minimax lower bounds. This gives rise to natural questions: what is the sharp sample complexity of Q-learning? Is Q-learning provably sub-optimal? This paper addresses these questions for the synchronous setting: (1) when $|\cA|=1$ (so that Q-learning reduces to TD learning), we prove that the sample complexity of TD learning is minimax optimal and scales as $\frac{|\mathcal{S}|}{(1-\gamma)^3\varepsilon^2}$ (up to log factor); (2) when $|\cA|\geq 2$, we settle the sample complexity of Q-learning to be on the order of $\frac{|\mathcal{S}||\mathcal{A}|}{(1-\gamma)^4\varepsilon^2}$ (up to log factor). Our theory unveils the strict sub-optimality of Q-learning when $|\cA|\geq 2$, and rigorizes the negative impact of over-estimation in Q-learning. Finally, we extend our analysis to accommodate asynchronous Q-learning (i.e., the case with Markovian samples), sharpening the horizon dependency of its sample complexity to be $\frac{1}{(1-\gamma)^4}$. 
\end{abstract}




\medskip

\noindent\textbf{Keywords: }Q-learning, temporal difference learning, effective horizon, sample
complexity, minimax optimality, lower bound, over-estimation 

\setcounter{tocdepth}{2}
\tableofcontents{}

\input{introduction.tex}

\input{backgrounds.tex}

\input{results.tex}

\input{analysis.tex}

\input{AsynQ.tex}

\input{conclusions.tex}

\section*{Acknowledgements}

Y.~Chen is supported in part by the Alfred P.~Sloan Research Fellowship, the Google Research Scholar Award, the AFOSR grant FA9550-22-1-0198, 
the ONR grant N00014-22-1-2354,  
and the NSF grants CCF-2221009, CCF-1907661, DMS-2014279, IIS-2218713 and IIS-2218773. 
Y.~Wei is supported in part by the the NSF grants CCF-2106778, DMS-2147546/2015447 and  CAREER award DMS-2143215. 
Y.~Chi is supported in part by the grants ONR N00014-18-1-2142 and N00014-19-1-2404, 
the NSF grants CCF-1806154, CCF-2007911, CCF-2106778, ECCS-2126634, and DMS-2134080. 
The authors are grateful to Laixi Shi for helpful discussions about the lower bound, and thank Shaocong Ma for pointing out some errors in an early version of this work.
Part of this work was done while G.~Li, Y.~Chen and Y.~Wei were visiting the Simons Institute for the Theory of Computing.

%

\appendix

\input{preliminary-Freedman.tex}

\input{infinite-synchronous-Q.tex}

\input{TD-learning.tex}

\input{lowerbound.tex}

\input{AsynQ_analysis.tex}

\input{lower_asynQ.tex}

\bibliographystyle{apalike}
\bibliography{bibfileRL}

\end{document}

%% file: introduction.tex
\section{Introduction}

Q-learning is arguably one of the most widely adopted model-free algorithms \citep{watkins1989learning,watkins1992q}. 
Characterizing its sample efficiency  lies at the core of the statistical foundation of reinforcement learning (RL) \citep{sutton2018reinforcement}.
While classical convergence analyses for Q-learning \citep{tsitsiklis1994asynchronous,szepesvari1998asymptotic,jaakkola1994convergence,borkar2000ode} focused primarily on the asymptotic regime---in which the number of iterations tends to infinity with other problem parameters held fixed---recent years have witnessed a paradigm shift from asymptotic analyses towards a finite-sample\,/\,finite-time framework \citep{even2003learning,kearns1999finite,beck2012error,chen2020finite,wainwright2019stochastic,lee2018stochastic,qu2020finite,xiong2020finitetime,chen2021lyapunov,li2020sample,weng2020momentum}.  Drawing insights from high-dimensional statistics \citep{wainwright2019high},  a modern non-asymptotic framework unveils more clear and informative impacts of salient problem parameters upon the sample complexity, particularly for those applications with enormous state/action space and long horizon. 
Motivated by its practical value, a suite of non-asymptotic theory has been recently developed for Q-learning to accommodate multiple sampling mechanisms \citep{even2003learning,beck2012error,jin2018q,wainwright2019stochastic,qu2020finite,li2020sample}.

%
%
%

In this paper, we revisit the sample complexity of Q-learning for tabular Markov decision processes (MDPs).
For concreteness, let us consider the synchronous setting, which assumes access to a generative model or a simulator that produces independent samples for all state-action pairs  in each iteration \citep{kearns2002sparse,kakade2003sample}; 
this setting is termed ``synchronous'' as the estimates w.r.t.~all state-action pairs are updated at once. We investigate the $\ell_{\infty}$-based sample complexity, namely, the number of samples needed for synchronous Q-learning to yield an entrywise $\varepsilon$-accurate estimate of the optimal Q-function. Despite a number of prior works tackling this setting, the dependence of the sample complexity on the effective horizon $\frac{1}{1-\gamma}$ remains unsettled. Take $\gamma$-discounted infinite-horizon MDPs for instance: the state-of-the-art sample complexity bounds \citep{wainwright2019stochastic,chen2020finite} scale on the order of $ \frac{|\cS||\cA|}{(1-\gamma)^5\varepsilon^{2}} $ (up to some log factor), where $\cS$ and $\cA$ represent the state space and the action space, respectively. However, it is unclear whether this scaling is sharp for Q-learning,  and whether it can be further improved via a more refined theory. On the one hand, the minimax lower limit for this setting has been shown to be on the order of $ \frac{|\cS||\cA|}{(1-\gamma)^3\varepsilon^{2}} $ (up to some log factor) \citep{azar2013minimax}; this limit is achievable by model-based approaches \citep{agarwal2019optimality,li2020breaking} and apparently smaller than prior sample complexity bounds for Q-learning. On the other hand, \citet{wainwright2019variance} argued through numerical experiments that ``{\em the usual Q-learning suffers from at least worst-case fourth-order scaling in the discount complexity $\frac{1}{1-\gamma}$, as opposed
to the third-order scaling $\ldots$}'', although no rigorous justification was provided therein.   
Given the gap between the achievability bounds and lower bounds in the status quo, it is natural to seek answers to the following questions:   
\begin{center}
{\em What is the tight sample complexity characterization of Q-learning?} \\
{\em How does it compare to the minimax sample complexity limit? }
\end{center}
%



%
%
%
%
%

\newcommand{\topsepremove}{\aboverulesep = 0mm \belowrulesep = 0mm} \topsepremove

\begin{table*}[t]
\centering
\begin{tabular}{c|c|c}
\toprule 
paper $\vphantom{\frac{1^{7}}{1^{7^3}}}$ &  learning rates & sample complexity   \\ \toprule
	\cite{even2003learning} $\vphantom{\frac{1^{7^{7^7}}}{1^{7^{7^{7^7}}}}}$ & linear: $\frac{1}{t}$ & $ 2^{\frac{1}{1-\gamma}}\frac{|\cS||\cA|}{(1-\gamma)^4\varepsilon^2}$    \\ \hline
	\cite{even2003learning} $\vphantom{\frac{1^{7^{7^7}}}{1^{7^{7^{7^7}}}}}$  & polynomial: $\frac{1}{t^{\omega}}$,\, $\omega \in (1/2,1)$ & $|\cS||\cA| \Big\{ \big( \frac{1}{(1-\gamma)^4\varepsilon^2}\big)^{1/\omega} +  \big( \frac{1}{1-\gamma }\big)^{\frac{1}{1-\omega}} \Big\}$   \\ \hline
\cite{beck2012error} $\vphantom{\frac{1^{7^{7^7}}}{1^{7^{7^{7^7}}}}}$  & constant: $\frac{(1-\gamma)^4\varepsilon^2}{|\cS||\cA|}$ & $\frac{|\cS|^2|\cA|^2}{(1-\gamma)^5\varepsilon^2} $ \\ \hline
\cite{wainwright2019stochastic} $\vphantom{\frac{1^{7^{7^7}}}{1^{7^{7^{7^7}}}}}$ & rescaled linear: $\frac{1}{1+(1-\gamma)t}$ & $\frac{|\cS||\cA|}{(1-\gamma)^5\varepsilon^2} $ \\ \hline
\cite{wainwright2019stochastic} $\vphantom{\frac{1^{7^{7^7}}}{1^{7^{7^{7^7}}}}}$  & polynomial:  $\frac{1}{t^{\omega}}$, $\omega \in (0,1)$ &   $|\cS||\cA| \Big\{ \big( \frac{1}{(1-\gamma)^4\varepsilon^2}\big)^{1/\omega} +  \big( \frac{1}{1-\gamma }\big)^{\frac{1}{1-\omega}} \Big\}$  \\ \hline
	\cite{chen2020finite} $\vphantom{\frac{1^{7^{7}}}{1^{7^{7^{7^{7^{7^7}}}}}}}$  & rescaled linear: $\frac{1}{\frac{1}{(1-\gamma)^2}+(1-\gamma)t}$ & $\frac{ |\cS||\cA|  }{(1-\gamma)^5\varepsilon^2}$ \\ \hline
\cite{chen2020finite} $\vphantom{\frac{1^{7^{7^7}}}{1^{7^{7^{7^7}}}}}$  & constant: $(1-\gamma)^4\varepsilon^2$ & $\frac{ |\cS||\cA|  }{(1-\gamma)^5\varepsilon^2} $ \\ \hline
	{\bf this work (Q-learning, $|\cA|\geq 2$)} $\vphantom{\frac{1^{7^{7^7}}}{1^{7^{7^{7^7}}}}}$  & rescaled linear: $\frac{1}{1+ (1-\gamma)t}$  &  $\frac{ |\cS||\cA|  }{(1-\gamma)^4\varepsilon^2} $  \\ \hline
	{\bf this work (Q-learning, $|\cA|\geq 2$)} $\vphantom{\frac{1^{7^{7^7}}}{1^{7^{7^{7^7}}}}}$  & constant: $(1-\gamma)^3 \varepsilon^2$ &  $\frac{ |\cS||\cA|  }{(1-\gamma)^4\varepsilon^2} $  \\ \toprule
	{\bf this work (TD learning, $|\cA|=1$)} $\vphantom{\frac{1^{7^{7^7}}}{1^{7^{7^{7^7}}}}}$  & rescaled linear: $\frac{1}{1+ (1-\gamma)t}$  &  $\frac{ |\cS|   }{(1-\gamma)^3\varepsilon^2} $  \\ \hline
	{\bf this work (TD learning, $|\cA|=1$)} $\vphantom{\frac{1^{7^{7^7}}}{1^{7^{7^{7^7}}}}}$  & constant: $(1-\gamma)^3 \varepsilon^2$ &  $\frac{ |\cS|  }{(1-\gamma)^3\varepsilon^2} $  \\ \toprule
\end{tabular}
	\caption{Comparisons of existing sample complexity upper bounds of {\em synchronous} Q-learning and TD learning for an infinite-horizon $\gamma$-discounted MDP with state space $\cS$ and action space $\cA$, where $0<\varepsilon<1$ is the target accuracy level. Here, sample complexity refers to the total number of samples  needed to yield either $\max_{s,a} |\widehat{Q}(s,a)- {Q}^{\star}(s,a)|\leq \varepsilon $ with high probability or $\mathbb{E}\big[\max_{s,a} |\widehat{Q}(s,a)- {Q}^{\star}(s,a)| \big] \leq \varepsilon$, where $\widehat{Q}$ is the estimate returned by Q-learning. All logarithmic factors are omitted in the table to simplify the expressions.   \label{table:comparisons}}
\end{table*}


\subsection{Main contributions}

Focusing on $\gamma$-discounted infinite-horizon MDPs with state space $\cS$ and action space $\cA$, 
this paper settles the $\ell_{\infty}$-based sample complexity of synchronous Q-learning.
Here and throughout, the standard notation $f(\cdot)=\widetilde{O} ( g(\cdot) )$ (resp.~$f(\cdot)=\widetilde{\Omega} ( g(\cdot) )$) 
means that $f(\cdot)$ is orderwise no larger than (resp.~no smaller than) $g(\cdot)$ modulo some logarithmic factors. Our main contributions regarding synchronous Q-learning are summarized below.
\begin{itemize}

	\item  When $|\cA|=1$, Q-learning coincides with temporal difference (TD) learning in a Markov reward process. For any $0<\varepsilon<1$, we prove that a total sample size of 
	      \begin{equation}
		      \widetilde{O} \Big( \frac{|\cS|}{(1-\gamma)^3\varepsilon^2} \Big)
		\end{equation}
		is sufficient for TD learning to guarantee $\varepsilon$-accuracy in an $\ell_{\infty}$ sense; see Theorem~\ref{thm:policy-evaluation}. This is sharp and minimax optimal (up to some log factor).

	\item Moving on to the case with $|\cA|\geq 2$, we demonstrate that a sample size of  
	      \begin{equation}
		      \widetilde{O} \Big( \frac{|\cS||\cA|}{(1-\gamma)^4\varepsilon^2} \Big)
		\end{equation}
		suffices for Q-learning to yield $\varepsilon$-accuracy in an $\ell_{\infty}$ sense for any $0<\varepsilon<1$; 
		see Theorem~\ref{thm:infinite-horizon-simple}. 
		Conversely, we construct a hard MDP instance with 4 states and 2 actions, for which Q-learning provably requires at least
      		\begin{equation}
		      \widetilde{\Omega} \Big( \frac{1} {(1-\gamma)^4\varepsilon^2} \Big)
		\end{equation}
		iterations to achieve $\varepsilon$-accuracy in an $\ell_{\infty}$ sense; see Theorem~\ref{thm:LB-example}. 
		These two theorems taken collectively lead to the first sharp characterization of the sample complexity of Q-learning,
		strengthening prior theory \citep{wainwright2019stochastic,chen2020finite} by a factor of $\frac{1}{1-\gamma}$. 
		In addition, the discrepancy between our sharp characterization and the minimax lower bound makes clear that Q-learning is {\em not} minimax optimal when $|\cA|\geq 2$, and is outperformed by, say, the model-based approaches \citep{agarwal2019optimality,li2020breaking} in terms of the sample efficiency.

		
\end{itemize}
Our results cover both rescaled linear and constant learning rates; see Table~\ref{table:comparisons} for more detailed comparisons with previous literature. 
On the technical side, 
(i) our analysis for the upper bound relies on a sort of crucial error decompositions and variance control that are previously unexplored, which might shed light on how to pin down the finite-sample efficacy of other variants of Q-learning such as double Q-learning; 
(ii) the development of our lower bound, which is inspired by  \cite{azar2013minimax,wainwright2019variance}, puts the negative impact of over-estimation on sample efficiency on a rigorous footing.

Finally, we extend our analysis framework to accommodate the asynchronous setting, in which the samples are non-i.i.d.~and take the form of a single Markovian trajectory. We show for the first time that the sample complexity of asynchronous Q-learning exhibits a $\frac{1}{(1-\gamma)^4}$ scaling w.r.t.~the effective horizon, which is nearly sharp and improves upon the prior state-of-the-art \cite{li2020sample}.




%

\subsection{Related works}

There is a growing literature dedicated to analyzing the non-asymptotic behavior of value-based model-free RL algorithms in a variety of scenarios. 
In the discussion below, we subsample the literature and discuss a couple of papers that are the closest to ours. 


\paragraph{Finite-sample $\ell_\infty$-based guarantees for synchronous Q-learning and TD learning.} The sample complexities derived in prior literature often rely crucially on the choices of learning rates. \citet{even2003learning} studied the sample complexity of Q-learning with linear learning rates $1/t$ or polynomial learning rates $1/t^{\omega}$, which scales as $\widetilde{O}(\frac{|\cS||\cA|}{(1-\gamma)^5 \varepsilon^{2.5}})$ when optimized w.r.t.~the effective horizon (attained when $\omega =4/5$). The resulting sample complexity, however, is sub-optimal in terms of its dependency on not only  $\frac{1}{1-\gamma}$ but also the target accuracy level $\varepsilon$. \citet{beck2012error} investigated the case of constant learning rates; however, their result suffered from an additional factor of $|\cS||\cA|$, which could be prohibitively large in practice. More recently, \citet{wainwright2019stochastic,chen2020finite} further analyzed the sample complexity of Q-learning with either constant learning rates or linearly rescaled learning rates, leading to the state-of-the-art bound $\widetilde{O}\big(\frac{|\cS||\cA|}{(1-\gamma)^5\varepsilon^2}\big)$. However, this result remains sub-optimal in terms of its scaling with $\frac{1}{1-\gamma}$.  See Table~\ref{table:comparisons} for details. 
In the special case with $|\cA|=1$, 
the recent works \cite{khamaru2021temporal,mou2020linear} 
developed instance-dependent results for TD learning with Polyak-Ruppert averaging, and studied the local (sub)-optimality of TD learning in a different local minimax framework.   

\paragraph{Finite-sample $\ell_\infty$-based guarantees for asynchronous Q-learning and TD learning.}
Moving beyond the synchronous model, \cite{even2003learning,beck2012error,qu2020finite,li2020sample, shah2018q, chen2021lyapunov}
developed non-asymptotic convergence guarantees for the asynchronous setting, where the data samples take the form of a single Markovian trajectory (following some behavior policy) and only a single state-action pair is updated in each iteration. A similar scaling of $\widetilde{O}\big(\frac{1}{(1-\gamma)^5}\big)$  also showed up in the state-of-the-art sample complexity bounds for asynchronous Q-learning \citep{li2020sample}, and our theory is the first to sharpen it to $\widetilde{O}\big(\frac{1}{(1-\gamma)^4}\big)$. 
When it comes to the special case with $|\cA|=1$,  the non-asymptotic performance guarantees for TD learning with Markovian sample trajectories (assuming that the behavior policy coincides with the target policy) have been recently derived by \cite{bhandari2021finite,srikant2019finite,mou2020linear}.


\paragraph{Finite-sample $\ell_\infty$-based guarantees of other Q-learning variants.}  
With the aim of alleviating the sub-optimal dependency on the effective horizon in vanilla Q-learning and improving sample efficiency, several variants of Q-learning have been proposed and analyzed. \citet{azar2011reinforcement} proposed speedy Q-learning, which achieves a sample complexity of $\widetilde{O} \big( \frac{|\cS||\cA|}{ (1-\gamma)^4\varepsilon^2 } \big)$ at the expense of doubling the computation and storage complexity. Our result on vanilla Q-learning matches that of speedy Q-learning in an order-wise sense. In addition, \citet{wainwright2019variance} proposed a variance-reduced Q-learning algorithm that is shown to be minimax optimal in the range $\epsilon \in (0,1)$ with a sample complexity $\widetilde{O} \big( \frac{|\cS||\cA|}{ (1-\gamma)^3\varepsilon^2 } \big)$, which was subsequently generalized to the asynchronous setting by \citet{li2020sample}. The $\ell_\infty$ statistical bounds for variance-reduced TD learning have been investigated in \citet{khamaru2021temporal} for the synchronous setting, and in \citet{li2020sample} for the asynchronous setting. Last but not least, \citet{xiong2020finitetime} established the finite-sample convergence of double Q-learning following the framework of \cite{even2003learning}; however, it is unclear whether  double Q-learning can provably outperform vanilla Q-learning in terms of the sample efficiency.

%


\paragraph{Others.} 
 There are also several other strands of related papers that tackle model-free algorithms but do not pursue $\ell_{\infty}$-based non-asymptotic guarantees. For instance, \citet{bhandari2021finite,lakshminarayanan2018linear,srikant2019finite,gupta2019finite,doan2019finite,wu2020finite,xu2019reanalysis,chen2019performance,xu2019two} developed finite-sample (weighted) $\ell_2$ convergence guarantees for several model-free algorithms, which also allow one to accommodate linear function approximation as well as off-policy evaluation.  Another line of recent work \citep{jin2018q,li2021breaking,bai2019provably,zhang2020almost} considered the sample efficiency of Q-learning type algorithms paired with proper exploration strategies (e.g., upper confidence bounds) under the framework of regret analysis.   
The asymptotic behaviors of some variants of Q-learning, e.g., double Q-learning \citep{weng2020mean} and relative Q-learning \citep{devraj2020q} are also studied. 
In addition, Q-learning in conjunction with the pessimism principle has proven effective in dealing with offline data \citep{shi2022pessimistic,yan2022efficacy}. 
The effect of more general function approximation schemes (e.g., certain families of neural network approximations) has been studied in \citet{yang2019theoretical,murphy2005generalization,cai2019neural,wai2019variance,xu2020finite}, 
whereas the extension to multi-agent scenarios has been looked at in \citet{hu2003nash,li2022minimax}. These are beyond the scope of the present paper.


%% file: backgrounds.tex
\section{Background and algorithms}
\label{sec:background}

This paper concentrates on discounted infinite-horizon MDPs \citep{bertsekas2017dynamic}. We shall start by introducing some basics of tabular MDPs, followed by a description of both Q-learning and  TD learning. 
Throughout this paper,  
we denote by $\cS=\{1,\cdots, |\cS|\}$ and $\cA=\{1,\cdots, |\cA|\}$ the state space and the action space of the MDP, respectively, 
and let $\Delta(\cS)$  represent the probability simplex over the set $\cS$.  



\paragraph{Basics of discounted infinite-horizon MDPs.}

Consider an infinite-horizon MDP as represented by a quintuple $\mathcal{M} = (\cS,\cA, P, r,\gamma)$, where $\gamma\in (0,1)$ indicates the discount factor, 
 $P:\cS\times\cA \rightarrow \Delta(\cS)$ represents the probability transition kernel (i.e., $P(s' \mymid {s,a})$ is the probability of transiting to state $s'$ from a state-action pair $(s,a)\in \cS\times \cA$),  and $r:  \cS\times\cA \rightarrow [0,1]$ stands for the reward function (i.e., $r(s,a)$ is the immediate reward collected in state $s\in \cS$ when action $a\in \cA$ is taken). Note that the immediate rewards are assumed to lie within $[0,1]$ throughout this paper.  
Moreover, we let $\pi: \cS \rightarrow \Delta(\cA)$ represent a policy, so that $\pi( \cdot \mymid s)\in \Delta(\cA)$ specifies the (possibly randomized) action selection rule in state $s$. If $\pi$ is a deterministic policy, then we denote by $\pi(s)$ the action selected by $\pi$ in state $s$.

A common objective in RL is to maximize a sort of long-term rewards called value functions or Q-functions. 
Specifically, given a policy $\pi$, the associated value function and Q-function of $\pi$ are defined respectively by
 \begin{align*}
   V^{\pi}(s) \defn \mathbb{E} \left[ \sum_{k=0}^{\infty} \gamma^k r(s_k,a_k ) \,\Big|\, s_0 =s \right]
\end{align*} 
 for all $s\in \cS$, and
 \begin{equation*}
Q^{\pi}(s,a) \defn \mathbb{E} \left[ \sum_{k=0}^{\infty} \gamma^k r(s_k,a_k ) \,\Big|\, s_0 =s, a_0 = a \right]
\end{equation*} 
for all $(s,a)\in \cS \times \cA$. 
Here, $\{(s_k,a_k)\}_{k\geq 0}$ is a trajectory of the MDP induced by the policy $\pi$ (except $a_0$ when evaluating the Q-function), and the expectations are evaluated with respect to the randomness of  the MDP trajectory. 
Given that the immediate rewards fall within $[0,1]$, it can be straightforwardly verified that $0\leq V^{\pi}(s)\leq \frac{1}{1-\gamma} $ and $0\leq Q^{\pi}(s,a)\leq \frac{1}{1-\gamma} $ for any $\pi$ and any state-action pair $(s,a)$. The optimal  value function $V^\star$ and optimal Q-function $Q^\star$ are defined respectively as
\begin{equation*}
	V^{\star}(s) \coloneqq \max_{\pi} V^{\pi}(s), \qquad Q^{\star}(s,a) \coloneqq \max_{\pi} Q^{\pi}(s,a)
\end{equation*}
for any state-action pair $(s,a)\in \cS\times \cA$.  
It is well known that there exists a {\em deterministic} optimal policy, denoted by $\pi^{\star}$, that attains $V^{\star}(s)$ and $Q^{\star}(s,a)$ simultaneously for all $(s,a)\in \cS\times \cA$ \citep{sutton2018reinforcement}.

\paragraph{Algorithms: Q-learning and TD learning (the synchronous setting).}

The synchronous setting assumes access to a generative model \citep{kearns1999finite,sidford2018near} such that:  
in each iteration $t$, we collect an independent sample $s_t(s,a) \sim P(\cdot \mymid s,a)$ for every state-action pair $(s,a)\in \cS\times \cA$.

With this sampling model in place, the Q-learning algorithm \citep{watkins1992q} maintains a Q-function estimate $Q_t: \cS\times \cA \rightarrow \mathbb{R}$ for all $t\geq 0$; in each iteration $t$, the algorithm updates {\em all} entries of the Q-function estimate at once via the following update rule
\begin{equation}
\label{eqn:q-learning}
	Q_t  = (1- \eta_t ) Q_{t-1} + \eta_t \mathcal{T}_t (Q_{t-1}) . 
\end{equation}
Here, $\eta_t \in (0,1]$ denotes the learning rate or the step size in the $t$-th iteration,  
and $\mathcal{T}_t$ denotes the empirical Bellman operator constructed by samples collected in the $t$-th iteration, i.e., 
\begin{align}  \label{defn:empirical-Bellman-t-inf}
	\mathcal{T}_t (Q) (s , a ) &\coloneqq 
		 r(s , a ) + \gamma \max_{a' \in \cA} Q(s_t, a') , \qquad s_t  \equiv s_t(s,a) \sim P(\cdot \mymid s,a) 
\end{align}
for each state-action pair $(s,a)\in \cS\times \cA$. 
Obviously, $\mathcal{T}_t$ is an unbiased estimate of the celebrated Bellman operator $\mathcal{T}$ given by  
\begin{equation*}
	\forall (s,a)\in \cS\times \cA: \qquad
	\mathcal{T}(Q)(s,a) := r(s,a) + \gamma \mathop{\mathbb{E}}_{s^{\prime} \sim P(\cdot| s,a)}  \Big[ \max_{a^{\prime}\in \cA} Q(s^{\prime}, a^{\prime}) \Big] .
\end{equation*}
Note that the optimal Q-function $Q^{\star}$ is the unique fixed point of the Bellman operator \citep{bellman1952theory}, that is, $\mathcal{T}(Q^{\star}) =Q^{\star}$.  Viewed in this light, synchronous Q-learning can be interpreted as a stochastic approximation scheme \citep{robbins1951stochastic} aimed at solving this fixed-point equation. 
Throughout this work, we initialize the algorithm in a way that obeys $0\leq Q_0(s,a) \leq \frac{1}{1-\gamma}$ for every state-action pair $(s,a)$.  
In addition, the corresponding value function estimate $V_t: \cS \rightarrow \mathbb{R}$ in the  $t$-th iteration is defined as
\begin{align}
	\label{defn:Vt}
	\forall s\in \cS: \qquad V_t(s) :=  \max_{a\in \cA} Q_{t}(s,a) .
\end{align}
The complete description of Q-learning is summarized in Algorithm~\ref{alg:q-infinite}.

\begin{algorithm}[t]
	\begin{algorithmic}[1] 
	\STATE\textbf{inputs:} learning rates $\{\eta_t\}$, number of iterations $T$, discount factor $\gamma$, initial estimate $Q_0$.
   \FOR{$t=1,2,\cdots,T$}
	\STATE{ Draw  $s_t(s,a) \sim  P(\cdot \mymid s,a)$ for each $(s,a)\in \cS\times \cA$. }
	\STATE{ Compute $Q_t$ according to \eqref{eqn:q-learning} and \eqref{defn:empirical-Bellman-t-inf}.}
\ENDFOR
   \end{algorithmic} 
    \caption{Synchronous Q-learning for infinite-horizon discounted MDPs.}
 \label{alg:q-infinite}
\end{algorithm}

As it turns out, TD learning \citep{sutton1988learning,tsitsiklis1997analysis,bhandari2021finite} in the synchronous setting can be viewed as a special instance of 
Q-learning when the action set $\cA$ is a singleton (i.e., $|\cA|=1$).    
In such a case, the MDP reduces to a Markov reward process (MRP) \citep{bertsekas2017dynamic}, and we shall abuse the notation to use  
 $P: \mathcal{S} \rightarrow \Delta(\cS)$ to describe the probability transition kernel, and employ $r: \mathcal{S} \rightarrow [0,1]$ to represent the reward function (with $r(s)$ indicating the immediate reward gained in state $s$). 
 The TD learning algorithm maintains an estimate $V_t: \mathcal{S} \rightarrow \mathbb{R}$ of the value function in each iteration $t$,\footnote{There is no need to maintain additional Q-estimates, as the Q-function and the value function coincide when $|\cA|=1$. } 
and carries out the following iterative update rule
\begin{align}
\label{eqn:td-learning}
	V_t (s) &= (1- \eta_t ) V_{t-1} (s) + \eta_t \big( r(s) + \gamma   V_{t-1}(s_t) \big),   
	\qquad s_t \equiv s_t(s) \sim P(\cdot\mymid s)
\end{align}
for each state $s\in \cS$. 
As before, $\eta_t\in (0,1]$ is the learning rate at time $t$, the initial estimate $V_0(s)$ is taken to be within $\big[0, \frac{1}{1-\gamma}\big]$, and in each iteration, 
the samples $\{s_t(s) \mymid s\in \cS\}$ are generated independently.  
 The whole algorithm of TD learning is summarized in Algorithm~\ref{alg:td-infinite}.

 \begin{algorithm}[t]
	\begin{algorithmic}[1] 
	\STATE\textbf{inputs:} learning rates $\{\eta_t\}$, number of iterations $T$, discount factor $\gamma$, initial estimate $V_0$. 
   \FOR{$t=1,2,\cdots,T$}
	\STATE{ Draw  $s_t(s) \sim  P(\cdot \mymid s)$ for each $s\in \cS$. }
	\STATE{ Compute $V_t$ according to \eqref{eqn:td-learning}.}
\ENDFOR
   \end{algorithmic} 
    \caption{Synchronous TD learning for infinite-horizon discounted MRPs.}
 \label{alg:td-infinite}
\end{algorithm}

Finally, while synchronous Q-learning is the main focal point of this paper, we shall also discuss the extension to asynchronous Q-learning,  which we will elaborate on in Section~\ref{sec:AsynQ}.

%% file: results.tex
\section{Main results: sample complexity of synchronous Q-learning}
\label{sec:main-results}

With the above backgrounds in place, we are in a position to state formally our main findings in this section, concentrating on the synchronous setting.

\subsection{Minimax optimality of TD learning}

We start with the special with $|\cA|=1$ and characterize the $\ell_{\infty}$-based sample complexity of synchronous TD learning.

\begin{theorem}
\label{thm:policy-evaluation}
	Consider any $\delta\in(0,1)$,  $\varepsilon\in(0,1]$, and $\gamma\in [1/2,1)$. Suppose that for any $0\leq t\leq T$, the
learning rates satisfy
\begin{subequations}
\label{eq:thm-td}
\begin{equation}
\frac{1}{1+\frac{c_1 (1-\gamma) T}{\log^2 T}}\le\eta_{t}\le\frac{1}{1+\frac{c_2 (1-\gamma) t}{\log^2 T}} \label{eq:thm:eta-TD}
\end{equation}
for some small enough universal constants $c_{1}\geq c_{2}>0$. 
Assume that the total number of iterations $T$ obeys
\begin{equation}
	T\ge\frac{c_{3}\big(\log^{3}T\big)\big(\log\frac{|\mathcal{S}|T}{\delta}\big)}{(1-\gamma)^{3}\varepsilon^{2}}
	\label{eq:thm:sample-size-TD-thm}
\end{equation}
for some sufficiently large universal constant $c_{3}>0$. 
\end{subequations}
If the initialization obeys $0\leq {V}_{0}(s) \leq\frac{1}{1-\gamma}$ for all $s\in \cS$, 
then  with probability at least $1-\delta$, Algorithm~\ref{alg:td-infinite} achieves
\begin{align}
	\max_{s\in \cS} \big| {V}_{T}(s) -  {V}^{\star}(s) \big| \le\varepsilon.
	\label{eq:VT-entrywise-TD-thm}
\end{align}
\end{theorem}
\begin{remark}[Mean estimation error] \label{remark:expected-error-inf-td}
	This high-probability bound immediately translates to a mean estimation error guarantee. Recognizing the crude upper bound 
	$\big| {V}_{T}(s) -  {V}^{\star}(s) \big| \leq \frac{1}{1-\gamma}$ (see \eqref{eq:V-i-LB-UB} in Section~\ref{sec:infinite_prelim_TD}) and taking $\delta\leq {\varepsilon(1-\gamma)}$, we reach
	\begin{equation}
		\mathbb{E}\Big[ \max_{s} \big| {V}_{T}(s) -  {V}^{\star}(s) \big| \Big] \le \varepsilon (1-\delta)+\delta \frac{1}{1-\gamma}
		\leq 2\varepsilon,
	\end{equation}
	provided that $T\ge\frac{c_{3} (\log^{3}T)\big(\log\frac{|\mathcal{S}|T}{\varepsilon (1-\gamma)}\big)}{(1-\gamma)^{3}\varepsilon^{2}}$. 
\end{remark}

Given that each iteration of synchronous TD learning makes use of $|\cS|$ samples, 
Theorem~\ref{thm:policy-evaluation} implies that the sample complexity of TD learning is at most
\begin{equation}
	\label{eq:TD-sample-complexity-simple}
	\widetilde{O} \bigg( \frac{|\cS|}{(1-\gamma)^3 \varepsilon^2} \bigg)
\end{equation}
for any target accuracy level $\varepsilon \in (0,1]$. 
This non-asymptotic result is valid as long as
the learning rates are chosen to be either a proper constant or rescaled linear (see~\eqref{eq:thm:eta-TD}).  
Compared to a large number of prior works studying the performance of TD learning \citep{borkar2000ode,bhandari2021finite,khamaru2021temporal,wainwright2019stochastic,chen2020finite,lakshminarayanan2018linear}, 
 Theorem~\ref{thm:policy-evaluation} strengthens prior results by uncovering an improved scaling (i.e., $\frac{1}{(1-\gamma)^3}$) in the effective horizon. In fact, prior results on plain TD learning were only able to obtain a scaling as $\frac{1}{(1-\gamma)^5}$ \citep{wainwright2019stochastic}.

To assess the tightness of the above result, we take a moment to compare it with the minimax lower bound recently established in the context of value function estimation. Specifically,  
\citet[Theorem 2(b)]{pananjady2019value}
asserted that no algorithm whatsoever can obtain an entrywise $\varepsilon$ approximation of the value function---in a minimax sense---unless the total sample size exceeds
\begin{equation}
	\label{eq:TD-sample-complexity-simple-LB}
	\widetilde{\Omega} \bigg( \frac{|\cS|}{(1-\gamma)^3 \varepsilon^2} \bigg). 
\end{equation}
In turn, this taken together with Theorem~\ref{thm:policy-evaluation} unveils the minimax optimality of the sample complexity (modulo some logarithmic factor) of TD learning  for the synchronous setting.  
While prior works have demonstrated how to attain the minimax limit \eqref{eq:TD-sample-complexity-simple-LB} using model-based methods or variance-reduced model-free algorithms (e.g., \citet{azar2013minimax,pananjady2019value,li2020breaking,khamaru2021temporal}), 
our theory provides the first rigorous evidence that plain TD learning alone is already minimax optimal, without the need of Polyak-Ruppert averaging or variance reduction. 

\begin{remark}[Runtime-oblivious learning rates] \label{remark:anytime-TD}
	Careful readers might remark that the choice \eqref{eq:thm:eta-TD} of the learning rates might still rely on prior knowledge on $T$ (or $\log T$). Fortunately, Theorem~\ref{thm:policy-evaluation} immediately leads to convergence guarantees for another choice of $\eta_t$ selected completely independent of $T$. More specifically, suppose that the learning rates obey
	\begin{align}
		\frac{1}{1 + \frac{\widetilde{c}_1(1-\gamma)t}{\log^2 (t+1)}} \le \eta_t \le \frac{1}{1 + \frac{\widetilde{c}_2(1-\gamma)t}{\log^2 (t+1)}}, \qquad \forall t \geq 1
			\label{eq:new-choice-learning-rates-remark}
	\end{align}
	for some universal constants $\widetilde{c}_1,\widetilde{c}_2>0$. Then the claim \eqref{eq:VT-entrywise-TD-thm} remains valid under this choice \eqref{eq:new-choice-learning-rates-remark}, 	provided that	
		\begin{align}
			T \geq \frac{2c_3 (\log^3 T)\big( \log \frac{|\cS|T}{\delta} \big) } {(1-\gamma)^3 \varepsilon^2}.
			\label{eq:sample-size-anytime-TD}
		\end{align}
See Appendix~\ref{sec:proof-remark:anytime-TD} for the proof. 
\end{remark}
\begin{remark}[Polyak-Ruppert averaging]
\label{remark:averaging-TD}
The results claimed in Remark~\ref{remark:anytime-TD} further allow us to control the estimation error of TD learning under Polyak-Ruppert averaging \citep{polyak1992acceleration}. More precisely, under the choice \eqref{eq:new-choice-learning-rates-remark} of learning rates, the averaged iterates satisfy
\begin{align}
	\max_{s\in \cS} \Big| \frac{1}{T}\sum_{t=1}^T {V}_{T}(s) -  {V}^{\star}(s) \Big| 
	\leq 4\sqrt{\frac{c_{3}(\log^{3}T)\big(\log\frac{|\cS|T}{\delta}\big)}{(1-\gamma)^{3}T}} 
\end{align}
with probability exceeding $1-\delta$. See Appendix~\ref{sec:proof-remark:anytime-TD} for the proof.
\end{remark}

\begin{remark}
	It is also noteworthy that: while the last iterate of plain TD learning is shown to be minimax optimal (which concerns worst-case optimality), 
	it might not necessarily enjoy local optimality.  
	As recently demonstrated by \citet{khamaru2021instance}, 
	additional algorithmic tricks like variance reduction might be needed in order to ensure local optimality. 
\end{remark}


%

\subsection{Tight sample complexity and sub-optimality of Q-learning}

Next, we move on to the more general case with $|\cA|\geq 2$ and study the performance of Q-learning. 
As it turns out, Q-learning with $|\cA|\geq 2$ is considerably more challenging to analyze than the TD learning case,
due to the presence of the nonsmooth max operator.  
Our $\ell_{\infty}$-based sample complexity bound for Q-learning is summarized as follows, 
strengthening the state-of-the-art results. 
\begin{theorem}
\label{thm:infinite-horizon-simple}
Consider any $\delta\in(0,1)$,  $\varepsilon\in(0,1]$, and $\gamma\in [1/2,1)$. Suppose that for any $0\leq t\leq T$, the
learning rates satisfy
\begin{subequations}
\label{eq:thm-infinite-horizon-condition}
\begin{equation}
\frac{1}{1+\frac{c_1 (1-\gamma) T}{\log^3 T}}\le\eta_{t}\le\frac{1}{1+\frac{c_2 (1-\gamma) t}{\log^3 T}} \label{eq:thm:infinite-horizaon-eta}
\end{equation}
for some small enough universal constants $c_{1}\geq c_{2}>0$. 
Assume that the total number of iterations $T$ obeys
\begin{equation}
	T\ge\frac{c_{3}\big(\log^{4}T\big)\big(\log\frac{|\mathcal{S}||\mathcal{A}|T}{\delta}\big)}{(1-\gamma)^{4}\varepsilon^{2}}
	\label{eq:thm:infinite-horizon-T}
\end{equation}
for some sufficiently large universal constant $c_{3}>0$. 
\end{subequations}
If the initialization obeys $0\leq {Q}_{0}(s,a) \leq\frac{1}{1-\gamma}$ for any $(s,a)\in \cS\times \cA$, 
then Algorithm \ref{alg:q-infinite} achieves
\begin{align}
	\max_{(s,a)\in \cS\times \cA} \big| {Q}_{T}(s,a) -  {Q}^{\star}(s,a) \big| \le\varepsilon
	\label{eq:QT-entrywise-Q-thm}
\end{align}
 with probability at least $1-\delta$. 
\end{theorem}
\begin{remark}[Mean estimation error] \label{remark:expected-error-inf}
	Repeating exactly the same argument as in Remark~\ref{remark:expected-error-inf-td}, one can readily translate this 
high-probability bound into the following  mean estimation error guarantee: 
	\begin{equation}
		\mathbb{E}\Big[ \max_{s,a} \big| {Q}_{T}(s,a) -  {Q}^{\star}(s,a) \big| \Big] \le \varepsilon (1-\delta)+\delta \frac{1}{1-\gamma}
		\leq 2\varepsilon,
	\end{equation}
	holds as long as $T\ge\frac{c_{3} (\log^{4}T)\big(\log\frac{|\mathcal{S}||\mathcal{A}|T}{\varepsilon (1-\gamma)}\big)}{(1-\gamma)^{4}\varepsilon^{2}}$. 
\end{remark}

In a nutshell, Theorem~\ref{thm:infinite-horizon-simple} develops a non-asymptotic bound on the iteration complexity of Q-learning in the presence of the synchronous model.  A few remarks and implications are in order.

\paragraph{Sample complexity and sharpened dependency on $\frac{1}{1-\gamma}$.} 
Recognizing that $|\cS||\cA|$ independent samples are drawn in each iteration, 
we can see from Theorem~\ref{thm:infinite-horizon-simple} the following sample complexity bound
\begin{equation}
	\widetilde{O}\Big( \frac{|\cS||\cA|}{(1-\gamma)^4 \varepsilon^2} \Big)
	\label{eq:sample-complexity-sync-Q-learning}
\end{equation}
in order for Q-learning to attain $\varepsilon$-accuracy ($0<\varepsilon<1$) in an entrywise sense.  
To the best of our knowledge, this is the first result that breaks the $\frac{|\cS||\cA|}{(1-\gamma)^5 \varepsilon^2}$ barrier that is present in all state-of-the-art analyses for vanilla Q-learning \citep{beck2012error,wainwright2019stochastic,chen2020finite,qu2020finite,li2020sample}. 

\paragraph{Learning rates.} Akin to the TD learning case, 
our result accommodates two commonly adopted learning rate schemes (cf.~\eqref{eq:thm:infinite-horizaon-eta}): (i) linearly rescaled learning rates $\frac{1}{1+\frac{c_2 (1-\gamma)}{\log^2 T}t}$, and (ii) iteration-invariant learning rates $\frac{1}{1+\frac{c_1 (1-\gamma) T}{\log^2 T}}$ (which depend on the total number of iterations $T$ but not the iteration number $t$).
In particular, when $T=\frac{c_{3} (\log^{4}T)\big(\log\frac{|\mathcal{S}||\mathcal{A}|T}{\delta}\big)}{(1-\gamma)^{4}\varepsilon^{2}}$, the constant learning rates can be taken to be on the order of
\[
	\eta_t \equiv \widetilde{O} \big( (1-\gamma)^3 \varepsilon^2 \big),  \qquad 0\leq t\leq T, 
\]
which depends almost solely on the discount factor $\gamma$ and the target accuracy $\varepsilon$. 
Interestingly, both learning rate schedules lead to the same $\ell_{\infty}$-based sample complexity bound (in an order-wise sense), making them appealing for practical use.  

\begin{remark}[Runtime-oblivious learning rates and Polyak-Ruppert averaging] \label{remark:anytime-Q}
	Akin to Remark~\ref{remark:anytime-TD}, 
	Theorem~\ref{thm:infinite-horizon-simple} can be easily extended to accommodate a family of learning rates chosen without prior knowledge of $T$. More concretely, suppose that the learning rates obey
	\begin{align}
		\frac{1}{1 + \frac{\widetilde{c}_1(1-\gamma)t}{\log^3 (t+1)}} \le \eta_t \le \frac{1}{1 + \frac{\widetilde{c}_2(1-\gamma)t}{\log^3 (t+1)}}, \qquad \forall t \geq 1
			\label{eq:new-choice-learning-rates-remark-Q}
	\end{align}
	for some suitable constants $\widetilde{c}_1,\widetilde{c}_2>0$. Then the claim \eqref{eq:QT-entrywise-Q-thm} continues to hold under this choice \eqref{eq:new-choice-learning-rates-remark-Q}, 	provided that	
		$
			T/2 \geq \frac{c_3 (\log^4 T)\big( \log \frac{|\cS||\cA|T}{\delta} \big) } {(1-\gamma)^4 \varepsilon^2}.
		$	
Additionally, similar to Remark~\ref{remark:averaging-TD}, we can demonstrate that the averaged Q-learning iterates
under the choice \eqref{eq:new-choice-learning-rates-remark-Q} of learning rates obey
\begin{align}
	\max_{(s,a)\in \cS\times \cA} \Big| \frac{1}{T}\sum_{t=1}^T {Q}_{T}(s,a) -  {Q}^{\star}(s,a) \Big| 
	\leq 4\sqrt{\frac{c_{3}(\log^{4}T)\big(\log\frac{|\cS||\cA|T}{\delta}\big)}{(1-\gamma)^{4}T}} 
\end{align}
with probability exceeding $1-\delta$.
	The proofs of these results are identical to those of Remarks~\ref{remark:anytime-TD}-\ref{remark:averaging-TD} (see Appendix~\ref{sec:proof-remark:anytime-TD}), and are hence omitted. 
\end{remark}
%


\paragraph{A matching lower bound and sub-optimality.} 

The careful reader might remark that there remains a gap between our sample complexity bound for Q-learning and the minimax lower bound \citep{azar2013minimax}. More specifically, the minimax lower bound scales on the order of $\frac{|\cS||\cA|}{(1-\gamma)^3\varepsilon^2}$ and is achievable---up to some logarithmic factor---by the model-based approach and variance-reduced methods \citep{azar2013minimax,agarwal2019optimality,li2020breaking,wainwright2019variance}. 
This raises natural questions regarding whether our sample complexity bound can be further improved, and whether there is any intrinsic bottleneck that prevents vanilla Q-learning from attaining optimal performance.  
To answer these questions, we develop the following lower bound for plain Q-learning, 
with the aim of confirming the sharpness of Theorem~\ref{thm:infinite-horizon-simple} and revealing the sub-optimality of Q-learning. 
%
\begin{theorem}
\label{thm:LB-example}
	Assume that $3/4\leq \gamma < 1$ and that $T\geq \frac{c_3}{(1-\gamma)^2}$ for some sufficiently large constant $c_3>0$. 
	Suppose that the initialization is $Q_0 \equiv 0$, and that the learning rates are taken to be either (i) $\eta_t = \frac{1}{1+\ceta (1-\gamma)t}$ for all $t\geq 0$, or (ii) $\eta_t \equiv \eta$ for all $t\geq 0$. 
	There exists a $\gamma$-discounted MDP with $|\cS|=4$ and $|\cA|=2$ such that Algorithm \ref{alg:q-infinite}---with any $\ceta >0$ and any $\eta \in (0,1)$---obeys
\begin{equation}
	\max_{s\in \cS} \mathbb{E} \Big[ \big|V_T(s) - V^{\star}(s)\big|^2 \Big] 
	\geq  \frac{c_{\mathsf{lb}}}{(1-\gamma)^4T\log^2 T} ,
	\label{eq:lower-bounding-VT-thm-sync}
\end{equation}
where $c_{\mathsf{lb}} > 0$ is some universal constant.
\end{theorem}
\begin{remark}
	This theorem constructs a hard MDP instance with no more than 4 states and 2 actions, with the emphasis of unveiling the sub-optimality of horizon dependency. It can be generalized to accommodate larger state/action space, as we shall elucidate in Section~\ref{sec:MDP-construction-hard}. 
\end{remark}
\begin{remark}
	Theorem~\ref{thm:LB-example} concentrates on two families of learning rates---rescaled linear, and constant learning rates---that are most widely used in practice. Note, however, that our current analysis does not readily generalize to arbitrary learning rates, which we leave for future investigation.  
\end{remark}
Theorem~\ref{thm:LB-example} provides an {\em algorithm-dependent} lower bound for vanilla Q-learning. 
As asserted by this theorem, it is impossible for Q-learning to attain $\varepsilon$-accuracy (in the sense that $\max_s \mathbb{E} \big[ \big|V_T(s) - V^{\star}(s)\big|^2 \big] \leq \varepsilon^2$) unless the number of iterations exceeds the order of 
\[
	\frac{1}{(1-\gamma)^{4}\varepsilon^{2}}
\]
up to some logarithmic factor. Consequently, the performance guarantees for Q-learning derived in Theorem~\ref{thm:infinite-horizon-simple} are sharp in terms of the dependency on the effective horizon $\frac{1}{1-\gamma}$.  
On the other hand, it has been shown in prior literature that the minimax sample complexity limit with a generative model is on the order of \citep{azar2013minimax,li2020sample} 
\begin{align}
	\frac{|\cS||\cA|}{(1-\gamma)^3\varepsilon^2} \qquad \text{(up to log factor)}; 
	\label{eq:minimax-lower-limit-simulator}
\end{align}
this in turn reveals the sub-optimality of plain Q-learning, 
whose horizon scaling is larger than the minimax limit by a factor of $\frac{1}{1-\gamma}$. 
Hence, more sophisticated algorithmic tricks are necessary in order to further reduce the sample complexity. 
For instance, a variance-reduced variant of Q-learning---namely, leveraging the idea of variance reduction originating from stochastic optimization \citep{johnson2013accelerating} to accelerate convergence of Q-learning---has been shown to attain minimax optimality 
\eqref{eq:minimax-lower-limit-simulator} for any $\varepsilon \in (0,1]$; see \citet{wainwright2019variance} for more details.


%

%% file: analysis.tex
\section{Key analysis ideas (the synchronous case)}
\label{sec:analysis}

This section outlines the key ideas for the establishment of our main results of Q-learning for the synchronous case, namely Theorem~\ref{thm:infinite-horizon-simple} and Theorem~\ref{thm:LB-example}. The proof for TD learning is deferred to Appendix~\ref{sec:TD-learning-analysis}. Before delving into the proof details, we first introduce convenient vector and matrix notation that shall be used frequently. 

\input{matrix_notation.tex}

\input{proof_outline}

%% file: matrix_notation.tex
\subsection{Vector and matrix notation}
\label{subsec:matrix-notation}


To begin with, for any matrix $\bm{M}$, 
the notation $\|\bm{M}\|_{1}:=\max_{i}\sum_{j}|M_{i,j}|$
is defined as the largest row-wise $\ell_{1}$ norm of $\bm{M}$.  
For any vector $\bm{a}=[a_{i}]_{i=1}^{n}\in\mathbb{R}^{n}$, we define
$\sqrt{\cdot}$ and $|\cdot|$ in a coordinate-wise manner, i.e.~$\sqrt{\bm{a}}:=[\sqrt{a_{i}}\,]_{i=1}^{n}\in\mathbb{R}^{n}$
and $|\bm{a}|:=[|a_{i}|]_{i=1}^{n}\in\mathbb{R}^{n}$. 
For a set of vectors $\bm{a}_1,\cdots,\bm{a}_m \in \mathbb{R}^n$ with $\bm{a}_k=[a_{k,j}]_{j=1}^n$ ($1\leq k\leq m$), we define the $\max$ operator in an entrywise fashion such that 
$\max_{1\leq k\leq m} \bm{a}_k \coloneqq [ \max_{k} a_{k,j} ]_{j=1}^n$. 
For any vectors $\bm{a}=[a_{i}]_{i=1}^{n}\in\mathbb{R}^{n}$ and $\bm{b}=[b_{i}]_{i=1}^{n}\in\mathbb{R}^{n}$,
the notation $\bm{a}\leq\bm{b}$ (resp.~$\bm{a}\geq\bm{b}$) means
$a_{i}\leq b_{i}$ (resp.~$a_{i}\geq b_{i}$) for all $1\leq i\leq n$.
We also let $\bm{a}\circ\bm{b}=[a_{i}b_{i}]_{i=1}^{n}$ denote the
Hadamard product. 
In addition, we denote by $\bm{1}$ (resp.~$\bm{e}_{i}$)
the all-one vector (resp.~the $i$-th standard basis vector), and let $\bm{I}$ be the identity matrix.

We shall also introduce the matrix $\bm{P}\in\mathbb{R}^{|\mathcal{S}||\mathcal{A}|\times|\mathcal{S}|}$
to represent the probability transition kernel $P$, whose $(s,a)$-th
row $\bm{P}_{s,a}$ is a probability vector representing $P(\cdot \mymid s,a)$.
Additionally, we define the {\em square} probability transition matrix
$\bm{P}^{\pi}\in\mathbb{R}^{|\mathcal{S}||\mathcal{A}|\times|\mathcal{S}||\mathcal{A}|}$
(resp.~$\bm{P}_{\pi}\in\mathbb{R}^{|\mathcal{S}|\times|\mathcal{S}|}$)
induced by a {\em deterministic} policy $\pi$ over the state-action pairs (resp.~states) as follows:
\begin{equation}
	\bm{P}^{\pi}\coloneqq\bm{P}\bm{\Pi}^{\pi}\qquad\text{and}\qquad\bm{P}_{\pi}\coloneqq\bm{\Pi}^{\pi}\bm{P} ,
	\label{eqn:ppivq}
\end{equation}
where $\bm{\Pi}^{\pi}\in\{0,1\}^{|\mathcal{S}|\times|\mathcal{S}||\mathcal{A}|}$
is a projection matrix associated with the deterministic policy $\pi$:
\begin{align}
\bm{\Pi}^{\pi} & =\begin{pmatrix}\bm{e}_{\pi(1)}^{\top}\\
 & \bm{e}_{\pi(2)}^{\top}\\
 &  & \ddots\\
 &  &  & \bm{e}_{\pi(|\mathcal{S}|)}^{\top}
\end{pmatrix}\label{eqn:bigpi}
\end{align}
with $\bm{e}_i$ the $i$-th standard basis vector. 
Moreover, for any vector $\bm{V}\in\mathbb{R}^{|\mathcal{S}|}$,
we define $\mathsf{Var}_{\bm{P}}(\bm{V})\in\mathbb{R}^{|\mathcal{S}||\mathcal{A}|}$ as follows:
\begin{align}
	\label{eq:defn-Var-P-V}
	\mathsf{Var}_{\bm{P}}(\bm{V})=\bm{P}(\bm{V}\circ\bm{V})-(\bm{P}\bm{V})\circ(\bm{P}\bm{V}).
\end{align}
In other words, the $(s,a)$-th entry of $\mathsf{Var}_{\bm{P}}(\bm{V})$ corresponds to the variance $\mathsf{Var}_{s'\sim P(\cdot|s,a)}(V(s'))$ w.r.t.~the distribution  $P(\cdot\mymid s,a)$.

Moreover, we use the vector $\bm{r}\in\mathbb{R}^{|\mathcal{S}||\mathcal{A}|}$
to represent the reward function $r$, so that for any $(s,a)\in\mathcal{S}\times\mathcal{A}$, the $(s,a)$-th entry of $\bm{r}$ is given by $r(s,a)$. 
Analogously, we shall employ the vectors $\bm{V}^{\pi}\in\mathbb{R}^{|\mathcal{S}|}$,
$\bm{V}^{\star}\in\mathbb{R}^{|\mathcal{S}|}$, $\bm{V}_{t}\in\mathbb{R}^{|\mathcal{S}|}$, $\bm{Q}^{\pi}\in\mathbb{R}^{|\mathcal{S}||\mathcal{A}|}$, $\bm{Q}^{\star}\in\mathbb{R}^{|\mathcal{S}||\mathcal{A}|}$ and $\bm{Q}_{t}\in\mathbb{R}^{|\mathcal{S}||\cA|}$  to represent $V^{\pi}$, $V^{\star}$, $V_t$, $Q^{\pi}$, $Q^{\star}$ and $Q_t$, respectively. 
Additionally, we define $\pi_{t}$ to be the policy associated with $Q_t$ such that for any
state-action pair $(s,a)$, 
\begin{align}
	\pi_{t}(s) = \min\Big\{ a' \,\big|\, Q_{t}(s,a')=\max_{a''}Q_{t}(s,a'')\Big\}.
	\label{eq:defn-pit-s}
\end{align}
%
%
In other words, for any $s\in\mathcal{S}$, the policy $\pi_{t}$
picks out the smallest indexed action that attains the largest Q-value
in the estimate $Q_{t}(s,\cdot)$. As an immediate consequence, one can easily verify 
\begin{equation}
\label{eq:Vt-Qt-pit-connection-infinite}
	Q_{t}\big(s,\pi_{t}(s)\big)=V_{t}(s)  \qquad \text{and} \qquad \bm{P}\bm{V}_{t}=\bm{P}^{\pi_{t}}\bm{Q}_{t}\geq\bm{P}^{\pi}\bm{Q}_{t}
\end{equation}
for any $\pi,$ where $\bm{P}^{\pi}$ is defined in (\ref{eqn:ppivq}). 
Further, we introduce a matrix $\bm{P}_{t}\in\{0,1\}^{|\mathcal{S}||\mathcal{A}|\times|\mathcal{S}|}$ such that
\begin{equation}
\bm{P}_{t}\big((s,a),s'\big):=\begin{cases}
1, & \text{if } s' = s_t(s,a)\\
0, & \text{otherwise}
\end{cases}\label{eq:defn-Pt}
\end{equation}
for any $(s,a)$, which is an empirical transition matrix constructed using samples collected in the $t$-th iteration.

Finally, let $\mathcal{X}\defn \big( |\cS|, |\cA|, \frac{1}{1-\gamma}, \frac{1}{\varepsilon} \big)$. 
The notation  $f(\mathcal{X})= O(g(\mathcal{X}))$ or $f(\mathcal{X})\lesssim g(\mathcal{X})$ (resp.~$f(\mathcal{X})\gtrsim g(\mathcal{X})$) means that there exists a universal constant $C_{0}>0$ such that $|f(\mathcal{X})|\leq C_{0}|g(\mathcal{X})|$ (resp.~$|f(\mathcal{X})|\geq C_{0}|g(\mathcal{X})|$). 
The notation $f(\mathcal{X})\asymp g(\mathcal{X})$ means  $f(\mathcal{X})\lesssim g(\mathcal{X})$ and $f(\mathcal{X})\gtrsim g(\mathcal{X})$ hold simultaneously. 
We define $\widetilde{O}(\cdot)$ in the same way as $O(\cdot)$ except that it hides logarithmic factors.

%% file: proof_outline.tex
\subsection{Proof outline for Theorem~\ref{thm:infinite-horizon-simple}}
\label{sec:proof-outline-thm:infinite-horizon}

We are now positioned to describe how to establish Theorem \ref{thm:infinite-horizon-simple}, towards which we first express the Q-learning update rule \eqref{eqn:q-learning} and \eqref{defn:empirical-Bellman-t-inf} using the above matrix notation. 
As can be easily verified, Q-learning employs the samples in $\bm{P}_{t}$ (cf.~\eqref{eq:defn-Pt}) to perform the following update
%
\begin{align}
\bm{Q}_{t} & =(1-\eta_{t})\bm{Q}_{t-1}+\eta_{t}(\bm{r}+\gamma\bm{P}_{t}\bm{V}_{t-1}) \label{eq:iteration-rule-infinite}
\end{align}
%
in the $t$-th iteration. In the sequel, we denote by 
\begin{equation}
	\bm{\Delta}_{t}\coloneqq\bm{Q}_{t}-\bm{Q}^{\star}
	\label{eq:defn-Delta-t-infinite}
\end{equation}
the error of the Q-function estimate in the $t$-th iteration.

\subsubsection{Basic decomposition}

We start by decomposing the estimation error term $\bm{\Delta}_{t}$. 
In view of the update rule (\ref{eq:iteration-rule-infinite}), we
arrive at the following elementary decomposition:
\begin{align}
\bm{\Delta}_{t} & =\bm{Q}_{t}-\bm{Q}^{\star}=(1-\eta_{t})\bm{Q}_{t-1}+\eta_{t}\big(\bm{r}+\gamma\bm{P}_{t}\bm{V}_{t-1}\big)-\bm{Q}^{\star}\nonumber \\
 & =(1-\eta_{t})\big(\bm{Q}_{t-1}-\bm{Q}^{\star}\big)+\eta_{t}\big(\bm{r}+\gamma\bm{P}_{t}\bm{V}_{t-1}-\bm{Q}^{\star}\big)\nonumber \\
 & =(1-\eta_{t})\bm{\Delta}_{t-1}+\eta_{t}\gamma\big(\bm{P}_{t}\bm{V}_{t-1}-\bm{P}\bm{V}^{\star}\big)\nonumber \\
 & =(1-\eta_{t})\bm{\Delta}_{t-1}   +\eta_{t}\gamma\big\{\bm{P}(\bm{V}_{t-1}-\bm{V}^{\star})+(\bm{P}_{t}-\bm{P})\bm{V}_{t-1}\big\}, 
 \label{eq:iteration-infinite}
\end{align}
where the third line exploits the Bellman equation $\bm{Q}^{\star}= \bm{r}+ \gamma \bm{P}\bm{V}^{\star}$. 
Further, the term $\bm{P}(\bm{V}_{t-1}-\bm{V}^{\star})$ can be linked
with $\bm{\Delta}_{t-1}$ using the definition \eqref{eq:defn-pit-s} of $\pi_t$ as follows 
\begin{subequations}
\label{eq:V2Q-infinite}
\begin{align}
 \bm{P} (\bm{V}_{t-1}-\bm{V}^{\star}) & =\bm{P}^{\pi_{t-1}}\bm{Q}_{t-1}-\bm{P}^{\pi^{\star}}\bm{Q}^{\star}  \leq\bm{P}^{\pi_{t-1}}\bm{Q}_{t-1}-\bm{P}^{\pi_{t-1}}\bm{Q}^{\star}=\bm{P}^{\pi_{t-1}}\bm{\Delta}_{t-1},\\
 \bm{P}  (\bm{V}_{t-1} -\bm{V}^{\star}) &=\bm{P}^{\pi_{t-1}}\bm{Q}_{t-1}-\bm{P}^{\pi^{\star}}\bm{Q}^{\star} \geq\bm{P}^{\pi^{\star}}\bm{Q}_{t-1}-\bm{P}^{\pi^{\star}}\bm{Q}^{\star}=\bm{P}^{\pi^{\star}}\bm{\Delta}_{t-1},
\end{align}
\end{subequations}
where we have made use of the relation \eqref{eq:Vt-Qt-pit-connection-infinite}.
Substitute \eqref{eq:V2Q-infinite} into \eqref{eq:iteration-infinite} to reach
%
%
\begin{align}
\label{eq:Delta-onestep-UB-LB}
\begin{array}{l}
\bm{\Delta}_{t}  \le(1-\eta_{t})\bm{\Delta}_{t-1}+\eta_{t}\gamma\big\{\bm{P}^{\pi_{t-1}}\bm{\Delta}_{t-1}+(\bm{P}_{t}-\bm{P})\bm{V}_{t-1}\big\};\\
\bm{\Delta}_{t}  \geq(1-\eta_{t})\bm{\Delta}_{t-1}+\eta_{t}\gamma\big\{\bm{P}^{\pi^{\star}}\bm{\Delta}_{t-1}+(\bm{P}_{t}-\bm{P})\bm{V}_{t-1}\big\}.
\end{array}
\end{align}
%
%
Applying these relations recursively, we obtain 
%
%
\begin{align}
\label{eq:Delta-allstep-UB-LB}
\begin{array}{l}
\bm{\Delta}_{t}  \le\eta_{0}^{(t)}\bm{\Delta}_{0}+\sum\limits_{i=1}^{t}\eta_{i}^{(t)}\gamma\big\{\bm{P}^{\pi_{i-1}}\bm{\Delta}_{i-1}+(\bm{P}_{i}-\bm{P})\bm{V}_{i-1}\big\}, \\
\bm{\Delta}_{t}  \ge\eta_{0}^{(t)}\bm{\Delta}_{0}+\sum\limits_{i=1}^{t}\eta_{i}^{(t)}\gamma\big\{\bm{P}^{\pi^{\star}}\bm{\Delta}_{i-1}+(\bm{P}_{i}-\bm{P})\bm{V}_{i-1}\big\}, 
\end{array}
\end{align}
%
%
where we define
\begin{equation}
\eta_{i}^{(t)}\coloneqq\begin{cases}
\prod_{j=1}^{t}(1-\eta_{j}), & \text{if }i=0,\\
\eta_{i}\prod_{j=i+1}^{t}(1-\eta_{j}), & \text{if }0<i<t,\\
\eta_{t}, & \text{if }i=t.
\end{cases}\label{def:eta-i-t-simple}
\end{equation}

\paragraph{Comparisons to prior approaches.} 

We take a moment to discuss how prior analyses handle the above elementary decomposition. 
Several prior works (e.g., \citet{wainwright2019stochastic,li2020sample}) tackled the second term on the right-hand side of the relation \eqref{eq:Delta-onestep-UB-LB} via the following crude bounds:
\begin{align*} 
	\bm{P}^{\pi_{i-1}}\bm{\Delta}_{i-1}  &\leq \big\|\bm{P}^{\pi_{i-1}} \big\|_1  \| \bm{\Delta}_{i-1} \|_{\infty} \bm{1} = \| \bm{\Delta}_{i-1} \|_{\infty} \bm{1} , \\
	 \bm{P}^{\pi^{\star}}\bm{\Delta}_{i-1}  &\geq - \big\|\bm{P}^{\pi^{\star}} \big\|_1  \| \bm{\Delta}_{i-1} \|_{\infty} \bm{1} = - \| \bm{\Delta}_{i-1} \|_{\infty} \bm{1} ,
\end{align*}
which, however, are too loose when characterizing the dependency on $\frac{1}{1-\gamma}$. 
By contrast, expanding terms recursively without the above type of crude bounding and carefully analyzing the aggregate terms (e.g., $\sum_{i=1}^t \eta_{i}^{(t)}\bm{P}^{\pi_{i-1}}\bm{\Delta}_{i-1}$) play a major role in sharpening the dependence of sample complexity on the effective horizon.

\subsubsection{Key intertwined relations underlying $\{\|\bm{\Delta}_{t}\|_{\infty}\}$}

By exploiting the crucial relations \eqref{eq:Delta-allstep-UB-LB} derived above, we proceed to upper and lower bound $\bm{\Delta}_{t}$ separately. 
To be more specific, defining 
\begin{align}
	\beta\coloneqq\frac{c_{4}(1-\gamma)}{\log T}
\end{align}
for some constant $c_4>0$, one can further decompose the upper bound in \eqref{eq:Delta-allstep-UB-LB} into several terms:
\begin{align}
	\bm{\Delta}_{t}  & \leq\underbrace{\eta_{0}^{(t)}\bm{\Delta}_{0}+\sum_{i=1}^{(1-\beta)t}   \eta_{i}^{(t)}\gamma\big(\bm{P}^{\pi_{i-1}}\bm{\Delta}_{i-1}+(\bm{P}_{i}-\bm{P})\bm{V}_{i-1}\big)}_{=:\,\bm{\zeta}_{t}} \\
	& \qquad \qquad +    \underbrace{\sum_{i=(1-\beta)t+1}^{t}  \eta_{i}^{(t)}\gamma(\bm{P}_{i}-\bm{P})\bm{V}_{i-1}}_{=:\,\bm{\xi}_{t}}  	+\sum_{i=(1-\beta)t+1}^{t}  \eta_{i}^{(t)}  \gamma\bm{P}^{\pi_{i-1}}\bm{\Delta}_{i-1}.  
	\label{eq:Delta-UB2}
\end{align}
Let us briefly remark on the effect of the first two terms:
\begin{itemize}
	\item Each component in the first term $\bm{\zeta}_t$ is fairly small, given that $\eta_{i}^{(t)}$ is sufficiently small for any $i\leq (1-\beta)t$ (meaning that each component has undergone contraction---the ones taking the form of $1-\eta_j$---for sufficiently many times). As a result, the influence of $\bm{\zeta}_t$ becomes somewhat negligible. 

	\item The second term $\bm{\xi}_t$, which can be controlled via Freedman's inequality \citep{freedman1975tail} due to its martingale structure, contributes to the main variance term in the above recursion.  Note, however, that the resulting variance term also depends on $\{\bm{\Delta}_i\}$.
	
\end{itemize}
In summary, the right-hand side of the above inequality can be further decomposed into some weighted superposition of $\{\bm{\Delta}_i\}$ in addition to  some negligible effect. This is formalized in the following two lemmas, which make apparent the key intertwined relations underlying $\{\bm{\Delta}_i\}$.
%
%
\begin{lemma}
\label{lem:inf-upper-bound}
	Suppose that $c_1c_2\leq c_4/8$. 
	With probability at least $1-\delta$,
\begin{align*}
\bm{\Delta}_{t} & \leq 30 \sqrt{\frac{\big(\log^{4}T\big)\big(\log\frac{|\mathcal{S}||\mathcal{A}|T}{\delta}\big)}{\gamma^2(1-\gamma)^{4}T}\Big(1+\max_{\frac{t}{2}\le i<t}\|\bm{\Delta}_{i}\|_{\infty}\Big)}\ \bm{1} 
\end{align*}
holds simultaneously for all $t\ge\frac{T}{c_{2}\log T}$.
\end{lemma}
%
%
\begin{lemma}
\label{lem:inf-lower-bound}
Suppose that $c_1c_2\leq c_4/8$. 
With probability at least $1-\delta$, 
\begin{equation*}
\bm{\Delta}_{t}\geq - 30 \sqrt{\frac{\big(\log^{4}T\big)\big(\log\frac{|\mathcal{S}||\mathcal{A}|T}{\delta}\big)}{\gamma^2(1-\gamma)^{4}T}\Big(1+\max_{\frac{t}{2}\le i<t}\|\bm{\Delta}_{i}\|_{\infty}\Big)}\ \bm{1} 
\end{equation*}
holds simultaneously for all $t\ge\frac{T}{c_{2}\log T}$. 
\end{lemma}
\begin{proof}
	The proofs of Lemma~\ref{lem:inf-upper-bound} and Lemma~\ref{lem:inf-lower-bound} are deferred to  Appendices~\ref{sec:proof_inf_upper_bound} and \ref{sec:proof_inf_lower_bound}, respectively. As a remark, our analysis collects all the error terms accrued through the iterations---instead of bounding them individually---by conducting a high-order nonlinear expansion of the estimation error through recursion, followed by careful control of the main variance term leveraging the structure of the discounted MDP.  
\end{proof}
 
Putting the preceding bounds in Lemmas~\ref{lem:inf-upper-bound} and \ref{lem:inf-lower-bound}
together, we arrive at
\begin{align}
\|\bm{\Delta}_{t}\|_{\infty} & \leq 30 \sqrt{\frac{\big(\log^{4}T\big)\big(\log\frac{|\mathcal{S}||\mathcal{A}|T}{\delta}\big)}{\gamma^2(1-\gamma)^{4}T}\bigg(1+\max_{\frac{t}{2}\le i<t}\|\bm{\Delta}_{i}\|_{\infty}\bigg)}
\label{eq:Delta-t-inf-recursion-inf-outline}
\end{align}
for all  $t\ge\frac{T}{c_{2}\log T}$ 
with probability exceeding $1-2\delta$, which forms the crux of our analysis. 
Employing elementary analysis tailored to the above recursive relation, one can demonstrate that
\begin{align}
	\|\bm{\Delta}_{T}\|_{\infty} \leq O\bigg( \sqrt{\frac{\big(\log^{4}T\big)\big(\log\frac{|\mathcal{S}||\mathcal{A}|T}{\delta}\big)}{(1-\gamma)^{4}T}} +\frac{\big(\log^{4}T\big)\big(\log\frac{|\mathcal{S}||\mathcal{A}|T}{\delta}\big)}{(1-\gamma)^{4}T} 
	\bigg) \label{eq:Delta-T-UB-final-inf}
\end{align}
with probability at least $1-2\delta$, which in turn allows us to establish the advertised result under the assumed sample size condition. The details are deferred to Appendix~\ref{sec:combine_ub_lb}.

 \subsection{Proof outline for Theorem~\ref{thm:LB-example}}
 \label{sec:MDP-construction-hard}
 
 \paragraph{Construction of a hard instance with 4 states and 2 actions.}
Let us construct an MDP $\mathcal{M}_{\mathsf{hard}}$ with state space $\cS=\{0,1,2,3\}$ (see a pictorial illustration in Figure~\ref{fig:hard-MDP}). We shall denote by $\cA_s$ the action space associated with state $s$. 
The probability transition kernel and reward function of $\mathcal{M}_{\mathsf{hard}}$ are specified as follows 
\begin{subequations}
\label{eq:construction-hard-MDP}
\begin{align}
	& \mathcal{A}_0=\{1\},  & P(0 \mymid 0, 1) &= 1, &\quad && r(0, 1) = 0, \\
	& \mathcal{A}_1=\{1,2\},  & P(1 \mymid 1, 1) &= p,   & P(0 \mymid 1, 1) = 1-p,   && r(1, 1) = 1,  \\
	& \quad   & P(1 \mymid 1, 2) &= p,  & P(0 \mymid 1, 2) = 1-p,  && r(1, 2) = 1, \\
	& \mathcal{A}_2=\{1\},  & P(2 \mymid 2, 1) &= p,  & P(0 \mymid 2,1) = 1-p,  && r(2,1) = 1, \\
	& \mathcal{A}_3=\{1\},  & P(3 \mymid 3, 1) &= 1,  &\quad && r(3,1) = 1, 
\end{align}
\end{subequations}
where the parameter $p$  is taken to be 
\begin{align}
	\label{defn-parameter-p}
	p = \frac{4\gamma-1}{3\gamma} .
\end{align}
%
\begin{figure}[t]
	\centering{\includegraphics[width=0.6\textwidth]{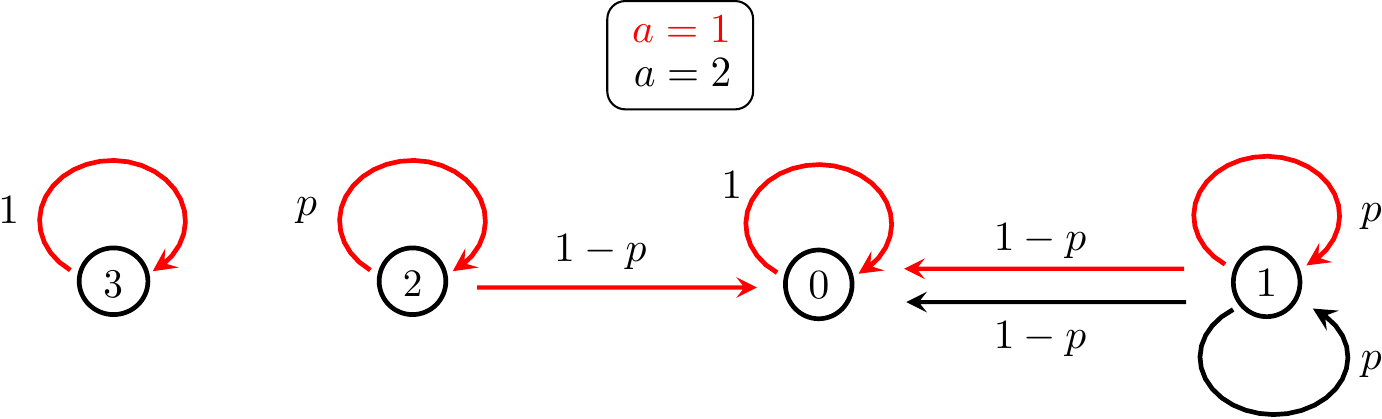}}
	\label{fig:hard-MDP}
	\caption{The constructed hard MDP instance used in the analysis of Theorem~\ref{thm:LB-example}, where $p= \frac{4\gamma-1}{3\gamma}$ and the specifications are described in \eqref{eq:construction-hard-MDP}.}
\end{figure}

Before moving forward to analyze the behavior of Q-learning, we first characterize the optimal value function and Q-function of this MDP; the proof is postponed to Section~\ref{sec:proof-lem:optimal-V-Q-hardMDP}. 
\begin{lemma}
\label{lem:optimal-V-Q-hardMDP}
Consider the MDP $\mathcal{M}_{\mathsf{hard}}$ constructed in \eqref{eq:construction-hard-MDP}. One has
\begin{subequations}
\label{eq:optimal-V-Q-hardMDP}
\begin{align}
&V^{\star}(0) = Q^{\star}(0,1) = 0; \\
&V^{\star}(1) = Q^{\star}(1, 1) = Q^{\star}(1, 2) = V^{\star}(2) = Q^{\star}(2,1) = \frac{1}{1-\gamma p}= \frac{3}{4(1-\gamma)}; \\
&V^{\star}(3) = Q^{\star}(3, 1) = \frac{1}{1-\gamma}.
\end{align}
\end{subequations}
\end{lemma}

Recognizing the elementary decomposition
\begin{equation}
	\mathbb{E}\left[\big(V^{\star}(s)-V_{T}(s)\big)^{2}\right]=\big(\mathbb{E}\left[V^{\star}(s)-V_{T}(s)\right]\big)^{2}+\mathsf{Var}\big(V_{T}(s)\big)
	\label{eq:bias-variance-decomposition}
\end{equation}
for any state $s$, our proof consists of lower bounding either the squared bias term $\big( \mathbb{E}[V^{\star}(s)-V_{T}(s)] \big)^2$ or the variance term $\mathsf{Var}\big(V_{T}(s)\big)$. 
In short,  we shall primarily analyze the dynamics w.r.t.~state 2 to handle the case when the learning rates are either too small or too large, and analyze the dynamics w.r.t.~state 1 to cope with the case with medium learning rates (with state 3 serving as a helper state to simplify the analysis). The latter case---corresponding to the learning rates adopted in establishing the upper bounds---is the most challenging: critically, from state 1 the agent can take one of two identical actions, whose value tends to be estimated with a high positive bias due to maximizing over the empirical state-action values, highlighting the well-recognized ``over-estimation'' issue of Q-learning  in practice \citep{hasselt2010double}.  The complete proof is deferred to Appendix~\ref{sec:lower-bounds}.

\paragraph{Extension: lower bounds for larger $|\cS|$ and $|\cA|$.} 
For pedagogical reasons, the hard instance \eqref{eq:construction-hard-MDP} constructed above contains no more than 4 states and 2 actions (as the focus has been to unveil sub-optimal dependency on the effective horizon). 
As it turns out, one can straightforwardly extend it to cover larger state and action spaces, 
with a more general hard instance constructed as follows.  
\begin{itemize}
	\item We begin by generating the following sub-MDP, denoted by $\mathcal{M}_{\mathsf{sub}}$, which comprises 4 states $\{1,2,3,4\}$ and no more than $|\cA|\geq 2$ actions:  
\begin{subequations}
\label{eq:construction-hard-MDP-large}
\begin{align}
	& \mathcal{A}_0=\{1\},  & P(0 \mymid 0, 1) &= 1, &\quad && r(0, 1) = 0, \\
	& \mathcal{A}_1=\{1,\ldots, |\cA|\},  & P(1 \mymid 1, a) &= p,   & P(0 \mymid 1, a) = 1-p,   && r(1, a) = 1, &&&\forall a\in \cA_1 \\
	& \mathcal{A}_2=\{1\},  & P(2 \mymid 2, 1) &= p,  & P(0 \mymid 2,1) = 1-p,  && r(2,1) = 1, \\
	& \mathcal{A}_3=\{1\},  & P(3 \mymid 3, 1) &= 1,  &\quad && r(3,1) = 1, 
\end{align}
\end{subequations}
where $p$ is still set according to \eqref{defn-parameter-p}. 

	\item The full MDP $\mathcal{M}_{\mathsf{full}}$ is then constructed by generating $|\cS|/4$ independent copies of $\mathcal{M}_{\mathsf{sub}}$. 

\end{itemize}
As can be easily verified (which we omit here for the sake of brevity), 
our analysis developed for the smaller MDP \eqref{eq:construction-hard-MDP} 
is directly applicable to studying the more general $\mathcal{M}_{\mathsf{full}}$, 
revealing that the lower bound \eqref{eq:lower-bounding-VT-thm-sync} w.r.t.~the iteration number $T$ remains valid. 
Recognizing that the total sample size scales as $|\cS||\cA|T$, we have established a general sample complexity lower bound $\frac{|\cS||\cA|}{(1-\gamma)^4\varepsilon^2}$ for synchronous Q-learning to yield $\varepsilon$-accuracy.

%% file: AsynQ.tex
\section{Extension: sample complexity of asynchronous Q-learning}
\label{sec:AsynQ}

Moving beyond the synchronous setting, another scenario of practical importance is the case where the acquired samples take the form of a single Markovian trajectory \citep{tsitsiklis1994asynchronous}.  
In this section, we extend our analysis framework for synchronous Q-learning to accommodate Markovian non-i.i.d.~samples.

\subsection{Markovian samples and asynchronous Q-learning}

\paragraph{Markovian sample trajectory.} Suppose that we obtain a Markovian sample trajectory $\{(s_t, a_t,r_t)\}_{t=0}^{\infty}$, which is generated by the MDP of interest when a stationary behavior policy $\pib$ is employed; in other words, 
\begin{equation}
	a_t \sim \pib(\cdot\mymid s_t),  \quad r_t = r(s_t,a_t),\quad s_{t+1} \sim P(\cdot\mymid s_t, a_t),\qquad t\geq 0. 
\end{equation}
When $\pib$ is stationary, the trajectory $\{(s_t,a_t)\}_{t=0}^{\infty}$ can be viewed as a sample path of a time-homogeneous Markov chain; in what follows, we shall denote by $\mu_{\pib}$ the stationary distribution of this Markov chain.  
Note that the behavior policy $\pib$ can often be quite different from the target optimal policy $\pi^{\star}$.

\paragraph{Asynchronous Q-learning.} 

In the presence of a single Markovian sample trajectory, the Q-learning algorithm implements the following iterative update rule
\begin{subequations}
\label{eq:async-Q-update-rule}
\begin{align}
Q_{t}(s_{t-1},a_{t-1}) & =(1-\eta_{t})Q_{t-1}(s_{t-1},a_{t-1})+\eta_{t}\Big\{ r(s_{t-1},a_{t-1})+\gamma\max_{a'\in\mathcal{A}}Q_{t-1}(s_{t},a')\Big\} , \\
Q_{t}(s,a) & =Q_{t-1}(s,a)\qquad\text{for all }(s,a)\neq(s_{t-1},a_{t-1})
\end{align}
\end{subequations}
for all $t\geq 1$, where $0<\eta_t\leq 1$ stands for the learning rate at time $t$. It is often referred to as {\em asynchronous Q-learning}, as only a single state-action pair is updated in each iteration (in contrast, synchronous Q-learning updates all state-action pairs simultaneously in each iteration). This also leads to the following estimate for the value function at time $t$:
\begin{align}
	V_t(s) \coloneqq \max_{a\in \cA} Q_t(s,a) \qquad \text{for all }s\in \mathcal{S}.
	\label{eq:Vt-async-Q-defn}
\end{align}
As can be expected, the presence of Markovian non-i.i.d.~data considerably complicates the analysis for asynchronous Q-learning.

\paragraph{Assumptions.}
In order to ensure sufficient coverage of the sample trajectory over the state/action space, we make the following assumption throughout this section, which is also commonly imposed in prior literature. 
\begin{assumption}
	The Markov chain induced by the behavior policy $\pib$ is uniformly ergodic.\footnote{See \citet[Section~1.2]{paulin2015concentration} for the definition of uniform ergodicity.} 
\end{assumption}
In addition, there are two crucial quantities concerning the sample trajectory that dictate the performance of asynchronous Q-learning. The first one is the minimum state-action occupancy probability of the sample trajectory, defined formally as 
\begin{equation}
	\label{defn:mu-min}
	\mumin \defn \min_{(s,a)\in \cS\times \cA} \mu_{\pib}(s,a). 
\end{equation}
This metric captures the information bottleneck incurred by the least visited state-action pair. The second key quantity is the mixing time associated with the sample trajectory, denoted by
\begin{equation}
	\label{defn-tmix}
	\tmix 
	\defn \min \Big\{ t \,\Big| \max_{(s,a)\in \cS\times \cA} d_{\mathsf{TV}}\big( P^t(\cdot\mymid s,a), \mu_{\pib} \big) \leq \frac{1}{4} \Big\}.
\end{equation}
Here, $d_{\mathsf{TV}}(\mu,\nu)\coloneqq \frac{1}{2}\sum_{x\in \mathcal{X}} |\mu(x) - \nu(x)|$ indicates the total variation distance between two measures $\mu$ and $\nu$ over $\mathcal{X}$ \citep{tsybakov2009introduction}, whereas $P^t(\cdot\mymid s,a)$ stands for the distribution of $(s_t,a_t)$ when the sample trajectory is initialized at $(s_0,a_0)=(s,a)$. 
In words, the mixing time reflects the time required for the Markov chain to become nearly independent of the initial states. See \citet[Section 2]{li2020sample} for a more detailed account of these quantities and assumptions.

\subsection{Sample complexity of asynchronous Q-learning}

While a number of previous works have been dedicated to understanding the performance of asynchronous Q-learning, 
its sample complexity bound remains loose when it comes to the dependency on the effective horizon $\frac{1}{1-\gamma}$. 
Encouragingly, the analysis framework laid out in this paper allows us to tighten the dependency on $\frac{1}{1-\gamma}$, 
as stated below. 
\begin{theorem}
\label{thm:infinite-horizon-asyn}
Consider any $\delta\in(0,1)$,  $\varepsilon\in(0,1]$, and $\gamma\in [1/2,1)$. Suppose that for any $0\leq t\leq T$, the
learning rates satisfy
\begin{subequations}
\label{eq:thm-infinite-horizon-condition-asyn}
\begin{equation}
\eta_t \equiv \eta = \frac{c_1\log^3 T}{(1-\gamma) T\mumin} \label{eq:thm:infinite-horizaon-eta-asyn}
\end{equation}
for some universal constants $0<c_{1}\le 1$. 
Assume that the total number of iterations $T$ obeys
\begin{equation}
	T\ge\frac{c_{2} \log^2\frac{|\mathcal{S}||\mathcal{A}|T}{\delta}}{\mumin}\max\bigg\{\frac{\log^{3}T}{(1-\gamma)^{4}\varepsilon^{2}}, \frac{t_{\mathsf{mix}}}{1-\gamma}\bigg\}
	\label{eq:thm:infinite-horizon-T-asyn}
\end{equation}
for some sufficiently large universal constant $c_{2}>0$. 
\end{subequations}
If the initialization obeys $0\leq {Q}_{0}(s,a) \leq\frac{1}{1-\gamma}$ for all $(s,a)\in \cS\times \cA$, 
then asynchronous Q-learning (cf.~\eqref{eq:async-Q-update-rule}) satisfies
\[
	\max_{(s,a)\in \cS\times \cA} \big| {Q}_{T}(s,a) -  {Q}^{\star}(s,a) \big| \le\varepsilon
\]
 with probability at least $1-\delta$. 
\end{theorem}

\begin{remark} \label{remark:expected-error-asynQ}
	Similar to Remark~\ref{remark:expected-error-inf-td} and Remark~\ref{remark:expected-error-inf}, one can immediately translate 
	the above
high-probability result into the following  mean estimation error bound: 
	\begin{equation}
		\mathbb{E}\Big[ \max_{s,a} \big| {Q}_{T}(s,a) -  {Q}^{\star}(s,a) \big| \Big] \le \varepsilon (1-\delta)+\delta \frac{1}{1-\gamma}
		\leq 2\varepsilon,
	\end{equation}
	which holds as long as $T\ge\frac{c_{2} \log^2\frac{|\mathcal{S}||\mathcal{A}|T}{\varepsilon(1-\gamma)}}{\mumin}\max\Big\{\frac{\log^{3}T}{(1-\gamma)^{4}\varepsilon^{2}}, \frac{t_{\mathsf{mix}}}{1-\gamma}\Big\}$ for some large enough constant $c_2>0$. 
\end{remark}

This theorem demonstrates that with high probability, the total sample size needed for asynchronous Q-learning to yield entrywise $\varepsilon$ accuracy is 
\begin{align}
	\widetilde{O}\bigg( \frac{1}{\mumin(1-\gamma)^{4}\varepsilon^{2}} + \frac{t_{\mathsf{mix}}}{\mumin(1-\gamma)} \bigg),
	\label{eq:sample-size-async-Q}
\end{align}
provided that the learning rates are taken to be some proper constant (see~\eqref{eq:thm:infinite-horizaon-eta-asyn}). 
The first term in \eqref{eq:sample-size-async-Q} resembles our sample complexity characterization of synchronous Q-learning (cf.~\eqref{eq:sample-complexity-sync-Q-learning}), except that we replace the number $|\cS||\cA|$ of state-action pairs  in \eqref{eq:sample-complexity-sync-Q-learning} with $1/\mumin$ in order to account for non-uniformity across state-action pairs.  
The second term in \eqref{eq:sample-size-async-Q} is nearly independent of the target accuracy (except for some logarithmic scaling), 
and can be viewed as the burn-in time taken for asynchronous Q-learning to mimic synchronous Q-learning despite Markovian data.

We now pause to  compare Theorem~\ref{thm:infinite-horizon-asyn} with prior non-asymptotic theory for asynchronous Q-learning. As far as we know, all existing sample complexity bounds \citep{beck2012error,qu2020finite,li2020sample,even2003learning,chen2021lyapunov} scale at least as $\frac{1}{(1-\gamma)^5}$ in terms of the dependency on the effective horizon, with Theorem~\ref{thm:infinite-horizon-asyn} being the first result to sharpen this dependency to $\frac{1}{(1-\gamma)^4}$. 
In particular, our sample complexity bound strengthens the state-of-the-art result \citet{li2020sample} by a factor up to $\frac{1}{1-\gamma}$,  
while improving upon \citet{qu2020finite} by a factor of at least $\frac{|\cS||\cA|}{1-\gamma} \min\big\{\tmix, \frac{1}{(1-\gamma)^3\varepsilon^2} \big\}$.\footnote{The sample complexity of \citet{li2020sample} scales as $\widetilde{O}\big( \frac{1}{\mumin(1-\gamma)^{5}\varepsilon^{2}} + \frac{t_{\mathsf{mix}}}{\mumin(1-\gamma)} \big)$, 
while	the sample complexity of \citet{qu2020finite} scales as $\widetilde{O}\big( \frac{\tmix}{\mumin^2 (1-\gamma)^5 \varepsilon^2} \big)$. It is worth noting that $1/\mumin \geq |\cS||\cA|$ and is therefore a large factor.}

Before concluding this section, 
we note that for a large enough sample size, the first term $\frac{1}{\mumin(1-\gamma)^{4}\varepsilon^{2}}$ 
in \eqref{eq:sample-size-async-Q} is essentially unimprovable (up to logarithmic factor). 
To make precise this statement, we develop  a matching algorithm-dependent lower bound as follows, 
which parallels Theorem~\ref{thm:LB-example} previously developed for the synchronous case. 
\begin{theorem}
\label{thm:LB-asynQ}
	Consider any $0.95\leq \gamma < 1$. 
	Suppose that $\mumin \le \frac{1}{c_3\log^2 T}$ and $T\geq \frac{c_3\log^3 T}{\mumin(1-\gamma)^7}$ for some sufficiently large constant $c_3>0$. 
	Assume that the initialization is $Q_0 \equiv 0$, and that the learning rates are taken to be $\eta_t \equiv \eta$ for all $t\geq 0$. 
Then there exist a $\gamma$-discounted MDP with $|\cS|=4$ and $|\cA|=3$ 
	and a behavior policy such that 
	(i) the minimum state-action occupancy probability of the sample trajectory is given by $\mumin$, and (ii) the asynchronous Q-learning update rule~\eqref{eq:async-Q-update-rule}---for any $\eta \in (0,1)$---obeys
\begin{equation}
	\max_{s, a} \mathbb{E} \Big[ \big| Q_T(s, a) - Q^{\star}(s, a) \big|^2 \Big]  
	\geq  \frac{c_{\mathsf{lb}}}{\mumin(1-\gamma)^4T \log^3 T} ,
	\label{eq:lower-bounding-VT-thm-sync}
\end{equation}
where $c_{\mathsf{lb}} > 0$ is some universal constant. 
\end{theorem}
In words, Theorem~\ref{thm:LB-asynQ} asserts that, for large enough sample size $T$, 
in general one cannot hope to achieve $\ell_{\infty}$-based $\varepsilon$-accuracy using fewer than $\widetilde{O}\big(\frac{1}{\mumin (1-\gamma)^4\varepsilon^2}\big)$ samples, thus confirming the sharpness of our upper bound.  
The proof of this theorem can be found in Appendix~\ref{sec:lower-bound-async-Q}.

%% file: conclusions.tex
\section{Concluding remarks}

In this paper, we have settled the sample complexity of synchronous Q-learning in $\gamma$-discounted infinite-horizon MDPs, 
which is shown to be 
on the order of $ \widetilde{O}\big( \frac{|\cS|}{(1-\gamma)^3\varepsilon^2} \big) $ when $|\cA|=1$ and $ \widetilde{O}\big( \frac{|\cS||\cA|}{(1-\gamma)^4\varepsilon^2} \big) $ when $|\cA|\geq 2$.
A matching lower bound has been developed when $|\cA|\geq 2$ through studying the dynamics of Q-learning on a hard MDP instance,
which unveils the negative impact of an inevitable over-estimation issue. 
Our theory has been further extended to accommodate asynchronous Q-learning, resulting in tight dependency of the sample complexity on the effective horizon.  
The analysis framework developed herein---which exploits novel error decompositions
and variance control that differ substantially from prior approaches---might suggest a plausible path towards sharpening the sample complexity of, as well as understanding the algorithmic bottlenecks for,  other model-free algorithms (e.g., double Q-learning \citep{hasselt2010double}).

%% file: preliminary-Freedman.tex
\section{Freedman's inequality}

The analysis of this work relies heavily on Freedman's inequality
\citep{freedman1975tail}, which is an extension of the Bernstein inequality
and allows one to establish concentration results for martingales.
For ease of presentation, we include a user-friendly version of Freedman's
inequality as follows.

\begin{theorem}\label{thm:Freedman}
Suppose that $Y_{n}=\sum_{k=1}^{n}X_{k}\in\mathbb{R}$,
where $\{X_{k}\}$ is a real-valued scalar sequence obeying 
\[
\left|X_{k}\right|\leq R\qquad\text{and}\qquad\mathbb{E}\left[X_{k}\mid\left\{ X_{j}\right\} _{j:j<k}\right]=0\quad\quad\quad\text{for all }k\geq1.
\]
Define
\[
W_{n}\coloneqq\sum_{k=1}^{n}\mathbb{E}_{k-1}\left[X_{k}^{2}\right],
\]
where we write $\mathbb{E}_{k-1}$ for the expectation conditional
on $\left\{ X_{j}\right\} _{j:j<k}$. Then for any given $\sigma^{2}\geq0$,
one has
\begin{equation}
\mathbb{P}\left\{ \left|Y_{n}\right|\geq\tau\text{ and }W_{n}\leq\sigma^{2}\right\} \leq2\exp\left(-\frac{\tau^{2}/2}{\sigma^{2}+R\tau/3}\right).\label{eq:Freedmans-general}
\end{equation}

In addition, suppose that $W_{n}\leq\sigma^{2}$ holds deterministically.
For any positive integer $K\geq1$, with probability at least $1-\delta$
one has
\begin{equation}
\left|Y_{n}\right|\leq\sqrt{8\max\Big\{ W_{n},\frac{\sigma^{2}}{2^{K}}\Big\}\log\frac{2K}{\delta}}+\frac{4}{3}R\log\frac{2K}{\delta}.\label{eq:Freedman-random}
\end{equation}
\end{theorem}\begin{proof}See \cite{freedman1975tail,tropp2011freedman}
for the proof of (\ref{eq:Freedmans-general}). As an immediate consequence
of (\ref{eq:Freedmans-general}), one has
\begin{equation}
\mathbb{P}\left\{ \left|Y_{n}\right|\geq\sqrt{4\sigma^{2}\log\frac{2}{\delta}}+\frac{4}{3}R\log\frac{2}{\delta}\text{ and }W_{n}\leq\sigma^{2}\right\} \leq\delta.\label{eq:Freedman-special}
\end{equation}

Next, we turn attention to (\ref{eq:Freedman-random}). Consider any
positive integer $K$. As can be easily seen, the event 
\[
\mathcal{H}_{K}
\coloneqq
\Bigg\{\left|Y_{n}\right|\ge\sqrt{8\max\Big\{ W_{n},\frac{\sigma^{2}}{2^{K}}\Big\}\log\frac{2K}{\delta}}+\frac{4}{3}R\log\frac{2K}{\delta}\Bigg\}
\]
 is contained within the union of the following $K$ events 
\[
\mathcal{H}_{K}\subseteq\bigcup_{0\leq k<K}\mathcal{B}_{k},
\]
where we define
\begin{align*}
\mathcal{B}_{k} & \coloneqq\left\{ \left|Y_{n}\right|\ge\sqrt{\frac{4\sigma^{2}}{2^{k-1}}\log\frac{2K}{\delta}}+\frac{4}{3}R\log\frac{2K}{\delta}\text{ and }\frac{\sigma^{2}}{2^{k}}\le W_{n}\leq\frac{\sigma^{2}}{2^{k-1}}\right\} ,\qquad1\leq k\leq K-1,\\
\mathcal{B}_{0} & \coloneqq\left\{ \left|Y_{n}\right|\ge\sqrt{\frac{4\sigma^{2}}{2^{K-1}}\log\frac{2K}{\delta}}+\frac{4}{3}R\log\frac{2K}{\delta}\text{ and }W_{n}\leq\frac{\sigma^{2}}{2^{K-1}}\right\} .
\end{align*}
Invoking inequality \eqref{eq:Freedman-special} with $\sigma^{2}$ set to be $\frac{\sigma^{2}}{2^{k-1}}$ and $\delta$ set to be $\frac{\delta}{K}$, 
we arrive at $\mathbb{P}\left\{ \mathcal{B}_{k}\right\} \leq\delta/K$. 
Taken this fact together with the union bound gives
\[
\mathbb{P}\left\{ \mathcal{H}_{K}\right\} \leq\sum_{k=0}^{K-1}\mathbb{P}\left\{ \mathcal{B}_{k}\right\} \leq\delta.
\]
This concludes the proof.
\end{proof}

%% file: infinite-synchronous-Q.tex
\section{Upper bounds for Q-learning (Theorem \ref{thm:infinite-horizon-simple})}

\label{sec:Analysis:-infinite-horizon-MDPs}

In this section, we fill in the details for the proof idea outlined in Section \ref{sec:proof-outline-thm:infinite-horizon} for synchronous Q-learning.  
In fact, our proof strategy leads to  a more general version that accounts for
the full $\varepsilon$-range $\varepsilon \in \big(0,\frac{1}{1-\gamma}\big]$, as stated below. 
\begin{theorem}\label{thm:infinite-horizon}
Consider any $\gamma\in (0,1)$ and any $\varepsilon\in\big(0,\frac{1}{1-\gamma}\big]$.
Theorem \ref{thm:infinite-horizon-simple} continues to hold if
\begin{equation}
	T \ge \frac{c_{3}\big(\log^{4}T\big)\big(\log\frac{|\mathcal{S}||\mathcal{A}|T}{\delta}\big)}{\gamma^2(1-\gamma)^{4}\min\{\varepsilon^{2},\varepsilon\}}
	\label{eq:sample-size-condition-complete-inf}
\end{equation}
for some large enough universal constant $c_{3}>0$. 
\end{theorem}
\begin{remark} Clearly, Theorem \ref{thm:infinite-horizon} subsumes Theorem \ref{thm:infinite-horizon-simple} as a special case. \end{remark}

As one can anticipate, the proof of Theorem \ref{thm:infinite-horizon} for Q-learning includes many key ingredients for establishing Theorem~\ref{thm:policy-evaluation} for TD learning. We will elaborate on how to modify the proof argument to establish Theorem~\ref{thm:policy-evaluation} in Section~\ref{sec:TD-learning-analysis}.

\subsection{Preliminaries}
\label{sec:infinite_prelim}

To begin with, we gather a few elementary facts that shall be used multiple times in the proof. 

\paragraph{Ranges of $\bm{Q}_{t}$ and $\bm{V}_{t}$.} When properly
initialized, the Q-function estimates and the value function estimates
always fall within a suitable range, as asserted by the following
lemma. 

\begin{lemma}
\label{lemma:non-negativity-Qt-Vt}
Suppose that $0\leq\eta_{t}\leq1$
for all $t\geq0$. Assume that $\bm{0}\leq\bm{Q}_{0}\leq\frac{1}{1-\gamma}\bm{1}$.
Then for any $t\geq0$, 
\begin{equation}
	\bm{0}\le\bm{Q}_{t}\le\frac{1}{1-\gamma}\bm{1}\qquad\text{and}\qquad\bm{0}\le\bm{V}_{t}\le\frac{1}{1-\gamma}\bm{1}.
	\label{eq:Q-i-LB-UB}
\end{equation}
\end{lemma}
\begin{proof}
We shall prove this by induction. First,
our initialization trivially obeys (\ref{eq:Q-i-LB-UB}) for $t=0$.
Next, suppose that (\ref{eq:Q-i-LB-UB}) is true for the $(t-1)$-th
iteration, namely, 
\begin{equation}
	\bm{0}\le\bm{Q}_{t-1}\le\frac{1}{1-\gamma}\bm{1}\qquad\text{and}\qquad\bm{0}\le\bm{V}_{t-1}\le\frac{1}{1-\gamma}\bm{1},
	\label{eq:induction-non-negativity-tminus1}
\end{equation}
and we intend to justify the claim for the $t$-th iteration. Recognizing
that $\bm{0}\leq\bm{r}\leq\bm{1}$, $\bm{P}_{t}\geq\bm{0}$ and $\|\bm{P}_{t}\|_{1}=1$,
one can straightforwardly see from the update rule (\ref{eq:iteration-rule-infinite})
and the induction hypothesis \eqref{eq:induction-non-negativity-tminus1}
that
\[
\bm{Q}_{t}=(1-\eta_{t})\bm{Q}_{t-1}+\eta_{t}(\bm{r}+\gamma\bm{P}_{t}\bm{V}_{t-1})\geq\bm{0} 
\]
and
\begin{align*}
\bm{Q}_{t} & =(1-\eta_{t})\bm{Q}_{t-1}+\eta_{t}(\bm{r}+\gamma\bm{P}_{t}\bm{V}_{t-1})\\
 & \leq(1-\eta_{t})\,\|\bm{Q}_{t-1}\|_{\infty}\bm{1}+\eta_{t}\big(\|\bm{r}\|_{\infty}+\gamma\|\bm{P}_{t}\|_{1}\|\bm{V}_{t-1}\|_{\infty}\big)\bm{1}\\
 & \leq(1-\eta_{t})\frac{1}{1-\gamma}\bm{1}+\eta_{t}\Big(1+\frac{\gamma}{1-\gamma}\Big)\bm{1}=\frac{1}{1-\gamma}\bm{1}.
\end{align*}
In addition, from the definition $V_{t}(s):=\max_{a}Q_{t}(s,a)$ for
all $t\geq0$ and all $s\in\mathcal{S}$, it is easily seen that 
\[
\bm{0}\le\bm{V}_{t}\le\frac{1}{1-\gamma}\bm{1},
\]
thus establishing (\ref{eq:Q-i-LB-UB}) for the $t$-th iteration.
Applying the induction argument then concludes the proof. \end{proof}

As a result of Lemma \ref{lemma:non-negativity-Qt-Vt}
and the fact $\bm{0}\le\bm{Q}^{\star}\le\frac{1}{1-\gamma}\bm{1}$,
we have
\begin{equation}
\|\bm{Q}_{t}-\bm{Q}^{\star}\|_{\infty}\le\frac{1}{1-\gamma}\qquad\text{for all }t\geq0, \label{eq:Delta-t-range-infinite}
\end{equation}
which also confirms that $0\leq \varepsilon\leq\frac{1}{1-\gamma}$ is the full $\varepsilon$-range we need to consider. 
Further, we make note of a direct
consequence of the claimed iteration number (\ref{eq:sample-size-condition-complete-inf})
when $\varepsilon\leq\frac{1}{1-\gamma}$: 
\begin{equation}
	T=\frac{c_{3}\big(\log^{4}T\big)\big(\log\frac{|\mathcal{S}||\mathcal{A}|T}{\delta}\big)}{(1-\gamma)^{4}\min\{\varepsilon,\varepsilon^{2}\}}
	\geq\frac{c_{3}\big(\log^{4}T\big)\big(\log\frac{|\mathcal{S}||\mathcal{A}|T}{\delta}\big)}{(1-\gamma)^{3}},\label{eq:T-bound-infinite}
\end{equation}
which will be useful for subsequent analysis. 

\paragraph{Several facts regarding the learning rates.} Next, we
gather a couple of useful bounds regarding the learning rates $\{\eta_{t}\}$.
To begin with, we find it helpful to introduce the following related
quantities introduced previously in \eqref{def:eta-i-t-simple}:
\begin{equation}
\eta_{i}^{(t)}\coloneqq\begin{cases}
\prod_{j=1}^{t}(1-\eta_{j}), & \text{if }i=0,\\
\eta_{i}\prod_{j=i+1}^{t}(1-\eta_{j}), & \text{if }0<i<t,\\
\eta_{t}, & \text{if }i=t. 
\end{cases}\label{def:eta-i-t}
\end{equation}
%
We now take a moment to bound $\eta_{i}^{(t)}$. From our assumption \eqref{eq:thm:infinite-horizaon-eta}
and the condition \eqref{eq:T-bound-infinite}, we know that the learning
rate obeys 
\begin{equation}
\frac{1}{2c_{1}(1-\gamma)T/\log^{3}T}\leq\frac{1}{1+c_{1}(1-\gamma)T/\log^{3}T}\le\eta_{t}\le\frac{1}{1+c_{2}(1-\gamma)t/\log^{3}T}\leq\frac{1}{c_{2}(1-\gamma)t/\log^{3}T}\label{eq:eta-t-UB-LB-inf}
\end{equation}
for some constants $c_{1},c_{2}>0$. Recalling that
\begin{equation}
\label{eq:defn-beta-c4}
\beta\coloneqq\frac{c_{4}(1-\gamma)}{\log T}
\end{equation}
for some universal constant $c_{4}>0$ and considering any $t$ obeying
\begin{equation}
t\ge\frac{T}{c_{2}\log T},\label{eq:t-LB-c2-inf}
\end{equation}
we shall bound $\eta_{i}^{(t)}$ by looking at two cases separately. 
\begin{itemize}
\item For any $0\leq i\le(1-\beta)t$, we can use \eqref{eq:eta-t-UB-LB-inf}
to show that
\begin{subequations}\label{eq:eta-i-t-UB12}
\begin{align}
\eta_{i}^{(t)} & \le\Big(1-\frac{1}{2c_{1}(1-\gamma)T/\log^{3}T}\Big)^{\beta t}\leq\Big(1-\frac{1}{2c_{1}(1-\gamma)T/\log^{3}T}\Big)^{\frac{c_{4}(1-\gamma)T}{c_{2}\log^{2}T}} \notag\\
 & =\Bigg(\Big(1-\frac{1}{2c_{1}(1-\gamma)T/\log^{3}T}\Big)^{\frac{2c_{1}(1-\gamma)T}{\log^{3}T}}\Bigg)^{\frac{c_{4}\log T}{2c_{1}c_{2}}}<\frac{1}{2T^{2}},
	\label{eq:eta-i-t-UB1}
\end{align}
where the last inequality holds as long as $c_{1}c_{2}\le c_{4}/8$. 

\item When it comes to the case with $i>(1-\beta)t\geq t/2$, one can upper bound
\begin{equation}
\eta_{i}^{(t)}\le\eta_{i}\leq\frac{1}{c_{2}(1-\gamma)i/\log^{3}T}<\frac{2}{c_{2}(1-\gamma)t/\log^{3}T}\leq\frac{2\log^{4}T}{(1-\gamma)T},\label{eq:eta-i-t-UB2}
\end{equation}
\end{subequations}
where we have used the constraint \eqref{eq:t-LB-c2-inf}. 
\end{itemize}
Moreover, the sum of $\eta_{i}^{(t)}$ over $i$ obeys
\begin{align}
\sum_{i=0}^{t}\eta_{i}^{(t)} & =\prod_{j=1}^{t}(1-\eta_{j})+\eta_{1}\prod_{j=2}^{t}(1-\eta_{j})+\eta_{2}\prod_{j=3}^{t}(1-\eta_{j})+\cdots+\eta_{t-1}(1-\eta_{t})+\eta_{t}\nonumber \\
 & =\prod_{j=2}^{t}(1-\eta_{j})+\eta_{2}\prod_{j=3}^{t}(1-\eta_{j})+\cdots+\eta_{t-1}(1-\eta_{t})+\eta_{t}=\cdots\nonumber \\
 & =(1-\eta_{t})+\eta_{t}=1.
	\label{eq:sum-eta-i-t-infinite}
\end{align}
Repeating the same argument further allows us to derive
\begin{align}
\sum_{i=\tau}^{t}\eta_{i}^{(t)} & = 1 - \prod_{j=\tau}^{t}(1-\eta_{j})
	\label{eq:sum-eta-i-t-infinite-tau}
\end{align}
for any $\tau \leq t$.



\subsection{Proof of Lemma~\ref{lem:inf-upper-bound} }
\label{sec:proof_inf_upper_bound}


We shall exploit the relation \eqref{eq:Delta-UB2} to prove this lemma. 
One of the key ingredients of our analysis lies in controlling the terms $\bm{\zeta}_{t}$
and $\bm{\xi}_{t}$ introduced in  \eqref{eq:Delta-UB2}, which in turn enables us to apply (\ref{eq:Delta-UB2})
recursively to control $\bm{\Delta}_{t}$. 

\paragraph{Step 1: bounding $\bm{\zeta}_{t}$.} 

We start by developing an upper bound on $\bm{\zeta}_{t}$ (cf.~\eqref{eq:Delta-UB2}) for
any $t$ obeying $\frac{T}{c_{2}\log T}\leq t\leq T$.
Invoking the preceding upper bounds \eqref{eq:eta-i-t-UB12} on $\eta_{i}^{(t)}$
implies that
\begin{align*}
\|\bm{\zeta}_{t}\|_{\infty} & \leq\eta_{0}^{(t)}\|\bm{\Delta}_{0}\|_{\infty}+t\max_{i\leq(1-\beta)t}\eta_{i}^{(t)}\max_{1\leq i\leq(1-\beta)t}\big(\|\bm{P}^{\pi_{i-1}}\bm{\Delta}_{i-1}\|_{\infty}+\|\bm{P}_{i}\bm{V}_{i-1}\|_{\infty}+\|\bm{P}\bm{V}_{i-1}\|_{\infty}\big)\\
 & \leq\eta_{0}^{(t)}\|\bm{\Delta}_{0}\|_{\infty}+t\max_{i\leq(1-\beta)t}\eta_{i}^{(t)}\max_{1\leq i\leq(1-\beta)t}\Big\{\|\bm{P}^{\pi_{i-1}}\|_{1}\|\bm{\Delta}_{i-1}\|_{\infty}+\big(\|\bm{P}_{i}\|_{1}+\|\bm{P}\|_{1}\big)\|\bm{V}_{i-1}\|_{\infty}\Big\}\\
 & \overset{(\mathrm{i})}{=}\eta_{0}^{(t)}\|\bm{\Delta}_{0}\|_{\infty}+t\max_{i\leq(1-\beta)t}\eta_{i}^{(t)}\max_{1\leq i\leq(1-\beta)t}\big(\|\bm{\Delta}_{i-1}\|_{\infty}+2\,\|\bm{V}_{i-1}\|_{\infty}\big)\\
 & \overset{(\mathrm{ii})}{\leq}\frac{1}{2T^{2}}\cdot\frac{1}{1-\gamma}+\frac{1}{2T^{2}}\cdot t\cdot\frac{3}{1-\gamma}\\
 & \le\frac{2}{(1-\gamma)T}.
\end{align*}
Here, (i) holds since $\|\bm{P}^{\pi_{i-1}}\|_{1}=\|\bm{P}_{i}\|_{1}=\|\bm{P}\|_{1}=1$
(as they are all probability transition matrices), whereas (ii) arises
from the previous bound (\ref{eq:eta-i-t-UB1}).

\paragraph{Step 2: bounding $\bm{\xi}_{t}$.} 

Moving on to the term $\bm{\xi}_{t}$, let us express it as
\[
\bm{\xi}_{t}=\sum_{i=(1-\beta)t+1}^{t}\bm{z}_{i}\qquad\text{with }\bm{z}_{i}\coloneqq\eta_{i}^{(t)}\gamma(\bm{P}_{i}-\bm{P})\bm{V}_{i-1},
\]
where the $\bm{z}_{i}$'s satisfy
\[
\mathbb{E}\left[\bm{z}_{i}\mymid\bm{V}_{i-1},\cdots,\bm{V}_{0}\right]=\bm{0}.
\]
This motivates us to invoke Freedman's inequality (see Theorem \ref{thm:Freedman})
to control $\bm{\xi}_{t}$ for any $t$ obeying $\frac{T}{c_{2}\log T}\leq t\leq T$.
Towards this, we need to calculate several quantities.
\begin{itemize}
\item First, it is seen that
\begin{align*}
B & :=\max_{(1-\beta)t<i\leq t}\|\bm{z}_{i}\|_{\infty}\leq\max_{(1-\beta)t<i\leq t}\|\eta_{i}^{(t)}(\bm{P}_{i}-\bm{P})\bm{V}_{i-1}\|_{\infty}\\
 & \leq\max_{(1-\beta)t<i\leq t}\eta_{i}^{(t)}\big(\|\bm{P}_{i}\|_{1}+\|\bm{P}\|_{1}\big)\|\bm{V}_{i-1}\|_{\infty}\leq\frac{4\log^{4}T}{(1-\gamma)^{2}T},
\end{align*}
where the last inequality is due to (\ref{eq:eta-i-t-UB2}), Lemma~\ref{lemma:non-negativity-Qt-Vt},
and the fact $\|\bm{P}_{i}\|_{1}=\|\bm{P}\|_{1}=1$.

\item Next, we turn to certain variance terms.  For any vector $\bm{a}=[a_j]$, let us use $\mathsf{Var}\big(\bm{a}\mymid\bm{V}_{i-1},\cdots,\bm{V}_{0}\big)$
to denote a vector whose $j$-th entry is given by $\mathsf{Var}\big(a_{j}\mymid\bm{V}_{i-1},\cdots,\bm{V}_{0}\big)$.
With this notation in place, and recalling the notation $\mathsf{Var}_{\bm{P}}(\bm{z})$ in (\ref{eq:defn-Var-P-V}), we obtain 
\begin{align}
\bm{W}_{t} & :=\sum_{i=(1-\beta)t+1}^{t}\mathsf{Var}\big(\bm{z}_{i}\mymid\bm{V}_{i-1},\cdots,\bm{V}_{0}\big) 
	=\gamma^{2}\sum_{i=(1-\beta)t+1}^{t}\big(\eta_{i}^{(t)}\big)^{2}\mathsf{Var}\Big( (\bm{P}_{i}-\bm{P})\bm{V}_{i-1} \mymid\bm{V}_{i-1} \Big) \nonumber \\
 & =\gamma^{2}\sum_{i=(1-\beta)t+1}^{t}\big(\eta_{i}^{(t)}\big)^{2}\mathsf{Var}_{\bm{P}}\big(\bm{V}_{i-1}\big)\nonumber \\
 & \le\Big(\max_{(1-\beta)t\leq i\leq t}\eta_{i}^{(t)}\Big)\Big(\sum_{i=(1-\beta)t+1}^{t}\eta_{i}^{(t)}\Big)\max_{(1-\beta)t\leq i<t}\mathsf{Var}_{\bm{P}}\big(\bm{V}_{i}\big)\nonumber \\
 & \leq\frac{2\log^{4}T}{(1-\gamma)T}\max_{(1-\beta)t\leq i<t}\mathsf{Var}_{\bm{P}}(\bm{V}_{i}),\label{eq:Wt-bound-Var-infinite}
\end{align}
where the last inequality relies on the previous bounds \eqref{eq:eta-i-t-UB2} and \eqref{eq:sum-eta-i-t-infinite}.
 
\item In the meantime, Theorem~\ref{lemma:non-negativity-Qt-Vt} leads us to the following
trivial upper bound:
\[
\big|\bm{W}_{t}\big|\leq\frac{2\log^{4}T}{(1-\gamma)T}\cdot\frac{1}{(1-\gamma)^2}\bm{1}
=\frac{2\log^{4}T}{(1-\gamma)^{3}T}\bm{1}\eqqcolon\sigma^{2}\bm{1}.
\]
By setting $K=\left\lceil 2\log_2\frac{1}{1-\gamma}\right\rceil $, one has
\begin{equation}
\frac{\sigma^{2}}{2^{K}}\leq\frac{2\log^{4}T}{(1-\gamma)T}.\label{eq:sigma-2K-UB-inf}
\end{equation}
\end{itemize}
With the above bounds in place, applying the Freedman inequality in
Theorem \ref{thm:Freedman} and invoking the union bound over all the $|\cS||\cA|$ entries
of $\bm{\xi}_{t}$ demonstrate that 
\begin{align*}
|\bm{\xi}_{t}| & \le
\sqrt{8\Big(\bm{W}_{t}+\frac{\sigma^{2}}{2^{K}}\bm{1}\Big)\log\frac{8|\mathcal{S}||\mathcal{A}|T\log\frac{1}{1-\gamma}}{\delta}}+\Big(\frac{4}{3}B\log\frac{8|\mathcal{S}||\mathcal{A}|T\log\frac{1}{1-\gamma}}{\delta}\Big)\cdot\bm{1}\\
 & \le\sqrt{16\Big(\bm{W}_{t}+\frac{2\log^{4}T}{(1-\gamma)T}\bm{1}\Big)\log\frac{|\mathcal{S}||\mathcal{A}|T}{\delta}}+\Big(3B\log\frac{|\mathcal{S}||\mathcal{A}|T}{\delta}\Big)\cdot\bm{1}\\
 & \le\sqrt{\frac{32\big(\log^{4}T\big)\big(\log\frac{|\mathcal{S}||\mathcal{A}|T}{\delta}\big)}{(1-\gamma)T}\Big(\max_{(1-\beta)t\leq i<t}\mathsf{Var}_{\bm{P}}(\bm{V}_{i})+\bm{1}\Big)}+\frac{12\big(\log^{4}T\big)\big(\log\frac{|\mathcal{S}||\mathcal{A}|T}{\delta}\big)}{(1-\gamma)^{2}T}\bm{1}
\end{align*}
with probability at least $1-\delta/T.$ Here, the second line holds
due to \eqref{eq:sigma-2K-UB-inf} and the fact $\log\frac{8|\mathcal{S}||\mathcal{A}|T\log\frac{1}{1-\gamma}}{\delta}\le2\log\frac{|\mathcal{S}||\mathcal{A}|T}{\delta}$
(cf.~\eqref{eq:T-bound-infinite}), whereas the last inequality makes
use of the relation \eqref{eq:Wt-bound-Var-infinite}.

\paragraph{Step 3: using the bounds on $\bm{\zeta}_{t}$ and $\bm{\xi}_{t}$
to control $\bm{\Delta}_{t}$.} Let us define
\begin{equation}
\bm{\varphi}_{t}\coloneqq 64\frac{\log^{4}T\log\frac{|\mathcal{S}||\mathcal{A}|T}{\delta}}{(1-\gamma)T}\bigg(\max_{\frac{t}{2}\le i\leq t}\mathsf{Var}_{\bm{P}}(\bm{V}_{i})+\bm{1}\bigg)\label{def:phi}
\end{equation}
%
In view of the upper bounds derived in Steps 1 and 2, and $\beta$ defined in \eqref{eq:defn-beta-c4}, 
we have---with probability exceeding
$1-\delta$---that
\begin{equation}
|\bm{\zeta}_{k}|+|\bm{\xi}_{k}|\leq\sqrt{\bm{\varphi}_{t}}\qquad\text{for all }~2t/3\le k\le t, \label{eq:zeta-xi-UB}
\end{equation}
provided that $T\geq \frac{c_9 (\log^4 T)\big( \log \frac{|\cS||\cA|T}{\delta}\big) }{ (1-\gamma)^3} $ for some sufficiently large constant $c_9>0$.
Substituting (\ref{eq:zeta-xi-UB}) into (\ref{eq:Delta-UB2}), we
can upper bound $\bm{\Delta}_{t}$ as follows
\begin{align}
\bm{\Delta}_{k} & \le\sqrt{\bm{\varphi}_{t}}+\sum_{i=(1-\beta)k+1}^{k}\eta_{i}^{(k)}\gamma\bm{P}^{\pi_{i-1}}\bm{\Delta}_{i-1}=\sqrt{\bm{\varphi}_{t}}+\sum_{i=(1-\beta)k}^{k-1}\eta_{i+1}^{(k)}\gamma\bm{P}^{\pi_{i}}\bm{\Delta}_{i}\qquad\text{for all }~2t/3\le k\le t.\label{eq:Deltak-UB-phit-infinite}
\end{align}

Further, we find it convenient to define $\big\{\alpha_{i}^{(t)}\big\}$
as follows
\begin{align}
\alpha_{i}^{(t)}:=\frac{\eta_{i+1}^{(t)}}{\sum_{j=(1-\beta)t}^{t-1}\eta_{j+1}^{(t)}}.
	\label{eq:defn-alpha-it-sync}
\end{align}
Clearly, this sequence satisfies
\begin{equation}
\alpha_{i}^{(t)}\ge\eta_{i+1}^{(t)}\qquad\text{and}\qquad\sum_{i=(1-\beta)t}^{t-1}\alpha_{i}^{(t)}=1\label{eq:relation-alpha-eta-inf}
\end{equation}
for any $t$, where the first inequality results from \eqref{eq:sum-eta-i-t-infinite}.
With these in place, we can write \eqref{eq:Deltak-UB-phit-infinite} as
\begin{equation}
\bm{\Delta}_{k}\le\sqrt{\bm{\varphi}_{t}}+\sum_{i_{1}=(1-\beta)k}^{k-1}\eta_{i_{1}+1}^{(k)}\gamma\bm{P}^{\pi_{i_{1}}}\bm{\Delta}_{i_{1}}=\sum_{i_{1}=(1-\beta)k}^{k-1}\Big(\alpha_{i_{1}}^{(k)}\sqrt{\bm{\varphi}_{t}}+\eta_{i_{1}+1}^{(k)}\gamma\bm{P}^{\pi_{i_{1}}}\bm{\Delta}_{i_{1}}\Big)\qquad\text{for all }2t/3\le k\le t.\label{eq:Delta-k-recursive-inf}
\end{equation}
Given that $(1-\beta)t\geq2t/3$ (see \eqref{eq:defn-beta-c4}), we
can invoke this relation recursively to yield
\begin{align}
\bm{\Delta}_{t} & \le\sum_{i_{1}=(1-\beta)t}^{t-1}\Big(\alpha_{i_{1}}^{(t)}\sqrt{\bm{\varphi}_{t}}+\eta_{i_{1}+1}^{(t)}\gamma\bm{P}^{\pi_{i_{1}}}\bm{\Delta}_{i_{1}}\Big)\nonumber \\
 & \leq\sum_{i_{1}=(1-\beta)t}^{t-1}\Bigg[\alpha_{i_{1}}^{(t)}\sqrt{\bm{\varphi}_{t}}+\eta_{i_{1}+1}^{(t)}\gamma\bm{P}^{\pi_{i_{1}}}\sum_{i_{2}=(1-\beta)i_{1}}^{i_{1}-1}\Big(\alpha_{i_{2}}^{(i_{1})}\sqrt{\bm{\varphi}_{t}}+\eta_{i_{2}+1}^{(i_{1})}\gamma\bm{P}^{\pi_{i_{2}}}\bm{\Delta}_{i_{2}}\Big)\Bigg]\nonumber \\
 & \leq\sum_{i_{1}=(1-\beta)t}^{t-1}\alpha_{i_{1}}^{(t)}\sqrt{\bm{\varphi}_{t}}+\sum_{i_{1}=(1-\beta)t}^{t-1}\sum_{i_{2}=(1-\beta)i_{1}}^{i_{1}-1}\alpha_{i_{1}}^{(t)}\alpha_{i_{2}}^{(i_{1})}\big(\gamma\bm{P}^{\pi_{i_{1}}}\big)\sqrt{\bm{\varphi}_{t}} \nonumber \\
 & \qquad\qquad +\sum_{i_{1}=(1-\beta)t}^{t-1}\sum_{i_{2}=(1-\beta)i_{1}}^{i_{1}-1}\eta_{i_{1}+1}^{(t)}\eta_{i_{2}+1}^{(i_{1})}\prod_{k=1}^{2}\big(\gamma\bm{P}^{\pi_{i_{k}}}\big)\bm{\Delta}_{i_{2}}\nonumber \\
 & =\sum_{i_{1}=(1-\beta)t}^{t-1}\sum_{i_{2}=(1-\beta)i_{1}}^{i_{1}-1}\alpha_{i_{1}}^{(t)}\alpha_{i_{2}}^{(i_{1})}\left\{ \bm{I}+\gamma\bm{P}^{\pi_{i_{1}}}\right\} \sqrt{\bm{\varphi}_{t}}+\sum_{i_{1}=(1-\beta)t}^{t-1}\sum_{i_{2}=(1-\beta)i_{1}}^{i_{1}-1}\eta_{i_{1}+1}^{(t)}\eta_{i_{2}+1}^{(i_{1})}\prod_{k=1}^{2}\big(\gamma\bm{P}^{\pi_{i_{k}}}\big)\bm{\Delta}_{i_{2}},\label{eq:Delta-recursion-inf}
\end{align}
where the second inequality relies on \eqref{eq:Delta-k-recursive-inf}, the third line uses the inequality $\eta_{i_{1}+1}^{(t)}\leq\alpha_{i_{1}}^{(t)}$ in \eqref{eq:relation-alpha-eta-inf}, and the fourth line is valid
since $\sum_{i_{2}=(1-\beta)i_{1}}^{i_{1}-1}\alpha_{i_{2}}^{(i_{1})}=1$
(see \eqref{eq:relation-alpha-eta-inf}). 

We intend to continue invoking \eqref{eq:Delta-k-recursive-inf}
recursively---similar to how we derive \eqref{eq:Delta-recursion-inf}---in
order to control $\bm{\Delta}_{t}$. To do so, we are in need of some
preparation. First, let us define
\begin{equation}
H\coloneqq\frac{\log T}{1-\gamma}\qquad\text{and}\qquad\alpha_{\{i_{k}\}_{k=1}^{H}}:=\alpha_{i_{1}}^{(t)}\alpha_{i_{2}}^{(i_{1})}\ldots\alpha_{i_{H}}^{(i_{H-1})}\ge0\label{eq:defn-H-infinite}
\end{equation}
for any $t>i_{1}>i_{2}>\cdots>i_{H}$, which clearly satisfies (see
\eqref{eq:relation-alpha-eta-inf})

\begin{equation}
\alpha_{\{i_{k}\}_{k=1}^{H}}\geq\eta_{i_{1}+1}^{(t)}\eta_{i_{2}+1}^{(i_{1})}\ldots\eta_{i_{H}+1}^{(i_{H-1})}.\label{eq:alpha-set-relation-inf-1}
\end{equation}
In addition, defining the index set 
\begin{equation}
\mathcal{I}_{t}\coloneqq\Big\{\left(i_{1},\cdots,i_{H}\right)\mymid(1-\beta)t\leq i_{1}\leq t-1,\ \forall1\leq j<H:(1-\beta)i_{j}\leq i_{j+1}\leq i_{j}-1\Big\},\label{eq:defn-It-inf}
\end{equation}
we have
\begin{equation}
\sum_{\left(i_{1},\cdots,i_{H}\right)\in\mathcal{I}_{t}}\alpha_{\{i_{k}\}_{k=1}^{H}}=1.\label{eq:alpha-set-relation-inf}
\end{equation}
Additionally, recalling that $\beta=c_{4}(1-\gamma)/\log T$, we see
that this choice of $H$ satisfies 
\[
(1-\beta)^{H}=\left(1-\frac{c_{4}(1-\gamma)}{\log T}\right)^{\frac{\log T}{1-\gamma}}\geq\frac{2}{3}
\]
 for $c_{4}$ small enough, thus implying that
\[
i_{1}>i_{2}>\cdots>i_{H}\geq(1-\beta)^{H}t\geq2t/3\qquad\text{for all }(i_{1},\cdots,i_{H})\in\mathcal{I}_{t}.
\]
This is an important property that allows one to invoke the relation
\eqref{eq:Delta-k-recursive-inf}. With these in place, applying the
preceding relation \eqref{eq:Delta-k-recursive-inf} recursively---in a way similar to \eqref{eq:Delta-recursion-inf}---further leads to
\begin{align}
\bm{\Delta}_{t} & \le\sum_{\left(i_{1},\cdots,i_{H}\right)\in\mathcal{I}_{t}}\alpha_{\{i_{k}\}_{k=1}^{H}}\Bigg\{\bigg(\bm{I}+\sum_{h=1}^{H-1}\gamma^{h}\prod_{k=1}^{h}\bm{P}^{\pi_{i_{k}}}\bigg)\sqrt{\bm{\varphi}_{t}}+\gamma^{H}\prod_{k=1}^{H}\bm{P}^{\pi_{i_{k}}}\big|\bm{\Delta}_{i_{H}}\big|\Bigg\}\nonumber \\
 & \leq\max_{\left(i_{1},\cdots,i_{H}\right)\in\mathcal{I}_{t}}\Bigg\{\underbrace{\bigg(\bm{I}+\sum_{h=1}^{H-1}\gamma^{h}\prod_{k=1}^{h}\bm{P}^{\pi_{i_{k}}}\bigg)\sqrt{\bm{\varphi}_{t}}}_{=:\,\bm{\beta}_{1}}+\underbrace{\gamma^{H}\prod_{k=1}^{H}\bm{P}^{\pi_{i_{k}}}\big|\bm{\Delta}_{i_{H}}\big|}_{=:\,\bm{\beta}_{2}}\Bigg\}\label{eq:Delta-UB-temp-inf}
\end{align}
for all $t\ge\frac{T}{c_{2}\log T}$, where we recall the definition of the entrywise $\max$ operator in Section~\ref{subsec:matrix-notation}. 
Here, the last inequality relies on the fact that $\sum_{\left(i_{1},\cdots,i_{H}\right)\in\mathcal{I}_{t}}\alpha_{\{i_{k}\}_{k=1}^{H}}=1$ (see \eqref{eq:alpha-set-relation-inf}). 
It remains to control $\bm{\beta}_{1}$ and $\bm{\beta}_{2}$, which we shall accomplish separately in the next two steps.

\paragraph{Step 4: bounding $\bm{\beta}_{2}$.}



The term $\bm{\beta}_{2}$ defined in \eqref{eq:Delta-UB-temp-inf} is relatively easier to control. Observing that $\prod_{k=1}^{H}\bm{P}^{\pi_{i_{k}}}$
is still a probability transition matrix, we can derive
\begin{align*}
|\bm{\beta}_{2}| & =\gamma^{H}\prod_{1\leq k\leq h}\bm{P}^{\pi_{i_{k}}}\big|\bm{\Delta}_{i_{H}}\big|\le\gamma^{H}\bigg\|\prod_{1\leq k\leq h}\bm{P}^{\pi_{i_{k}}}\bigg\|_{1}\|\bm{\Delta}_{i_{H}}\|_{\infty}=\gamma^{H}\|\bm{\Delta}_{i_{H}}\|_{\infty}\\
 & \overset{(\mathrm{i})}{\leq}\frac{1}{1-\gamma}\gamma^{H}\overset{(\mathrm{ii})}{\leq}\frac{1}{(1-\gamma)T},
\end{align*}
where (i) results from the crude bound (\ref{eq:Delta-t-range-infinite}).
To justify the inequality (ii), we recall the definition \eqref{eq:defn-H-infinite}
of $H$ to see that
\[
\gamma^{H}\overset{(\mathrm{ii})}{=}\big(1-(1-\gamma)\big)^{\frac{1}{1-\gamma}\log T}\leq e^{-\log T}=\frac{1}{T},
\]
where the inequality comes from the elementary fact that $\gamma^{\frac{1}{1-\gamma}}\leq e^{-1}$
for any $0<\gamma<1$.

\paragraph{Step 5: bounding $\bm{\beta}_{1}$.}

When it comes to the term $\bm{\beta}_{1}$ defined in \eqref{eq:Delta-UB-temp-inf}, we can upper bound the entrywise
square of $\bm{\beta}_{1}$---denoted by $|\bm{\beta}_{1}|^{2}$---as follows
\begin{align}
	|\bm{\beta}_{1}|^{2} & =\bigg| \Big(  \sum_{h=0}^{H-1}\gamma^{h}\prod_{1\leq k\leq h}\bm{P}^{\pi_{i_{k}}} \Big)\sqrt{\bm{\varphi}_{t}}\bigg|^{2}\overset{(\mathrm{i})}{\leq}\bigg|\sum_{h=0}^{H-1}\gamma^{h/2}\cdot\gamma^{h/2}\sqrt{\prod_{1\leq k\leq h}\bm{P}^{\pi_{i_{k}}}\bm{\varphi}_{t}}\,\bigg|^{2}\nonumber \\
 & \overset{(\mathrm{ii})}{\leq}\sum_{h=0}^{H-1}\gamma^{h}\cdot\sum_{h=0}^{H-1}\gamma^{h}\prod_{k=1}^{h}\bm{P}^{\pi_{i_{k}}}\bm{\varphi}_{t}\nonumber \\
 & \overset{(\mathrm{iii})}{\leq}\frac{1}{1-\gamma}\sum_{h=0}^{H-1} \gamma^{h}\prod_{k=1}^{h}\bm{P}^{\pi_{i_{k}}}\frac{64\big(\log^{4}T\big)\big(\log\frac{|\mathcal{S}||\mathcal{A}|T}{\delta}\big)}{(1-\gamma)T}\bigg(\max_{\frac{t}{2}\le i<t}\mathsf{Var}_{\bm{P}}(\bm{V}_{i})+\bm{1}\bigg)\nonumber \\
 & \overset{(\mathrm{iv})}{\leq}\frac{64\big(\log^{4}T\big)\big(\log\frac{|\mathcal{S}||\mathcal{A}|T}{\delta}\big)}{(1-\gamma)^{2}T}\sum_{h=0}^{H-1} \gamma^{h}\prod_{k=1}^{h}\bm{P}^{\pi_{i_{k}}}\max_{\frac{t}{2}\le i<t}\mathsf{Var}_{\bm{P}}(\bm{V}_{i})+\frac{64\big(\log^{4}T\big)\big(\log\frac{|\mathcal{S}||\mathcal{A}|T}{\delta}\big)}{(1-\gamma)^{3}T}\bm{1}\nonumber.
\end{align}
Here, (i) follows from Jensen's inequality and the fact that $\prod_{k=1}^{h}\bm{P}^{\pi_{i_{k}}}$
is a probability transition matrix; (ii) holds due to the Cauchy-Schwarz
inequality; (iii) utilizes the definition of $\bm{\varphi}_{t}$ in
(\ref{def:phi}); (iv) follows since $\prod_{1\leq k\leq h}\bm{P}^{\pi_{i_{k}}}\bm{1}=\bm{1}$
and $\sum_{0\leq h<H}\gamma^{h}\leq\frac{1}{1-\gamma}$.
To further control the right-hand side of the above inequality, we resort to the following lemma.  
\begin{lemma}\label{lemma:inf-alpha-claim}
Suppose that $t\ge\frac{T}{c_{2}\log T}$. For any $\left(i_{1},\cdots,i_{H}\right)\in\mathcal{I}_{t}$, 
the following holds:
\begin{equation}
\sum_{h=0}^{H-1}\gamma^{h}\prod_{k=1}^{h}\bm{P}^{\pi_{i_{k}}}\max_{\frac{t}{2}\le i<t}\mathsf{Var}_{\bm{P}}(\bm{V}_{i})
	\leq \frac{4}{\gamma^2(1-\gamma)^{2}}\Big(1+2\max_{\frac{t}{2}\le i<t}\|\bm{\Delta}_{i}\|_{\infty}\Big)\bm{1}.\label{eq:inf-alpha-claim}
\end{equation}
\end{lemma}
\begin{proof} This lemma, which is inspired by but significantly more complicated than \citet[Lemma 8]{azar2013minimax},
plays a key role in shaving one $\frac{1}{1-\gamma}$ factor. 
See Section~\ref{subsec:Proof-of-eq:inf-alpha-claim} for the proof.\end{proof}
Therefore, the above result directly implies that 
\begin{align}
|\bm{\beta}_{1}|^{2} 
& \leq 
\frac{320\big(\log^{4}T\big)\big(\log\frac{|\mathcal{S}||\mathcal{A}|T}{\delta}\big)}{\gamma^2 (1-\gamma)^{4}T}\Big(1+2\max_{\frac{t}{2}\le i<t}\|\bm{\Delta}_{i}\|_{\infty}\Big)\bm{1}.\label{eq:alpha-1-UB}
\end{align}

\paragraph{Step 6: putting all this together.}
Substituting the preceding bounds for $\bm{\beta}_{1}$ and $\bm{\beta}_{2}$
into \eqref{eq:Delta-UB-temp-inf}, we can demonstrate that: with
probability at least $1-\delta$, 
\begin{align}
\bm{\Delta}_{t} & \leq \frac{1}{(1-\gamma)T}\bm{1}+\sqrt{\frac{320\big(\log^{4}T\big)\big(\log\frac{|\mathcal{S}||\mathcal{A}|T}{\delta}\big)}{\gamma^2 (1-\gamma)^{4}T}\Big(1+2\max_{\frac{t}{2}\le i<t}\|\bm{\Delta}_{i}\|_{\infty}\Big)}\ \bm{1}\nonumber \\
 & \leq 30 \sqrt{\frac{\big(\log^{4}T\big)\big(\log\frac{|\mathcal{S}||\mathcal{A}|T}{\delta}\big)}{\gamma^2 (1-\gamma)^{4}T}\Big(1+\max_{\frac{t}{2}\le i<t}\|\bm{\Delta}_{i}\|_{\infty}\Big)}\ \bm{1}\label{eq:Delta-UB-final}
\end{align}
holds simultaneously for all $t\ge\frac{T}{c_{2}\log T}$,
where the second line is valid since $\frac{1}{(1-\gamma)T}\leq \sqrt{\frac{(\log^{4}T)\big(\log\frac{|\mathcal{S}||\mathcal{A}|T}{\delta}\big)}{(1-\gamma)^{4}T}}$
under our sample size condition \eqref{eq:T-bound-infinite}.

\subsection{Proof of Lemma~\ref{lem:inf-lower-bound} }
\label{sec:proof_inf_lower_bound}

Next, we move forward to develop an lower bound on $\bm{\Delta}_{t}$,
which can be accomplished in an analogous manner as for the above
upper bound. Applying a similar argument for (\ref{eq:Delta-UB-temp-inf})
(except that we need to replace $\pi_{i}$ with $\pi^{\star}$), one
can deduce that
\begin{align}
\bm{\Delta}_{t} & \ge-\max_{\left(i_{1},\cdots,i_{H}\right)\in\mathcal{I}_{t}}\left\{\sum_{h=0}^{H-1}\gamma^{h}\prod_{k=1}^{h}\bm{P}^{\pi^{\star}}\sqrt{\bm{\varphi}_{t}}+\gamma^{H}\prod_{k=1}^{H}\bm{P}^{\pi^{\star}}\big|\bm{\Delta}_{i_{H}}\big|\right\}\label{eq:Delta-t-lower-bound-inf}
\end{align}
for any $t\ge\frac{c_{2}T}{\log\frac{1}{1-\gamma}}$. It is straightforward
to bound the second term on the right-hand side of \eqref{eq:Delta-t-lower-bound-inf}
as
\[
\gamma^{H}\prod_{1\leq k\leq H}\bm{P}^{\pi^{\star}}\big|\bm{\Delta}_{i_{H}}\big|\le\gamma^{H}\bigg\|\prod_{1\leq k\leq H}\bm{P}^{\pi^{\star}}\bigg\|_{1}\|\bm{\Delta}_{i_{H}}\|_{\infty}\bm{1}\le\frac{1}{(1-\gamma)T}\bm{1},
\]
where the second inequality makes use of \eqref{eq:Delta-t-range-infinite}
as well as the fact that $\prod_{k}\bm{P}^{\pi^{\star}}$ is a probability
transition matrix (so that $\|\prod_{k}\bm{P}^{\pi^{\star}}\big\|_{1}=1$).
As for the first term on the right-hand side of \eqref{eq:Delta-t-lower-bound-inf},
we can invoke a similar argument for (\ref{eq:alpha-1-UB}) to obtain
\begin{align*}
\bigg|\sum_{h=0}^{H-1}\gamma^{h}\prod_{k=1}^{h}\bm{P}^{\pi^{\star}}\sqrt{\bm{\varphi}_{t}}\bigg|^{2} & 
\leq 320\frac{\big(\log^{4}T\big)\big(\log\frac{|\mathcal{S}||\mathcal{A}|T}{\delta}\big)}{\gamma^2(1-\gamma)^{4}T}\bigg(1+2\max_{\frac{t}{2}\le i<t}\|\bm{\Delta}_{i}\|_{\infty}\bigg)\bm{1}.
\end{align*}
Taking these two bounds together, we see that with probability at
least $1-\delta$, 
\begin{equation}
\bm{\Delta}_{t}\geq -30 \sqrt{\frac{\big(\log^{4}T\big)\big(\log\frac{|\mathcal{S}||\mathcal{A}|T}{\delta}\big)}{\gamma^2 (1-\gamma)^{4}T}\bigg(1+\max_{\frac{t}{2}\le i<t}\|\bm{\Delta}_{i}\|_{\infty}\bigg)}\ \bm{1}\label{eq:Delta-LB-final}
\end{equation}
holds simultaneously for all $t\ge\frac{T}{c_{2}\log T}$. 


 \subsection{Solving the recurrence relation regarding $\bm{\Delta}_{t}$}

\label{sec:combine_ub_lb}

 Recall from \eqref{eq:Delta-t-inf-recursion-inf-outline} that with probability exceeding $1-2\delta$, 
 the following recurrence relation  
 \begin{align}
 \|\bm{\Delta}_{t}\|_{\infty} & \leq 30 \sqrt{\frac{\big(\log^{4}T\big)\big(\log\frac{|\mathcal{S}||\mathcal{A}|T}{\delta}\big)}{\gamma^2(1-\gamma)^{4}T}\bigg(1+\max_{\frac{t}{2}\le i<t}\|\bm{\Delta}_{i}\|_{\infty}\bigg)}\qquad\text{for all}\qquad t\ge\frac{T}{c_{2}\log T}\label{eq:Delta-t-inf-recursion-inf}
 \end{align}
 holds, which plays a crucial role
 in establishing the desired estimation error bound. Specifically,
 for any $k\ge0$, let us define 
 \begin{equation}
 u_{k}:=\max\left\{ \|\bm{\Delta}_{t}\|_{\infty}\ \Big|\ 2^{k}\frac{T}{c_{2}\log T}\leq t\leq T\right\} .\label{eq:uk-definition-inf}
 \end{equation}
 To bound this sequence, we first obtain a crude bound as a result
 of \eqref{eq:Delta-t-range-infinite}: 
 \begin{equation}
	 u_{0}\le\frac{1}{1-\gamma}. \label{eq:u0-trivial-UB}
 \end{equation}
 Next, it is directly seen from \eqref{eq:Delta-t-inf-recursion-inf}
 and the definition of $u_{k}$ that
 \begin{equation}
 u_{k}\leq c_6\sqrt{\frac{\big(\log^{4}T\big)\big(\log\frac{|\mathcal{S}||\mathcal{A}|T}{\delta}\big)}{\gamma^2 (1-\gamma)^{4}T}\left(1+u_{k-1}\right)},\qquad k\geq1\label{eq:uk-recursion-inf}
 \end{equation}
 for some constant $c_{6}=20/\gamma>0$. In order to analyze the size of $u_{k}$,
 we divide into two cases. 
 \begin{itemize}
 \item If $u_{k}\leq1$ for some $k\geq1$, then \eqref{eq:uk-recursion-inf}
 tells us that
 \[
 u_{k+1}\leq c_{6}\sqrt{\frac{\big(\log^{4}T\big)\big(\log\frac{|\mathcal{S}||\mathcal{A}|T}{\delta}\big)}{\gamma^2(1-\gamma)^{4}T}\left(1+u_{k}\right)}\leq c_{6}\sqrt{\frac{2\big(\log^{4}T\big)\big(\log\frac{|\mathcal{S}||\mathcal{A}|T}{\delta}\big)}{\gamma^2(1-\gamma)^{4}T}}\leq1,
 \]
 as long as $T\geq\frac{2c_{6}^{2}\log^{4}T\log\frac{|\mathcal{S}||\mathcal{A}|T}{\delta}}{\gamma^2(1-\gamma)^{4}}$.
 In other words, once $u_{k-1}$ drops below $1$, then all subsequent
 quantities will remain bounded above by 1, namely, $\max_{j:j\geq k}u_{j}\leq1$.
 As a result, 
 \[
 u_{j}\leq c_{6}\sqrt{\frac{\big(\log^{4}T\big)\big(\log\frac{|\mathcal{S}||\mathcal{A}|T}{\delta}\big)}{\gamma^2(1-\gamma)^{4}T}\left(1+u_{j-1}\right)}
		 \leq c_{6}\sqrt{\frac{2\big(\log^{4}T\big)\big(\log\frac{|\mathcal{S}||\mathcal{A}|T}{\delta}\big)}{\gamma^2(1-\gamma)^{4}T}}\qquad\text{for all }j>k.
 \]

 \item Instead, suppose that $u_{j}>1$ for all $0\leq j\leq k$. Then it
 is seen from \eqref{eq:uk-recursion-inf} that
 \[
 u_{j+1}\leq c_{6}\sqrt{\frac{\big(\log^{4}T\big)\big(\log\frac{|\mathcal{S}||\mathcal{A}|T}{\delta}\big)}{\gamma^2(1-\gamma)^{4}T}\left(1+u_{j}\right)}\leq c_{6}\sqrt{\frac{2\big(\log^{4}T\big)\big(\log\frac{|\mathcal{S}||\mathcal{A}|T}{\delta}\big)}{\gamma^2(1-\gamma)^{4}T}u_{j}}\qquad\text{for all }j\leq k.
 \]
 This is equivalent to saying that
 \[
 \log u_{j+1}\leq\log\alpha_{u}+\frac{1}{2}\log u_{j}\qquad\text{for all }j\leq k,
 \]
 where $\alpha_{u}=c_{6}\sqrt{\frac{2\big(\log^{4}T\big)\big(\log\frac{|\mathcal{S}||\mathcal{A}|T}{\delta}\big)}{\gamma^2(1-\gamma)^{4}T}}.$
 Invoking a standard analysis strategy for this type of recursive relations
 yields
 \[
	 \log u_{j+1} - 2\log\alpha_{u} \leq  \frac{1}{2} \left(\log u_{j}-2\log\alpha_{u}\right)\qquad\text{for all }j\leq k,
 \]
 and hence
 \[
 \log u_{j+1}\leq2\log\alpha_{u}+\left(\frac{1}{2}\right)^{j+1}\left(\log u_{0}-2\log\alpha_{u}\right)\qquad\text{for all }j\leq k.
 \]
 This is equivalent to saying that
 \[
 u_{j}\leq\alpha_{u}^{2}\left(\frac{u_{0}}{\alpha_{u}^{2}}\right)^{1/2^{j}}=\left(\alpha_{u}^{2}\right)^{1-1/2^{j}}\left(u_{0}\right)^{1/2^{j}}\qquad\text{for all }j\leq k+1.
 \]
 \end{itemize}
 Putting the above two cases together and using \eqref{eq:u0-trivial-UB}, we conclude that
 \begin{align*}
 u_{k} & \leq\sqrt{\frac{2c_{6}^{2}\big(\log^{4}T\big)\big(\log\frac{|\mathcal{S}||\mathcal{A}|T}{\delta}\big)}{\gamma^2(1-\gamma)^{4}T}}+\Big(\frac{2c_{6}^{2}\big(\log^{4}T\big)\big(\log\frac{|\mathcal{S}||\mathcal{A}|T}{\delta}\big)}{\gamma^2(1-\gamma)^{4}T}\Big)^{1-1/2^{k}}u_{0}^{1/2^{k}}\\
  & \leq\sqrt{\frac{2c_{6}^{2}\big(\log^{4}T\big)\big(\log\frac{|\mathcal{S}||\mathcal{A}|T}{\delta}\big)}{\gamma^2(1-\gamma)^{4}T}}+\Big(\frac{2c_{6}^{2}\big(\log^{4}T\big)\big(\log\frac{|\mathcal{S}||\mathcal{A}|T}{\delta}\big)}{\gamma^2(1-\gamma)^{4}T}\Big)^{1-1/2^{k}}\left(\frac{1}{1-\gamma}\right)^{1/2^{k}},\qquad k\geq1.
 \end{align*}
 In particular, as long as $k\geq c_7 \log\log\frac{1}{1-\gamma}$ for some constant $c_7>0$, one
 has $\big(\frac{1}{1-\gamma}\big)^{1/2^{k}} \leq O(1)$ and
 \[
 \Big(\frac{2c_{6}^{2}\big(\log^{4}T\big)\big(\log\frac{|\mathcal{S}||\mathcal{A}|T}{\delta}\big)}{(1-\gamma)^{4}T}\Big)^{1-1/2^{k}}\le\max\left\{ \sqrt{\frac{2c_{6}^{2}\big(\log^{4}T\big)\big(\log\frac{|\mathcal{S}||\mathcal{A}|T}{\delta}\big)}{\gamma^2(1-\gamma)^{4}T}},\frac{2c_{6}^{2}\big(\log^{4}T\big)\big(\log\frac{|\mathcal{S}||\mathcal{A}|T}{\delta}\big)}{\gamma^2(1-\gamma)^{4}T}\right\} .
 \]
 As a result, the above bound simplifies to
 \begin{align*}
	 u_{k} & \leq c_8 \bigg( \sqrt{\frac{\big(\log^{4}T\big)\big(\log\frac{|\mathcal{S}||\mathcal{A}|T}{\delta}\big)}{\gamma^2(1-\gamma)^{4}T}}+\frac{\big(\log^{4}T\big)\big(\log\frac{|\mathcal{S}||\mathcal{A}|T}{\delta}\big)}{\gamma^2(1-\gamma)^{4}T}\bigg),
	 \qquad k\geq c_7 \log\log\frac{1}{1-\gamma} 
 \end{align*}
 for some  constant $c_8>0$. 

 Consequently, taking $t=T$ and choosing $k= c_7\log\log\frac{1}{1-\gamma}$
 for some appropriate constant $c_7>0$ (so as to ensure $2^{k}\frac{T}{c_{2}\log T}<T$),
 we immediately see from the definition \eqref{eq:uk-definition-inf}
 of $u_{k}$ that 
 \begin{equation}
	 \|\bm{\Delta}_{T}\|_{\infty}\leq c_8 \bigg( \sqrt{\frac{\big(\log^{4}T\big)\big(\log\frac{|\mathcal{S}||\mathcal{A}|T}{\delta}\big)}{\gamma^2(1-\gamma)^{4}T}}+\frac{\big(\log^{4}T\big)\big(\log\frac{|\mathcal{S}||\mathcal{A}|T}{\delta}\big)}{\gamma^2(1-\gamma)^{4}T} \bigg) .\label{eq:Delta-T-UB-final-inf}
 \end{equation}
 with probability at least $1-2\delta$. To finish up, we note that
 the sample size assumption (\ref{eq:sample-size-condition-complete-inf})
 is equivalent to
 \[
 \frac{\big(\log^{4}T\big)\big(\log\frac{|\mathcal{S}||\mathcal{A}|T}{\delta}\big)}{\gamma^2(1-\gamma)^{4}T}\le\frac{\min\{\varepsilon^{2},\varepsilon\}}{c_{3}}.
 \]
 When $c_{3}>0$ is sufficiently large, substituting this relation
 into \eqref{eq:Delta-T-UB-final-inf} gives
 \begin{align*}
 \|\bm{\Delta}_{T}\|_{\infty} & \leq\frac{1}{2}\sqrt{\min\{\varepsilon^{2},\varepsilon\}}+\frac{1}{2}\min\{\varepsilon^{2},\varepsilon\}=\begin{cases}
 \frac{1}{2}\varepsilon+\frac{1}{2}\varepsilon^{2} & \text{if }\varepsilon\leq1\\
 \frac{1}{2}\sqrt{\varepsilon}+\frac{1}{2}\varepsilon\quad & \text{if }\varepsilon>1
 \end{cases}\\
  & \leq\varepsilon
 \end{align*}
 as claimed in Theorem \ref{thm:infinite-horizon}. 

\subsection{Proof of Lemma \ref{lemma:inf-alpha-claim}}

\label{subsec:Proof-of-eq:inf-alpha-claim}

We first claim that
\begin{equation}
\max_{\frac{t}{2}\le i<t}\mathsf{Var}_{\bm{P}}(\bm{V}_{i})-\mathsf{Var}_{\bm{P}}(\bm{V}^{\star})\le\frac{4}{1-\gamma}\max_{\frac{t}{2}\le i<t}\|\bm{\Delta}_{i}\|_{\infty}\bm{1}.\label{eq:Var-P-Var-star-gap-inf}
\end{equation}
If this claim were valid (which we shall justify towards the end of
this subsection), then it would lead to
\begin{equation}
\sum_{h=0}^{H-1}\gamma^{h}\prod_{k=1}^{h}\bm{P}^{\pi_{i_{k}}}\max_{\frac{t}{2}\le i<t}\mathsf{Var}_{\bm{P}}(\bm{V}_{i})\le\sum_{h=0}^{H-1}\gamma^{h}\prod_{k=1}^{h}\bm{P}^{\pi_{i_{k}}}\mathsf{Var}_{\bm{P}}(\bm{V}^{\star})+\frac{4}{(1-\gamma)^{2}}\max_{\frac{t}{2}\le i<t}\|\bm{\Delta}_{i}\|_{\infty}\bm{1}.\label{eq:mixture-gamma-P-Var-V-inf}
\end{equation}

It then boils down to bounding the first term on the right-hand side
of \eqref{eq:mixture-gamma-P-Var-V-inf}. 
Let us first upper bound the variance term involving $\bm{V}^{\star}$.
For any $0\leq h<H$, one can express (see \eqref{eq:defn-Var-P-V})
\begin{align}
\mathsf{Var}_{\bm{P}}(\bm{V}^{\star}) & =\bm{P}(\bm{V}^{\star}\circ\bm{V}^{\star})-(\bm{P}\bm{V}^{\star})\circ(\bm{P}\bm{V}^{\star})\nonumber\\
 & \overset{(\mathrm{i})}{=}\bm{P}^{\pi_{i_{h+1}}}(\bm{Q}^{\star}\circ\bm{Q}^{\star})+\bm{P}(\bm{V}^{\star}\circ\bm{V}^{\star})-\bm{P}^{\pi_{i_{h+1}}}(\bm{Q}^{\star}\circ\bm{Q}^{\star})-\frac{1}{\gamma^{2}}(\bm{Q}^{\star}-\bm{r})\circ(\bm{Q}^{\star}-\bm{r})\nonumber\\
 & =\bm{P}^{\pi_{i_{h+1}}}(\bm{Q}^{\star}\circ\bm{Q}^{\star})+\bm{P}^{\pi^{\star}}(\bm{Q}^{\star}\circ\bm{Q}^{\star})-\bm{P}^{\pi_{i_{h+1}}}(\bm{Q}^{\star}\circ\bm{Q}^{\star})-\frac{1}{\gamma^{2}}(\bm{Q}^{\star}-\bm{r})\circ(\bm{Q}^{\star}-\bm{r})\nonumber\\
 & \leq\bm{P}^{\pi_{i_{h+1}}}(\bm{Q}^{\star}\circ\bm{Q}^{\star})+\big\|\bm{P}^{\pi^{\star}}(\bm{Q}^{\star}\circ\bm{Q}^{\star})-\bm{P}^{\pi_{i_{h+1}}}(\bm{Q}^{\star}\circ\bm{Q}^{\star})\big\|_{\infty}\bm{1}-\frac{1}{\gamma^{2}}(\bm{Q}^{\star}-\bm{r})\circ(\bm{Q}^{\star}-\bm{r})\nonumber\\
 & \overset{(\mathrm{ii})}{\leq}\bm{P}^{\pi_{i_{h+1}}}(\bm{Q}^{\star}\circ\bm{Q}^{\star})+\frac{4}{1-\gamma}\max_{\frac{t}{2}\le i<t}\|\bm{\Delta}_{i}\|_{\infty}\bm{1}-\frac{1}{\gamma^{2}}(\bm{Q}^{\star}-\bm{r})\circ(\bm{Q}^{\star}-\bm{r})\nonumber\\
 & =\frac{1}{\gamma^{2}}\big(\gamma^{2}\bm{P}^{\pi_{i_{h+1}}}(\bm{Q}^{\star}\circ\bm{Q}^{\star})-\bm{Q}^{\star}\circ\bm{Q}^{\star}\big)+\frac{4}{1-\gamma}\max_{\frac{t}{2}\le i<t}\|\bm{\Delta}_{i}\|_{\infty}\bm{1}-\frac{1}{\gamma^{2}}\bm{r}\circ\bm{r}+\frac{2}{\gamma^{2}}\bm{Q}^{\star}\circ\bm{r}\nonumber\\
 & \overset{(\mathrm{iii})}{\leq}\frac{1}{\gamma}\big(\gamma\bm{P}^{\pi_{i_{h+1}}}(\bm{Q}^{\star}\circ\bm{Q}^{\star})-\bm{Q}^{\star}\circ\bm{Q}^{\star}\big)+\frac{2}{\gamma^{2}}\bm{Q}^{\star}\circ\bm{r}+\frac{4}{1-\gamma}\max_{\frac{t}{2}\le i<t}\|\bm{\Delta}_{i}\|_{\infty}\bm{1},\label{eq:var-V-UB}
\end{align}
where (i) relies on the identity $\bm{Q}^{\star}=\bm{r}+\gamma\bm{P}\bm{V}^{\star}$,
and (iii) holds since $0<\gamma<1$. To justify (ii), we make the
following observation:
\begin{align*}
 & \big\|\bm{P}^{\pi_{i_{h+1}}}\big(\bm{Q}^{\star}\circ\bm{Q}^{\star}\big)-\bm{P}^{\pi^{\star}}\big(\bm{Q}^{\star}\circ\bm{Q}^{\star}\big)\big\|_{\infty}=\big\|\bm{P}\bm{\Pi}^{\pi_{i_{h+1}}}\big(\bm{Q}^{\star}\circ\bm{Q}^{\star}\big)-\bm{P}\bm{\Pi}^{\pi^{\star}}\big(\bm{Q}^{\star}\circ\bm{Q}^{\star}\big)\big\|_{\infty}\\
 & \quad\overset{(\mathrm{iv})}{\leq}\big\|\bm{\Pi}^{\pi_{i_{h+1}}}\big(\bm{Q}^{\star}\circ\bm{Q}^{\star}\big)-\bm{\Pi}^{\pi^{\star}}\big(\bm{Q}^{\star}\circ\bm{Q}^{\star}\big)\big\|_{\infty}\\
 & \quad=\big\|\big(\bm{\Pi}^{\pi_{i_{h+1}}}\bm{Q}^{\star}-\bm{\Pi}^{\pi^{\star}}\bm{Q}^{\star}\big)\circ\big(\bm{\Pi}^{\pi_{i_{h+1}}}\bm{Q}^{\star}+\bm{\Pi}^{\pi^{\star}}\bm{Q}^{\star}\big)\big\|_{\infty}\\
 & \quad\overset{(\mathrm{v})}{\leq}\frac{2}{1-\gamma}\big\|\bm{\Pi}^{\pi_{i_{h+1}}}\bm{Q}^{\star}-\bm{\Pi}^{\pi^{\star}}\bm{Q}^{\star}\big\|_{\infty}\\
 & \quad\le\frac{2}{1-\gamma}\Big(\big\|\bm{\Pi}^{\pi_{i_{h+1}}}\bm{Q}^{\star}-\bm{\Pi}^{\pi_{i_{h+1}}}\bm{Q}_{i_{h+1}}\big\|_{\infty}+\big\|\bm{\Pi}^{\pi_{i_{h+1}}}\bm{Q}_{i_{h+1}}-\bm{V}^{\star}\big\|_{\infty}\Big)\\
	& \quad \overset{\mathrm{(vi)}}{\leq} \frac{2}{1-\gamma}\Big(\big\|\bm{Q}^{\star}-\bm{Q}_{i_{h+1}}\big\|_{\infty}+\big\|\bm{V}_{i_{h+1}}-\bm{V}^{\star}\big\|_{\infty}\Big)\\
 & \quad\overset{\mathrm{(vii)}}{\leq}\frac{4}{1-\gamma}\max_{\frac{t}{2}\le i<t}\|\bm{\Delta}_{i}\|_{\infty},
\end{align*}
where (iv) arises from the fact $\|\bm{P}\bm{z}\|_{\infty}\leq\|\bm{P}\|_{1}\|\bm{z}\|_{\infty}=\|\bm{z}\|_{\infty}$, 
(v) is valid because $\|\bm{Q}^{\star}\|_{\infty}\leq  1/(1-\gamma)$,  (vi) follows from the fact that $\bm{V}_{i_{h+1}}=\bm{\Pi}^{\pi_{i_{h+1}}}\bm{Q}_{i_{h+1}}$,
and (vi) holds since $\|\bm{V}_{i_{h+1}}-\bm{V}^{\star}\|_{\infty}\leq \|\bm{Q}_{i_{h+1}}-\bm{Q}^{\star}\|_{\infty}$.  

As it turns out, the first term in \eqref{eq:var-V-UB} allows one to build a telescoping sum. Specifically, invoking (\ref{eq:var-V-UB}) allows one to bound
\begin{align}
\sum_{h=0}^{H-1}\prod_{k=1}^{h}\gamma\bm{P}^{\pi_{i_{k}}}\mathsf{Var}_{\bm{P}}(\bm{V}^{\star}) & \leq\frac{1}{\gamma}\sum_{h=0}^{H-1}\prod_{k=1}^{h}\gamma\bm{P}^{\pi_{i_{k}}}\big(\gamma\bm{P}^{\pi_{i_{h+1}}}(\bm{Q}^{\star}\circ\bm{Q}^{\star})-\bm{Q}^{\star}\circ\bm{Q}^{\star}\big) \nonumber \\
 & \qquad\qquad+\frac{4}{1-\gamma}\max_{\frac{t}{2}\le i<t}\|\bm{\Delta}_{i}\|_{\infty}\sum_{h=0}^{H-1}\gamma^{h}\prod_{k=1}^{h}\bm{P}^{\pi_{i_{k}}}\bm{1}+\frac{2}{\gamma^{2}}\sum_{h=0}^{H-1}\gamma^{h}\prod_{k=1}^{h}\bm{P}^{\pi_{i_{k}}}\left(\bm{Q}^{\star}\circ\bm{r}\right) \nonumber \\
 & \overset{(\mathrm{i})}{=}\frac{1}{\gamma}\Big(\sum_{h=0}^{H-1}\prod_{k=1}^{h+1}\gamma\bm{P}^{\pi_{i_{k}}}-\sum_{h=0}^{H-1}\prod_{k=1}^{h}\gamma\bm{P}^{\pi_{i_{k}}}\Big)\left(\bm{Q}^{\star}\circ\bm{Q}^{\star}\right) \nonumber \\
 & \qquad\qquad+\frac{4}{1-\gamma}\max_{\frac{t}{2}\le i<t}\|\bm{\Delta}_{i}\|_{\infty}\sum_{h=0}^{H-1}\gamma^{h}\bm{1}+\frac{2}{\gamma^{2}}\sum_{h=0}^{H-1}\gamma^{h}\prod_{k=1}^{h}\bm{P}^{\pi_{i_{k}}}\left(\bm{Q}^{\star}\circ\bm{r}\right) \nonumber \\
 & \leq\frac{1}{\gamma}\bigg(\prod_{k=1}^{H}\gamma\bm{P}^{\pi_{i_{k}}}-\bm{I}\bigg)\left(\bm{Q}^{\star}\circ\bm{Q}^{\star}\right)+\frac{4}{(1-\gamma)^{2}}\max_{\frac{t}{2}\le i<t}\|\bm{\Delta}_{i}\|_{\infty}\bm{1} \nonumber \\
 & \qquad\qquad+\frac{2}{\gamma^{2}}\sum_{h=0}^{H-1}\gamma^{h}\prod_{k=1}^{h}\bm{P}^{\pi_{i_{k}}}\left(\bm{Q}^{\star}\circ\bm{r}\right) \nonumber \\
 & \overset{(\mathrm{ii})}{\leq}\bigg(\frac{2}{\gamma}\|\bm{Q}^{\star}\|_{\infty}^{2}+\frac{4}{(1-\gamma)^{2}}\max_{\frac{t}{2}\le i<t}\|\bm{\Delta}_{i}\|_{\infty}+\frac{2}{\gamma^{2}}\frac{1}{1-\gamma}\|\bm{Q}^{\star}\|_{\infty}\|\bm{r}\|_{\infty}\bigg)\bm{1} \nonumber \\
 & \overset{(\mathrm{iii})}{\leq}\frac{1}{(1-\gamma)^{2}}\bigg(\frac{2}{\gamma}+4\max_{\frac{t}{2}\le i<t}\|\bm{\Delta}_{i}\|_{\infty}+\frac{2}{\gamma^{2}}\bigg)\bm{1} \nonumber \\
 & \le\frac{1}{(1-\gamma)^{2}}\bigg(\frac{4}{\gamma^{2}}+4\max_{\frac{t}{2}\le i<t}\|\bm{\Delta}_{i}\|_{\infty}\bigg)\bm{1}. \label{eq:sum-prod-P-var-UB}
\end{align}
Here, (i) comes from the identity $\prod_{k=1}^{h}\bm{P}^{\pi_{i_{k}}}\bm{1}=\bm{1}$;
(ii) holds because each row of $\prod_{k=1}^{h}\bm{P}^{\pi_{i_{k}}}$
has unit $\|\cdot\|_{1}$ norm for any $h$; (iii) arises from the
bound $\|\bm{Q}^{\star}\|_{\infty}\leq1/(1-\gamma)$. This completes
the proof, as long as the claim \eqref{eq:Var-P-Var-star-gap-inf}
can be justified.

\paragraph{Proof of the inequality \eqref{eq:Var-P-Var-star-gap-inf}.}
To validate this result, we make the observation that
\begin{align*}
\mathsf{Var}_{\bm{P}}(\bm{V}_{i})-\mathsf{Var}_{\bm{P}}(\bm{V}^{\star}) & =\big[\bm{P}(\bm{V}_{i}\circ\bm{V}_{i})-(\bm{P}\bm{V}_{i})\circ(\bm{P}\bm{V}_{i})\big]-\big[\bm{P}(\bm{V}^{\star}\circ\bm{V}^{\star})-(\bm{P}\bm{V}^{\star})\circ(\bm{P}\bm{V}^{\star})\big]\\
 & =\bm{P}(\bm{V}_{i}\circ\bm{V}_{i}-\bm{V}^{\star}\circ\bm{V}^{\star})+(\bm{P}\bm{V}^{\star})\circ(\bm{P}\bm{V}^{\star})-(\bm{P}\bm{V}_{i})\circ(\bm{P}\bm{V}_{i})\\
	& =\bm{P}\big((\bm{V}_{i}-\bm{V}^{\star})\circ(\bm{V}_{i}+\bm{V}^{\star}) \big)
	+(\bm{P}\bm{V}^{\star}-\bm{P}\bm{V}_{i})\circ(\bm{P}\bm{V}^{\star}+\bm{P}\bm{V}_{i})\\
	& \le\Big\{ \big\|\bm{P}\big( (\bm{V}_{i}-\bm{V}^{\star}) \circ(\bm{V}_{i}+\bm{V}^{\star}) \big)\big\|_{\infty}
	+ \big\|(\bm{P}\bm{V}^{\star}-\bm{P}\bm{V}_{i})\circ(\bm{P}\bm{V}^{\star}+\bm{P}\bm{V}_{i}) \big\|_{\infty}\Big\}\bm{1}\\
 & \leq\frac{4}{1-\gamma}\|\bm{\Delta}_{i}\|_{\infty}\bm{1}.
\end{align*}
Here, the last inequality follows since (by applying Lemma \ref{lemma:non-negativity-Qt-Vt})
\[
	\big\|\bm{P}\big( (\bm{V}_{i}-\bm{V}^{\star})\circ(\bm{V}_{i}+\bm{V}^{\star}) \big)\big\|_{\infty}\le\|\bm{P}\|_{1}\|\bm{V}_{i}-\bm{V}^{\star}\|_{\infty}\|\bm{V}_{i}+\bm{V}^{\star}\|_{\infty}\le\frac{2}{1-\gamma}\|\bm{\Delta}_{i}\|_{\infty},
\]
\[
\text{and}\quad 
\big\|(\bm{P}\bm{V}^{\star}-\bm{P}\bm{V}_{i})\circ(\bm{P}\bm{V}^{\star}+\bm{P}\bm{V}_{i}) \big\|_{\infty}\le\|\bm{P}\|_{1}\|\bm{V}_{i}-\bm{V}^{\star}\|_{\infty}\cdot\|\bm{P}\|_{1}\|\bm{V}_{i}+\bm{V}^{\star}\|_{\infty}\le\frac{2}{1-\gamma}\|\bm{\Delta}_{i}\|_{\infty}.
\]

\paragraph{A useful extension of Lemma~\ref{lemma:inf-alpha-claim}.}
Before concluding, we  make note of the following extension that proves useful for studying asynchronous Q-learning. 
\begin{lemma}
\label{lemma:inf-alpha-claim-async}
Suppose that $t\geq \frac{T}{c_2\log T}$. Then one has
\begin{equation}
	\sum_{h=0}^{H-1}\gamma^{h}\prod_{k=1}^{h}\bm{P}^{\widehat{\pi}_{k}}\max_{\frac{t}{2}\le i<t}\mathsf{Var}_{\bm{P}}(\bm{V}_{i})
	\leq \frac{4}{\gamma^2(1-\gamma)^{2}}\Big(1+2\max_{\frac{t}{2}\le i<t}\|\bm{\Delta}_{i}\|_{\infty}\Big)\bm{1} \label{eq:inf-alpha-claim}
\end{equation}
for any set of policies $\{\widehat{\pi}_{k}\}$ obeying $\{\widehat{\pi}_{k}\}\subseteq \Pi$.  Here, we define
\begin{eqnarray}
\label{defn:Pi-asyncQ-123}
\begin{aligned}
	\Pi  \defn\big\{\pi=[\pi(s)]_{s\in\cS}\, \big| \, \pi(s) \in\Pi_{s}, \forall s\in\cS\big\}, 
	\qquad \Pi_{s}  \defn\big\{ \pi_{i}(s) \mid i\in[t/2,t)\big\}. 
\end{aligned}
\end{eqnarray}
\end{lemma}
The key difference between Lemma~\ref{lemma:inf-alpha-claim-async} and Lemma~\ref{lemma:inf-alpha-claim} is that: 
the components of $\widehat{\pi}_{k}$ corresponding to different states can be chosen in a separate manner. 
The proof follows from an identical argument as the above proof of Lemma~\ref{lemma:inf-alpha-claim}, and is hence omitted.

%% file: TD-learning.tex
\section{Analysis for TD learning (Theorem~\ref{thm:policy-evaluation})}
\label{sec:TD-learning-analysis}

As it turns out, if $|\cA|=1$ (which reduces to the case of TD learning), we can further modify the previous analysis in Section~\ref{sec:Analysis:-infinite-horizon-MDPs} to yield an improved $\frac{1}{(1-\gamma)^3}$ scaling. This forms the main content of this section, which leads to the proof of Theorem~\ref{thm:policy-evaluation} for TD learning. Akin to the Q-learning case, we proceed to establish a more general version of Theorem~\ref{thm:policy-evaluation} that covers the full $\varepsilon$-range. This is formally stated below, which subsumes Theorem~\ref{thm:policy-evaluation} as a special case. 
\begin{theorem}
	\label{thm:policy-evaluation-general}
	Consider any $\gamma\in(0,1)$ and any $\varepsilon \in \big( 0, \frac{1}{1-\gamma} \big]$. Theorem~\ref{thm:policy-evaluation} continues to hold if 
	\begin{equation}
		T\ge \frac{c_{3}\big(\log^{3}T\big)\big(\log\frac{|\mathcal{S}|T}{\delta}\big)}{\gamma^2(1-\gamma)^{3}\min\{\varepsilon,\varepsilon^{2}\}}
		\label{eq:thm:sample-size-TD}
	\end{equation}
	for some sufficiently large universal constant $c_3>0$. 
\end{theorem}
%


\subsection{Preliminary facts}
\label{sec:infinite_prelim_TD}

Before embarking on the analysis, we begin by presenting several useful preliminary facts. 
The first one is a direct
consequence of the claimed iteration complexity (\ref{eq:thm:sample-size-TD})
when $\varepsilon\leq\frac{1}{1-\gamma}$: 
\begin{equation}
	T\geq \frac{c_{3}\big(\log^{3}T\big)\big(\log\frac{|\mathcal{S}|T}{\delta}\big)}{\gamma^2(1-\gamma)^{3}\min\{\varepsilon,\varepsilon^{2}\}}
	\geq\frac{c_{3}\big(\log^{3}T\big)\big(\log\frac{|\mathcal{S}|T}{\delta}\big)}{\gamma^2(1-\gamma)^{2}},
	\label{eq:T-bound-TD}
\end{equation}
a simple fact that will be used multiple times. 
In addition, the update rule \eqref{eqn:td-learning} of TD learning can be expressed using vector/matrix notation as follows
\begin{align} 
\label{eq:iteration-rule-td}
	\bm V_t = (1-\eta_t)\bm V_{t-1} + \eta_t(\bm r + \gamma \bm P_t\bm V_{t-1}) 
	\qquad \text{for all } t\geq 1,
\end{align}
where the matrix $\bm{P}_t\in \{0,1\}^{|\cS|\times |\cS|}$ obeys
\[
	\bm{P}_t(s,s') \coloneqq \begin{cases} 1, \quad &\text{if }s'=s_t(s) \\ 0, & \text{else} \end{cases}
\]
for any $s,s'\in \cS$. 
In the sequel, we collect a few other facts concerning the range of $\bm{V}_{t}$ and learning rates.

\paragraph{Range of $\bm{V}_{t}$.}

We claim that: when the initialization $\bm{V}_{0}$ obeys $\bm{0} \leq \bm{V}_{0}\leq \frac{1}{1-\gamma} \bm{1}$, 
the TD learning iterates obey
\begin{align} \label{eq:V-i-LB-UB}
	\bm{0}\le\bm{V}_{t}\le\frac{1}{1-\gamma}\bm{1}  \qquad \text{and} 
	\qquad \|\bm{V}_{t}-\bm{V}^{\star}\|_{\infty}\le\frac{1}{1-\gamma}
	\qquad \text{for all }t\geq 0, 
\end{align}
provided that $0\leq \eta_t \leq 1$ for all $t\geq 0$. The proof follows immediately by repeating the proof of Lemma~\ref{lemma:non-negativity-Qt-Vt} (see Section~\ref{sec:infinite_prelim}) with $|\cA|=1$, and is hence omitted for brevity.

\paragraph{Learning rates.}

We shall also collect several useful results concerning the learning rates $\{\eta_{t}\}$.
Let us abuse the notation by defining the following crucial quantities:
\begin{equation}
\eta_{k}^{(t)}\coloneqq\begin{cases}
\prod_{i=1}^{t}\big(1-\eta_{i}(1-\gamma) \big), & \text{if }k=0,\\
\eta_{k}\prod_{i=k+1}^{t} \big( 1-\eta_{i}(1-\gamma) \big), & \text{if }0<k<t,\\
\eta_{t}, & \text{if }k=t. 
\end{cases}\label{def:eta-k-t}
\end{equation}
Note that this definition \eqref{def:eta-k-t} differs from the one \eqref{def:eta-i-t} used for Q-learning, and will only be employed in this section. 
Consider any iteration number $t$ satisfying
\begin{equation}
t\ge\frac{T}{c_{2}\log T}. \label{eq:t-LB-c2}
\end{equation}
Clearly, the learning rate $\eta_t$ under Assumption~\eqref{eq:thm:eta-TD} obeys
\begin{equation}
(1-\gamma)\eta_{t}\geq\frac{1-\gamma}{1+\frac{c_{1}(1-\gamma)T}{\log^{2}T}}\geq\frac{1-\gamma}{\frac{2c_{1}(1-\gamma)T}{\log^{2}T}}=\frac{\log^{2}T}{2c_{1}T}.
	\label{eq:eta-t-lower-bound-123}
\end{equation}
In what follows, we intend to bound $\eta_{k}^{(t)}$ for two cases separately.
\begin{itemize}
\item For any $i$ obeying $0\leq i\le t/2$, it is easily seen from \eqref{eq:eta-t-lower-bound-123} that
\begin{subequations}\label{eq:eta-i-t-UB12-TD}
\begin{align}
\eta_{i}^{(t)} & \le\big(1-\eta_{t/2}(1-\gamma)\big)^{t/2}\le\Big(1-\frac{\log^{2}T}{2c_{1}T}\Big)^{t/2}\leq\Big(1-\frac{\log^{2}T}{2c_{1}T}\Big)^{\frac{T}{2c_{2}\log T}} \notag\\
 & =\left\{ \Big(1-\frac{\log^{2}T}{2c_{1}T}\Big)^{\frac{2c_{1}T}{\log^{2}T}}\right\} ^{\frac{\log T}{4c_{1}c_{2}}}\leq\frac{1}{T^{2}},
	\label{eq:eta-i-t-UB1-TD}
\end{align}
where the last inequality holds as long as $c_{1}c_{2}\le 1/8$ and \eqref{eq:T-bound-TD} holds.

\item When it comes to the case with $i> t/2$, we can develop the following upper bound
\begin{equation}
\eta_{i}^{(t)}\le\eta_{i}\leq\frac{1}{c_{2}(1-\gamma)i/\log^{2}T}<\frac{2\log^{3}T}{(1-\gamma)T},\label{eq:eta-i-t-UB2-TD}
\end{equation}
which relies on Assumption~\eqref{eq:thm:eta-TD}. 
\end{subequations}
\end{itemize} 
In addition, 
given that $\bm{P}^k\bm{1}=\bm{1}$ for any integer $k>0$, it can be easily verified that
\[
	\prod_{i = k+1}^t \big(\bm{I}-\eta_i(\bm{I}-\gamma\bm P)\big) \bm{1} = \prod_{i = k+1}^t \big(1-\eta_i(1-\gamma)\big) \bm{1} ,
\]
and as a result,
\begin{align}
	\Bigg\|\prod_{i = k+1}^t \big(\bm{I}-\eta_i(\bm{I}-\gamma\bm P)\big)\Bigg\|_1 =  \prod_{i = k+1}^t \big(1-\eta_i(1-\gamma)\big).
	\label{eq:eta-prod-complex-formula}
\end{align}

\subsection{Proof of Theorem~\ref{thm:policy-evaluation-general}}

\paragraph{Step 1: decomposing the error $\bm V_t - \bm V^{\star}$. }

Taking $\bm \Delta_t \defn \bm V_t - \bm V^{\star}$, via the basic relation \eqref{eq:iteration-infinite}, the TD learning update rule can be written as 
%
\begin{align}
\bm \Delta_t 
&= (1-\eta_t)\bm \Delta_{t-1} + \eta_t\gamma \big(\bm P\bm \Delta_{t-1} + (\bm P_t - \bm P)\bm V_{t-1} \big) \notag\\
	&= \big(\bm{I}-\eta_t(\bm{I}-\gamma\bm{P})\big)\bm \Delta_{t-1} + \eta_t\gamma  (\bm P_t - \bm P)\bm V_{t-1} .
\end{align}
%
Invoking the above relation recursively then leads to
\begin{align}
\bm \Delta_t = \prod_{i = 1}^t \big(\bm{I}-\eta_i(\bm{I}-\gamma\bm P)\big)\bm \Delta_0 + \underbrace{\sum_{k = 1}^t\eta_k\prod_{i = k+1}^t \big(\bm{I}-\eta_i(\bm{I}-\gamma\bm P)\big)\gamma(\bm P_k - \bm P)\bm V_{k-1}}_{ \eqqcolon\, \bm \xi_t}.
	\label{eq:defn-xit-TD}
\end{align}

\paragraph{Step 2: controlling the first term of \eqref{eq:defn-xit-TD}.} 
With regards to the first term  of \eqref{eq:defn-xit-TD}, we make the observation that
\begin{align}
\Big\Vert \prod_{i=1}^{t}\big(\bm{I}-\eta_{i}(\bm{I}-\gamma\bm{P})\big)\bm{\Delta}_{0}\Big\Vert _{\infty} 
& \le \Big\Vert \prod_{i=1}^{t}\big(\bm{I}-\eta_{i}(\bm{I}-\gamma\bm{P})\big)\Big\Vert _{1}\left\Vert \bm{\Delta}_{0}\right\Vert _{\infty} \notag\\
 & =\left\{ \prod_{i=1}^{t}\big(1-\eta_{i}(1-\gamma)\big)\right\} \left\Vert \bm{\Delta}_{0}\right\Vert _{\infty} \notag\\
 & \le\eta_{0}^{(t)}\cdot\frac{1}{1-\gamma}\le\frac{1}{(1-\gamma)T^{2}}, 
\end{align}
where the second line arises from \eqref{eq:eta-prod-complex-formula}, 
and the last inequality holds true due to \eqref{eq:eta-i-t-UB1-TD} as long as $t\ge\frac{T}{c_{2}\log T}.$

\paragraph{Step 3: controlling the second term of \eqref{eq:defn-xit-TD}.} 

We then move on to the second term $\bm \xi_t$ in \eqref{eq:defn-xit-TD}, which admits the following expression
\begin{align}
	\bm \xi_t = \sum_{k = 1}^t\bm{z}_{k}\qquad\text{with }\bm{z}_{k}\coloneqq\eta_k\prod_{i = k+1}^t \big(\bm{I}-\eta_i(\bm{I}-\gamma\bm P)\big)\gamma(\bm P_k - \bm P)\bm V_{k-1}. 
	\label{eq:defn-xit-zk-TD}
\end{align}
Here, the summands $\{\bm{z}_{k}\}$ clearly satisfy
\[
\mathbb{E}\big[\bm{z}_{k}\mymid\bm{V}_{k-1},\cdots,\bm{V}_{0}\big]=\bm{0}.
\]
We then attempt to invoke the Freedman inequality (see Theorem~\ref{thm:Freedman})
to control this term.
Towards this end, there are several quantities that need to be calculated.  
\begin{itemize}
\item First of all, we observe that
\begin{align}
B & \coloneqq \max_{1\le k\leq t} \|\bm{z}_{k}\|_{\infty}\leq\max_{1\le k \leq t} \Big\|\eta_k\prod_{i = k+1}^t \big(\bm{I}-\eta_i(\bm{I}-\gamma\bm P)\big)\gamma(\bm P_k - \bm P)\bm V_{k-1} \Big\|_{\infty} \nonumber\\
 & \leq\max_{1\le k \leq t} \Big\|\eta_k\prod_{i = k+1}^t \big(\bm{I}-\eta_i(\bm{I}-\gamma\bm P)\big) \Big\|_{1}  \big\| (\bm P_k - \bm P)\bm V_{k-1} \big\|_{\infty} \nonumber\\
	& =\max_{1\le k \leq t} \left\{ \eta_k\prod_{i = k+1}^t \big(1-\eta_i(1-\gamma)\big) \right\} 
	\big\| (\bm P_k - \bm P)\bm V_{k-1} \big\|_{\infty} \nonumber\\	
 & \leq\max_{1\le k \leq t}\eta_{k}^{(t)}\big(\|\bm{P}_{k}\|_{1}+\|\bm{P}\|_{1}\big)\|\bm{V}_{k-1}\|_{\infty}\leq\frac{4\log^{3}T}{(1-\gamma)^2T}, \label{eq:TD_B}
\end{align}
where the third line again makes use of the relation \eqref{eq:eta-prod-complex-formula}
and the last line follows the facts $\|\bm{P}_{k}\|_{1}=\|\bm{P}\|_{1}=1$, $\|\bm{V}_{k-1}\|_{\infty} \le 1/(1-\gamma)$,
as well as the properties \eqref{eq:eta-i-t-UB12-TD}.

\item The next step is to control certain variance terms.  
Towards this, we first make note of a userful fact.  
For any given non-negative vector $\bm{u}=[u_{i}]_{1\leq i\leq|\cS|}\geq\bm{0}$
and any vector $\bm{v}$, it is easily seen that
\begin{align}
\mathsf{Var}\Big(\bm{u}^{\top}(\bm{P}_{k}-\bm{P})\bm{v}\Big) & =\sum_{i=1}^{|\cS|}u_{i}^{2}\mathsf{Var}\Big((\bm{P}_{k}-\bm{P})_{i,\cdot}\bm{v}\Big)\leq\Big\{\max_{i}|u_{i}|\Big\}\big[u_{1},\cdots,u_{|\cS|}\big]\mathsf{Var}_{\bm{P}}\left(\bm{v}\right)\nonumber \\
	& \leq \|\bm{u}\|_1 \bm{u}^{\top}\mathsf{Var}_{\bm{P}}\left(\bm{v}\right),\label{eq:Var-u-Pk-v-bound}
\end{align}
where we remind the reader of the notation $\mathsf{Var}_{\bm{P}}(\bm{v})$ in \eqref{eq:defn-Var-P-V}. 
Additionally, 
for any vector $\bm{a}=[a_j]$, let us employ the notation $\mathsf{Var}\big(\bm{a}\mymid\bm{V}_{k-1},\cdots,\bm{V}_{0}\big)$
to represent a vector whose $j$-th entry is given by $\mathsf{Var}\big(a_{j}\mymid\bm{V}_{k-1},\cdots,\bm{V}_{0}\big)$.
Armed with this notation, we obtain
\begin{align}
\mathsf{Var}\big(\bm{z}_{k}\mymid\bm{V}_{k-1},\cdots,\bm{V}_{0}\big) 
& \leq\gamma^{2}\Big\Vert \eta_{k}\prod_{i=k+1}^{t}\big(\bm{I}-\eta_{i}(\bm{I}-\gamma\bm{P})\big)\Big\Vert _{1}\left\{ \eta_{k}\prod_{i=k+1}^{t}\big(\bm{I}-\eta_{i}(\bm{I}-\gamma\bm{P})\big)\right\} \mathsf{Var}_{\bm{P}}\left(\bm{V}_{k-1}\right) \notag\\
 & =\gamma^{2}\left\{ \eta_{k}\prod_{i=k+1}^{t}\big(1-\eta_{i}(1-\gamma)\big)\right\} \left\{ \eta_{k}\prod_{i=k+1}^{t}\big(\bm{I}-\eta_{i}(\bm{I}-\gamma\bm{P})\big)\right\} \mathsf{Var}_{\bm{P}}\left(\bm{V}_{k-1}\right) \notag\\
 & \leq \eta_{k}\eta_{k}^{(t)}\prod_{i=k+1}^{t}\big(\bm{I}-\eta_{i}(\bm{I}-\gamma\bm{P})\big)\mathsf{Var}_{\bm{P}}\left(\bm{V}_{k-1}\right),
\end{align}
where the first inequality is a consequence of  \eqref{eq:Var-u-Pk-v-bound} and the definition of $\bm{z}_k$ (cf.~\eqref{eq:defn-xit-zk-TD}), the second line arises from \eqref{eq:eta-prod-complex-formula}, and the last relation results from the definition of $\eta_k^{(t)}$. 
This in turn allows us to compute 
\begin{align}
\bm{W}_{t} & \coloneqq\sum_{k=1}^{t}\mathsf{Var}\big(\bm{z}_{k}\mymid\bm{V}_{k-1},\cdots,\bm{V}_{0}\big)\le\sum_{k=1}^{t}\eta_{k}\eta_{k}^{(t)}\prod_{i=k+1}^{t}\big(\bm{I}-\eta_{i}(\bm{I}-\gamma\bm{P})\big)\mathsf{Var}_{\bm{P}}\left(\bm{V}_{k-1}\right)\nonumber\\
& \le \sum_{k=1}^{t/2}\eta_{k}^{(t)}\Big\Vert \eta_{k}\prod_{i=k+1}^{t}\big(\bm{I}-\eta_{i}(\bm{I}-\gamma\bm{P})\big)\Big\Vert _{1}\|\bm{V}_{k-1}\|_{\infty}^{2}  \bm{1}
	+\sum_{k=t/2+1}^{t}\eta_{k}\eta_{k}^{(t)}\prod_{i=k+1}^{t}\big(\bm{I}-\eta_{i}(\bm{I}-\gamma\bm{P})\big)\mathsf{Var}_{\bm{P}}\left(\bm{V}_{k-1}\right)\nonumber\\
 & \leq\sum_{k=1}^{t/2}\big(\eta_{k}^{(t)}\big)^{2}\frac{1}{(1-\gamma)^{2}} \bm{1} 
	+ \Big\{ \max_{k:\,t/2<k\leq t} \eta_{k}^{(t)} \Big\}  
	\sum_{k=t/2+1}^{t}\eta_{k}\prod_{i=k+1}^{t}\big(\bm{I}-\eta_{i}(\bm{I}-\gamma\bm{P})\big)\mathsf{Var}_{\bm{P}}\left(\bm{V}_{k-1}\right)\nonumber\\
 & \leq\frac{1}{2(1-\gamma)^{2}T^{3}} \bm{1} +\frac{2\log^{3}T}{(1-\gamma)T} \Bigg(\sum_{k=t/2+1}^{t}\eta_{k}\prod_{i=k+1}^{t}\big(\bm{I}-\eta_{i}(\bm{I}-\gamma\bm{P})\big)\Bigg)\max_{k:\,t/2\leq k<t}\mathsf{Var}_{\bm{P}}\big(\bm{V}_{k}\big)\nonumber\\
 & \leq\frac{1}{2(1-\gamma)^{2}T^{3}} \bm{1}
	+\frac{2\log^{3}T}{(1-\gamma)T}(\bm{I}-\gamma\bm{P})^{-1}\max_{k:\,t/2\leq k<t}\mathsf{Var}_{\bm{P}}\big(\bm{V}_{k}\big),
	\label{eq:Wt-bound-Var}
\end{align}
where the penultimate inequality results from \eqref{eq:eta-i-t-UB12-TD}; 
to see why the last inequality holds, observe that
\begin{align*}
&\sum_{k=t/2+1}^{t}\eta_{k}\prod_{i=k+1}^{t}\big(\bm{I}-\eta_{i}(\bm{I}-\gamma\bm{P})\big) \\
&\qquad\qquad= (\bm{I}-\gamma\bm{P})^{-1}\sum_{k=t/2+1}^{t}\eta_{k}(\bm{I}-\gamma\bm{P})\prod_{i=k+1}^{t}\big(\bm{I}-\eta_{i}(\bm{I}-\gamma\bm{P})\big) \\
&\qquad\qquad= (\bm{I}-\gamma\bm{P})^{-1}\sum_{k=t/2+1}^{t}\bigg[\prod_{i=k+1}^{t}\big(\bm{I}-\eta_{i}(\bm{I}-\gamma\bm{P})\big) - \prod_{i=k}^{t}\big(\bm{I}-\eta_{i}(\bm{I}-\gamma\bm{P})\big)\bigg] \\
&\qquad\qquad= (\bm{I}-\gamma\bm{P})^{-1} - (\bm{I}-\gamma\bm{P})^{-1}\prod_{i=t/2+1}^{t}\big(\bm{I}-\eta_{i}(\bm{I}-\gamma\bm{P})\big) \le (\bm{I}-\gamma\bm{P})^{-1},
\end{align*}
where we have used the fact that all entries of $(\bm{I}-\gamma\bm{P})^{-1}$ and $\bm{I}-\eta_{i}(\bm{I}-\gamma\bm{P})$ are non-negative.

\item In addition, we also derive the following
trivial upper bound based on \eqref{eq:Wt-bound-Var}:
\begin{align}
\big|\bm{W}_{t}\big| & \leq\frac{1}{2(1-\gamma)^{2}T^{3}}\bm{1}+\frac{2\log^{3}T}{(1-\gamma)T}\big\|(\bm{I}-\gamma\bm{P})^{-1}\big\|_{1}\max_{k:\,t/2\leq k<t}\big\|\mathsf{Var}_{\bm{P}}\big(\bm{V}_{k}\big)\big\|_{\infty}\bm{1} \notag\\
 & \leq\frac{1}{2(1-\gamma)^{2}T^{3}}\bm{1}+\frac{2\log^{3}T}{(1-\gamma)^{4}T}\bm{1}\leq\frac{3\log^{3}T}{(1-\gamma)^{4}T}\bm{1}\eqqcolon\sigma^{2}\bm{1},
	\label{eq:Wt-trivial-UB-TD}
\end{align}
where we have invoked the fact that $\|(\bm{I}-\gamma \bm{P})^{-1} \|_1 = 1/(1-\gamma)$. 
Therefore, by setting $K=\big\lceil 2\log_2\frac{1}{1-\gamma}\big\rceil $, one arrives at
\begin{equation}
\frac{\sigma^{2}}{2^{K}}\leq\frac{3\log^{3}T}{(1-\gamma)^2T}.\label{eq:sigma-2K-UB}
\end{equation}
\end{itemize}
Equipped with the preceding bounds, let us apply the Freedman inequality in
Theorem~\ref{thm:Freedman} and invoke the union bound over all entries
of $\bm{\xi}_{t}$ to show that 
\begin{align*}
|\bm{\xi}_{t}| & \le
\sqrt{8\Big(\bm{W}_{t}+\frac{\sigma^{2}}{2^{K}}\bm{1}\Big)\log\frac{8|\mathcal{S}|T\log\frac{1}{1-\gamma}}{\delta}}+\Big(\frac{4}{3}B\log\frac{8|\mathcal{S}|T\log\frac{1}{1-\gamma}}{\delta}\Big) \bm{1}\\
 & \le\sqrt{16\Big(\bm{W}_{t}+\frac{3\log^{3}T}{(1-\gamma)^2T}\bm{1}\Big)\log\frac{|\mathcal{S}|T}{\delta}}+\Big(3B\log\frac{|\mathcal{S}|T}{\delta}\Big) \bm{1}\\
	& \le\sqrt{\frac{32\big(\log^{3}T\big)\big(\log\frac{|\mathcal{S}|T}{\delta}\big)}{(1-\gamma)T}\Big((\bm{I}-\gamma\bm P)^{-1}\max_{k:\,t/2\leq k<t}\mathsf{Var}_{\bm{P}}\big(\bm{V}_{k}\big)+\frac{2}{1-\gamma}\bm{1}\Big)}+\frac{12\big(\log^{3}T\big)\big(\log\frac{|\mathcal{S}|T}{\delta}\big)}{(1-\gamma)^{2}T}\bm{1}
\end{align*}
with probability at least $1-\delta/T.$ 
Here, the second line follows since 
$$\log\frac{8|\mathcal{S}|T\log\frac{1}{1-\gamma}}{\delta} \le 2\log\frac{|\mathcal{S}|T}{\delta}$$
as long as $\frac{|\mathcal{S}|T}{\delta} \ge 8\log\frac{1}{1-\gamma}$, 
whereas the last line holds by using \eqref{eq:TD_B}, \eqref{eq:Wt-bound-Var} and \eqref{eq:sigma-2K-UB}.
Further, we make the observation that
\begin{align*}
(\bm{I}-\gamma\bm{P})^{-1}\mathsf{Var}_{\bm{P}}\big(\bm{V}^{\star}\big) & =(\bm{I}-\gamma\bm{P})^{-1}\Big(\bm{P}(\bm{V}^{\star}\circ\bm{V}^{\star})-(\bm{P}\bm{V}^{\star})\circ(\bm{P}\bm{V}^{\star})\Big)\\
 & =(\bm{I}-\gamma\bm{P})^{-1}\left(\bm{P}(\bm{V}^{\star}\circ\bm{V}^{\star})-\frac{1}{\gamma^{2}}(\bm{V}^{\star}-\bm{r})\circ(\bm{V}^{\star}-\bm{r})\right)\\
 & \le(\bm{I}-\gamma\bm{P})^{-1}\left(\bm{P}(\bm{V}^{\star}\circ\bm{V}^{\star})-\frac{1}{\gamma^{2}}\bm{V}^{\star}\circ\bm{V}^{\star}+\frac{2}{\gamma^{2}}\bm{r}\circ\bm{V}^{\star}\right)\\
 & \le(\bm{I}-\gamma\bm{P})^{-1}\left(\bm{P}(\bm{V}^{\star}\circ\bm{V}^{\star})-\frac{1}{\gamma}\bm{V}^{\star}\circ\bm{V}^{\star}+\frac{2}{\gamma^{2}}\bm{r}\circ\bm{V}^{\star}\right)\\
 & = \frac{1}{\gamma} (\bm{I}-\gamma\bm{P})^{-1}\left(\gamma\bm{P}-\bm{I}\right)\left(\bm{V}^{\star}\circ\bm{V}^{\star}\right)+\frac{2}{\gamma^{2}}(\bm{I}-\gamma\bm{P})^{-1}\left(\bm{r}\circ\bm{V}^{\star}\right)\\
 & \leq\frac{2}{\gamma^{2}}(\bm{I}-\gamma\bm{P})^{-1}\left(\bm{r}\circ\bm{V}^{\star}\right)\leq\frac{2}{\gamma^{2}(1-\gamma)^{2}}\bm{1},
\end{align*}
where the second line makes use of the basic relation $\bm{V}^{\star} = \bm{r} + \gamma\bm{P}\bm{V}^{\star}$.
As a consequence, we conclude 
\begin{align}
	(\bm{I}-\gamma\bm P)^{-1}\max_{k:\,t/2\le k<t} \mathsf{Var}_{\bm{P}}\big(\bm{V}_{k}\big) &\le (\bm{I}-\gamma\bm P)^{-1}\Big(\mathsf{Var}_{\bm{P}}\big(\bm{V}^{\star}\big) + \frac{4}{1-\gamma}\max_{k:\,t/2\le k<t} \|\bm{\Delta}_{k}\|_{\infty}\bm{1}\Big) \nonumber\\
&\le \frac{2}{\gamma (1-\gamma)^2}\Big(1 + 2\max_{k:\,t/2\le k<t}\|\bm{\Delta}_{k}\|_{\infty}\Big) \bm{1}. 
\end{align}
Here, the first inequality arises from \eqref{eq:Var-P-Var-star-gap-inf}, while the second inequality holds due to the facts that $\big\|(\bm{I}-\gamma\bm P)^{-1}\big\|_1 = 1/(1-\gamma)$. 

\paragraph{Step 4: putting everything together.}

Consequently, substituting the bounds in Steps 2-3 into \eqref{eq:defn-xit-TD} yields
\begin{align}
 \|\bm{\Delta}_{t}\|_{\infty} & \leq 30 \sqrt{\frac{\big(\log^{3}T\big)\big(\log\frac{|\mathcal{S}|T}{\delta}\big)}{\gamma^2(1-\gamma)^{3}T}\bigg(1+\max_{\frac{t}{2}\le i<t}\|\bm{\Delta}_{i}\|_{\infty}\bigg)}\qquad\text{for all}\quad t\ge\frac{T}{c_{2}\log T}. \label{eq:Delta-t-inf-recursion}
 \end{align}
%
Repeating the same argument as in Section~\ref{sec:combine_ub_lb}, we see that 
\begin{align}
\|\bm{\Delta}_{T}\|_{\infty} & \leq c_{9}\bigg(\sqrt{\frac{\big(\log^{3}T\big)\big(\log\frac{|\mathcal{S}|T}{\delta}\big)}{\gamma^2(1-\gamma)^{3}T}}+\frac{\big(\log^{3}T\big)\big(\log\frac{|\mathcal{S}|T}{\delta}\big)}{\gamma^2(1-\gamma)^{3}T}\bigg)
	\label{eq:Delta-T-UB-TD-123}
\end{align}
holds with probability at least $1-\delta$, where $c_9>0$ is some universal constant.  
As a result, one has 
\begin{align*}
	\|\bm{\Delta}_{T}\|_{\infty} & \leq\frac{1}{2}\bigg(\sqrt{\min\{\varepsilon,\varepsilon^{2}\}}+\min\{\varepsilon,\varepsilon^{2}\}\bigg)=\frac{1}{2}\left(\varepsilon+\varepsilon^{2}\right)\mathds{1}\{\varepsilon\leq1\}+\frac{1}{2}\left(\varepsilon+\varepsilon^{2}\right)\mathds{1}\{\varepsilon>1\}\leq\varepsilon,
\end{align*}
as long as the sample size satisfies the following 
\[
\frac{\big(\log^{3}T\big)\big(\log\frac{|\mathcal{S}|T}{\delta}\big)}{\gamma^2 (1-\gamma)^{3}T}\leq\frac{\min\{\varepsilon,\varepsilon^{2}\}}{c_{3}},
\]
for some constant $c_{3} \geq \max\{1, 2c_9\}$. This requirement is equivalent to condition \eqref{eq:thm:sample-size-TD} as claimed.

\subsection{Proof for Remarks~\ref{remark:anytime-TD} and \ref{remark:averaging-TD}}
\label{sec:proof-remark:anytime-TD}

\paragraph{Proof for Remark~\ref{remark:anytime-TD}.}
Let us divide the dynamics of the algorithm into two parts. 
\begin{itemize}

	\item For any $1\leq t\leq T/2$, it follows from Lemma~\ref{lemma:non-negativity-Qt-Vt} that 
		\begin{align}
			\bm{0} \leq \bm{V}_{T/2} \leq \frac{1}{1-\gamma} \bm{1}. 
		\end{align}

	\item Next, let us consider any $T/2 < t\leq T$, and set $\widetilde{t}=t - T/2$. 
		It comes from the choice \eqref{eq:new-choice-learning-rates-remark} that
\[
	\frac{1}{1+\frac{\widetilde{c}_{1}(1-\gamma)(\widetilde{t}+T/2)}{\log^{2}(\widetilde{t}+T/2+1)}}\le\eta_{\widetilde{t}+T/2}=\eta_{t}\le\frac{1}{1+\frac{\widetilde{c}_{2}(1-\gamma)(\widetilde{t}+T/2)}{\log^{2}(\widetilde{t}+T/2+1)}}
\]
\[
	\Longrightarrow\qquad\frac{1}{1+\frac{c_{1}(1-\gamma)(T/2)}{\log^{2}(T/2)}}\le\eta_{\widetilde{t}+T/2}\le\frac{1}{1+\frac{c_{2}(1-\gamma)\widetilde{t}}{\log^{2}(T/2)}}, 
\]
		provided that $\widetilde{c}_1$ and $\widetilde{c}_2$ are suitably chosen. 
		By treating $\bm{V}_{T/2}$ as the initial point and invoking Theorem~\ref{thm:policy-evaluation}, 
		we immediately establish \eqref{eq:VT-entrywise-TD-thm} under the choice \eqref{eq:new-choice-learning-rates-remark} 
		and the sample size condition \eqref{eq:sample-size-anytime-TD}. 

\end{itemize}

\paragraph{Proof for Remark~\ref{remark:averaging-TD}.}
Let us define, for each $1\leq t\leq T$, 
\begin{align}
	\varepsilon_{t}\coloneqq \min\Bigg\{ \sqrt{\frac{2c_{3}(\log^{3}T)\big(\log\frac{|\cS|T}{\delta/T}\big)}{(1-\gamma)^{3}t}}, 
	\frac{1}{1-\gamma} \Bigg\}. 
	\label{eq:defn-epsilon-t-anytime-TD}
\end{align}
Invoking the claim in Remark~\ref{remark:anytime-TD} in conjunction with Lemma~\ref{lemma:non-negativity-Qt-Vt} reveals that with probability at least $1-\delta/T$, 
\begin{align*}
	 \max_{s\in \cS} \big| V_t(s) - V^{\star}(s) \big| \leq \varepsilon_{t} 
\end{align*}
for any given $t \leq T$. 
Taking  the union bound over all $1\leq t\leq T$ implies that with probability exceeding $1-\delta$, 
\begin{align*}
	\max_{s\in \cS} \Big| \frac{1}{T}\sum_{t=1}^T V_t(s) - V^{\star}(s) \Big| 
	&\leq  \frac{1}{T}\sum_{t=1}^T \max_{s\in \cS} \big|  V_t(s) - V^{\star}(s) \big| 
	\leq  \frac{1}{T}\sum_{t=1}^T \varepsilon_{t} \\
	&\leq  \sqrt{\frac{2c_{3}(\log^{3}T)\big(\log\frac{|\cS|T^2}{\delta}\big)}{(1-\gamma)^{3}}} \frac{1}{T}\sum_{t=1}^T \frac{1}{\sqrt{t}}
	\leq 4\sqrt{\frac{c_{3}(\log^{3}T)\big(\log\frac{|\cS|T}{\delta}\big)}{(1-\gamma)^{3}T}},
\end{align*}
where the last inequality results from the elementary inequality $\sum_{t=1}^T 1/\sqrt{t} \leq 2\sqrt{T}$. This finishes the proof.

%% file: lowerbound.tex
\section{Lower bound: sub-optimality of synchronous Q-learning (Theorem~\ref{thm:LB-example})}
\label{sec:lower-bounds}

In this section, a main focus is to establish the lower bound claimed in Theorem~\ref{thm:LB-example} by analyzing synchronous Q-learning for the MDP instance constructed in Section~\ref{sec:MDP-construction-hard}. Without loss of generality, we assume
\begin{equation} \label{eq:logT_ub}
\log T\leq \frac{1}{1-\gamma}
\end{equation} 
throughout the proof; otherwise the lower bound in Theorem~\ref{thm:LB-example} is worse than the minimax lower bound $\frac{1}{(1-\gamma)^3T}$ in \citet{azar2013minimax}.

Throughout, we shall use $P_t$ to represent the sample transitions such that for any triple $(s,a,s')$,
\begin{align}
	P_t(s'\mymid s,a) \coloneqq \begin{cases}  1, \quad & \text{if } s_t(s,a)=s', \\ 0, & \text{otherwise}, \end{cases}
\end{align}
where $s_t(s,a)$ stands for the sample collected in the $t$-th iteration (see \eqref{defn:empirical-Bellman-t-inf}). 
Recognizing that state  2 is associated with a singleton action space, we shall often write
\[
	P_t(s'\mymid 2) \coloneqq   P_t(s'\mymid 2,1)
\]
for notational simplicity.

\subsection{Key quantities related to learning rates}

We find it convenient to define the following quantities (by abuse of notation) 
\begin{subequations}
\label{eq:defn-eta-kt-hard}
\begin{align}
\eta_k^{(t)} &\defn \eta_k\prod_{i = k+1}^t \big(1-\eta_i(1-\gamma p)\big) \qquad\text{for any }1 \le k < t, \\
\eta_0^{(t)} &\defn \prod_{i = 1}^t \big(1-\eta_i(1-\gamma p)\big), \\
\eta_t^{(t)} &\defn \eta_t. 
\end{align}
\end{subequations}
It is helpful to establish several basic properties about these quantities. 
As can be easily verified, 
\begin{align}
	\eta_0^{(t)} + (1-\gamma p)\sum_{k = 1}^t \eta_k^{(t)} = \prod_{i=1}^{t}(1-\widehat{\eta}_{i})+\widehat{\eta}_{1}\prod_{i=2}^{t}(1-\widehat{\eta}_{i})+\widehat{\eta}_{2}\prod_{i=3}^{t}(1-\widehat{\eta}_{i})+\cdots+\widehat{\eta}_{t-1}(1-\widehat{\eta}_{t})+\widehat{\eta}_{t}=1,
	\label{eq:eta-sum-hat}
\end{align}
where we denote $\widehat{\eta}_{i} \defn \eta_{i}(1-\gamma p)$ to simplify notation. 
Similarly, for any given integer $0\leq \tau < t$ one has
\begin{align}
	\label{eq:eta-sum-hat-tau}
	\prod_{i = \tau+1}^t \big(1-\eta_i(1-\gamma p)\big) + (1-\gamma p)\sum_{k = \tau+1}^t \eta_k^{(t)} = 1.
\end{align}

\subsection{Preliminary calculations}

Before moving forward, we record several basic relations as a result of the Q-learning update rule.

\subsubsection{Basic update rules and expansion}
Given that $Q_0 = V_0 = 0$ and that state 0 is absorbing, the  update rule \eqref{eqn:q-learning} gives
\begin{align} \label{eq:vt-0}
V_t(0) = Q_t(0,1) = \big(1-\eta_t(1-\gamma)\big)Q_{t-1}(0,1) = \prod_{i = 1}^t \big(1-\eta_i(1-\gamma)\big)Q_0(0,1) = 0
\end{align}
for all $t\geq 1$. 
Regarding state 2, the update rule \eqref{eqn:q-learning} taken together with \eqref{eq:vt-0} leads to
\begin{align} 
	V_t(2) = Q_t(2,1)  & = \big(1-\eta_t\big)Q_{t-1}(2,1) + \eta_t\big\{ r(2,1)+ \gamma P_t(2 \mymid 2)V_{t-1}(2) + \gamma P_t(0 \mymid 2)V_{t-1}(0) \big\} \notag\\
	& = \big(1-\eta_t\big)V_{t-1}(2) + \eta_t\big\{ 1 + \gamma P_t(2 \mymid 2)V_{t-1}(2)\big\},  \label{eq:vt-2}
\end{align}
and for state 3,
\begin{align} 
	V_t(3) = Q_t(3,1)  & = \big(1-\eta_t\big)Q_{t-1}(3,1) + \eta_t\big\{ r(3,1)+ \gamma V_{t-1}(3) \big\} \notag\\
	& = \big(1-\eta_t(1-\gamma)\big)V_{t-1}(3) + \eta_t .  \label{eq:vt-3}
\end{align}
Similarly, one also has
\begin{subequations}
\label{eq:Qt-1}
\begin{align}
	Q_t(1, 1) &= (1-\eta_t)Q_{t-1}(1, 1) + \eta_t\big\{ 1 + \gamma P_{t}( 1 \mymid 1, 1) V_{t-1}(1)\big\}, \\
	Q_t(1, 2) &= (1-\eta_t)Q_{t-1}(1, 2) + \eta_t\big\{ 1 + \gamma P_{t}( 1 \mymid 1, 2) V_{t-1}(1)\big\}.
\end{align}
\end{subequations}
In what follows, we shall first determine a crude range for certain quantities relates to the learning rates $\eta_t$, and then combine this with the above relations to establish the desired result.

Next, we record some elementary decomposition of $V_t(2)$.  For any iteration $t$ and $\tau<t$, one can continue the derivation in \eqref{eq:vt-2} to obtain
\begin{align}
	V_t(2) &= \big(1-\eta_t(1-\gamma p)\big)V_{t-1}(2) + \eta_t\big\{ 1+ \gamma \big( P_t(2 \mymid 2) - p \big) V_{t-1}(2)\big\} \nonumber\\
	&= \prod_{i = \tau+1}^t \big(1-\eta_i(1-\gamma p)\big)V_{\tau}(2) + \sum_{k = \tau+1}^t\eta_k\prod_{i = k+1}^t \big(1-\eta_i(1-\gamma p)\big)\big\{ 1+ \gamma \big(P_k(2 \mymid 2) - p \big)V_{k-1}(2)\big\} \nonumber\\
	&=  \prod_{i = \tau+1}^t \big(1-\eta_i(1-\gamma p)\big)V_{\tau}(2) + \sum_{k = \tau+1}^t \eta_k^{(t)}  + \sum_{k = \tau+1}^t \eta_k^{(t)}  \gamma \big(P_k(2 \mymid 2) - p \big)V_{k-1}(2) \nonumber\\
	&=  \prod_{i = \tau+1}^t \big(1-\eta_i(1-\gamma p)\big)V_{\tau}(2) + \frac{1 - \prod_{i = \tau+1}^t \big(1-\eta_i(1-\gamma p)\big) }{1-\gamma p}  + \sum_{k = \tau+1}^t \eta_k^{(t)}  \gamma \big(P_k(2 \mymid 2) - p \big)V_{k-1}(2) \nonumber\\
	&= \frac{1}{1-\gamma p} - \prod_{i = \tau+1}^t \big(1-\eta_i(1-\gamma p)\big)\left[\frac{1}{1-\gamma p} - V_{\tau}(2)\right] + \sum_{k = \tau+1}^t\eta_k^{(t)}\gamma \big(P_k(2\mymid 2) - p\big) V_{k-1}(2), 
	\label{eq:expansion-Vt2-tau}
\end{align}
where the penultimate line arises from \eqref{eq:eta-sum-hat-tau}.
In particular, in the special case where $\tau=0$ (so that $V_{\tau}(2)=V_{0}(2)=0$), this simplifies to
\begin{align} 
	V_t(2) = \frac{1-\eta_0^{(t)}}{1-\gamma p} + \sum_{k = 1}^t\eta_k^{(t)}\gamma \big(P_k(2 \mymid 2) - p \big)V_{k-1}(2),  
	\label{eq:Vt2-expansion-case1}
\end{align}
which relies on the definition of $\eta_0^{(t)}$ in \eqref{eq:defn-eta-kt-hard}. 
With similar derivation, \eqref{eq:vt-3} leads to
\begin{align} 
	V_t(3) = \frac{1}{1-\gamma}\Big[1-\prod_{i = 1}^T \big(1-\eta_i(1-\gamma)\big)\Big] = V^{\star}(3) - \frac{1}{1-\gamma}\prod_{i = 1}^T \big(1-\eta_i(1-\gamma)\big).
	\label{eq:Vt3-expansion-case3}
\end{align}

\subsubsection{Mean and variance of $V^{\star}(2) - V_T(2)$ }

We start by computing the mean $V^{\star}(2) - \mathbb{E}[V_t(2)]$. 
From the construction \eqref{eq:construction-hard-MDP}, it is easily seen that $\mathbb{E}[ P_k(2 \mymid 2) ] = p$, which together with the identity \eqref{eq:Vt2-expansion-case1} leads to
\begin{align}
	\label{eq:EV2}
	\mathbb{E} \big[V_T(2)\big] = \frac{1-\eta_0^{(T)}}{1-\gamma p}
	\qquad \text{and}\qquad
	V^{\star}(2) - \mathbb{E} \big[V_T(2)\big] = \frac{\eta_0^{(T)}}{1-\gamma p}. 
\end{align}
Similarly, applying the above argument to \eqref{eq:expansion-Vt2-tau} and rearranging terms, we immediately arrive at
\begin{align}
	 V^{\star}(2) - \mathbb{E}\big[V_{T}(2)\big] 
	= \prod_{i = \tau+1}^T \big(1-\eta_i(1-\gamma p)\big)\left[\frac{1}{1-\gamma p} - \mathbb{E}\big[ V_{\tau}(2)\big]\right]
	\label{eq:Vstar2-V2-expansion-Case2}
\end{align}
for any integer $0\leq \tau <T$.

Next, we develop a lower bound on the variance $\mathsf{Var}\big(V_T(2)\big)$. 
Towards this end, consider first a martingale sequence $\{Z_k\}_{0\leq k\leq T}$ adapted to 
a filtration $\mathcal{F}_0\subseteq  \mathcal{F}_1 \subseteq \cdots \subseteq \mathcal{F}_T$, 
namely,  $\mathbb{E} [Z_{k+1} \mymid \mathcal{F}_{k}]=0$ and $\mathbb{E} [Z_k \mymid \mathcal{F}_{k}]=Z_k$ for all $0\leq k\leq T$. 
In addition, consider any $0\leq \tau < T$, and let $W_{0}$ be a random variable such that $\mathbb{E}[W_{0}\mymid \mathcal{F}_{\tau}]=W_0$. 
Then the law of total variance together with basic martingale properties tells us that 
\begin{align}
\mathsf{Var}\left(W_{0}+\sum_{k=\tau+1}^{T}Z_{k}\right) & =\mathbb{E}\left[\mathsf{Var}\left(W_{0}+\sum_{k=\tau+1}^{T}Z_{k}\mid\mathcal{F}_{T-1}\right)\right]+\mathsf{Var}\left(\mathbb{E}\left[W_{0}+\sum_{k=\tau+1}^{T}Z_{k}\mid\mathcal{F}_{T-1}\right]\right) \notag\\
 & =\mathbb{E}\big[\,\mathsf{Var}\left(Z_{T}\mid\mathcal{F}_{T-1}\right)\big]+\mathsf{Var}\left(W_{0}+\sum_{k=\tau+1}^{T-1}Z_{k}\right)=\cdots \notag\\
 & =\sum_{k=\tau+1}^{T}\mathbb{E}\big[\,\mathsf{Var}\left(Z_{k}\mid\mathcal{F}_{k-1}\right)\big]+\mathsf{Var}\left(W_{0}\right)\geq\sum_{k=\tau+1}^{T}\mathbb{E}\big[\,\mathsf{Var}\left(Z_{k}\mid\mathcal{F}_{k-1}\right)\big].
	\label{eq:variance-expansion-martingale}
\end{align}
Consequently, for any $0\leq \tau < T-1$, it follows from the decomposition \eqref{eq:expansion-Vt2-tau} (with $\tau$ replaced by $\tau+1$) that
\begin{align}
\mathsf{Var}\big(V_{T}(2)\big) & \ge\mathbb{E}\left[\sum_{k=\tau+2}^{T}\mathsf{Var}\Big(\eta_{k}^{(T)}\gamma\big(P_{k}(2\mymid2)-p\big)V_{k-1}(2)\mymid V_{k-1}(2)\Big)\right] \notag\\
 & =\sum_{k=\tau+2}^{T}\big(\eta_{k}^{(T)}\gamma\big)^{2}p(1-p)\mathbb{E}\big[\big(V_{k-1}(2)\big)^{2}\big] \notag\\
 & \geq\frac{\left(1-\gamma\right)\left(4\gamma-1\right)}{9}\cdot\frac{1}{4(1-\gamma)^{2}}\sum_{k=\tau+2}^{T}\big(\eta_{k}^{(T)}\big)^{2} \notag\\
 & =\frac{ 4\gamma-1}{36\left(1-\gamma\right)}\sum_{k=\tau+2}^{T}\big(\eta_{k}^{(T)}\big)^{2},
	\label{eq:Var-VT-2-LB-123}
\end{align}
where the first identity relies on the fact that $P_k(2\mymid 2)$ is a Bernoulli random variable with mean $p$, and the inequality comes from the definition of $\tau$ (see \eqref{eq:defn-tau-Case2-LB}) and the choice of $p$ (see \eqref{defn-parameter-p}). As an implication, the sum of squares of $\eta_{k}^{(T)}$ plays a crucial role in determining the variance of $V_{T}(2)$.

\subsection{Lower bounds for three cases}

\subsubsection{Case 1: small learning rates ($\ceta \geq \log T$ or $0\leq \eta \leq \frac{1}{(1-\gamma)T}$)}

In this case, we focus on lower bounding $V^{\star}(2) - \mathbb{E}\big[ V_T(2)\big]$.  
In view of this identity \eqref{eq:EV2}, this boils down to controlling $\eta_0^{(T)}$.

Suppose that $\ceta > \log T$ (for rescaled linear learning rates) or $0\leq \eta < \frac{1}{(1-\gamma)T}$ (for constant learning rates).  
A little algebra then gives 
\begin{align}
\label{eq:etat-gammap-UB1}
\eta_{t}(1-\gamma p) & \leq\begin{cases}
\frac{1-\gamma p}{(1-\gamma)t\log T}=\frac{4}{3t\log T}\leq\frac{1}{2},\qquad & \text{if }\eta_{t}=\frac{1}{1+\ceta(1-\gamma)t}\\
\frac{1-\gamma p}{(1-\gamma)T}=\frac{4}{3T}\leq\frac{1}{2}, & \text{if }\eta_{t}=\eta
\end{cases}
\end{align}
for any $t\geq 1$, provided that $T\geq 15$. Consequently, one can derive
\begin{align}
	\log \eta_0^{(T)} 
	= \sum_{i = 1}^T \log \big(1-\eta_i(1-\gamma p)\big) \ge - 1.5\sum_{i = 1}^T \eta_i(1-\gamma p) \ge -2,
	\label{eq:log-eta-0T-LB}
\end{align}
where the first inequality holds due to the elementary fact $\log(1-x)\geq -1.5x$ for all $0\leq x \leq 0.5$, 
and the last inequality follows from the following bound (which makes use of \eqref{eq:etat-gammap-UB1})
\[
\sum_{i=1}^{T}\eta_{i}(1-\gamma p)\leq\begin{cases}
\frac{3}{4\log T}\sum_{i=1}^{T}\frac{1}{i}\leq1, & \text{if }\eta_{t}=\frac{1}{1+\ceta(1-\gamma)t}\\
\frac{4}{3T}\sum_{i=1}^{T}1=\frac{4}{3}, & \text{if }\eta_{t}=\eta.
\end{cases}
\]
Combining the above result with the properties \eqref{eq:EV2} and \eqref{eq:log-eta-0T-LB} then yields
\begin{equation}
\label{eqn:Vstar-lb-violin}
	V^{\star}(2) - \mathbb{E} \big[V_T(2)\big]  =  \frac{\eta_0^{(T)}}{1-\gamma p} \ge  \frac{e^{-2}}{1-\gamma p} = \frac{3}{4e^2(1-\gamma)}.
\end{equation}
This taken together with \eqref{eq:bias-variance-decomposition} gives
\begin{equation}
	\label{eq:MSE-V2-case1}
	\mathbb{E}\big[ \big( V^{\star}(2) - V_T(2) \big)^2 \big] 
	\geq \big( V^{\star}(2) - \mathbb{E} \big[V_T(2)\big] \big)^2 
	\geq \frac{9}{16e^4(1-\gamma)^2}.
\end{equation}

\subsubsection{Case 2: large learning rates ($\ceta \leq 1-\gamma$ or $ \eta \geq \frac{1}{(1-\gamma)^2T}$)}

By virtue of \eqref{eq:Vstar2-V2-expansion-Case2},  the mean gap $V^{\star}(2) - \mathbb{E}\big[V_{T}(2)\big]$ depends on two factors: (i) the choice of the learning rates, and (ii) the gap between $\frac{1}{1-\gamma p}$ and $\mathbb{E}\big[ V_{\tau}(2)\big]$, where $\tau$ is an integer obeying $0\leq \tau <T$. 
To control the factor (ii), we need to choose $\tau$ properly. 
Let us start by considering the simple scenario with $\mathbb{E}\big[ \big( V_{T}(2) \big)^2 \big] < \frac{1}{4(1-\gamma)^2}$, for which we have
\begin{align}
	V^{\star}(2) -  \mathbb{E}\big[  V_{T}(2)  \big] 
	\geq \frac{3}{4(1-\gamma)} -  \sqrt{\mathbb{E}\big[ \big( V_{T}(2) \big)^2 \big]} 
	\geq \frac{1}{4(1-\gamma)}. 
	\label{eq:mean-diff-large-tau-T}
\end{align}
Here, we have used \eqref{eq:optimal-V-Q-hardMDP} and the elementary fact $\mathbb{E}[X]\leq \sqrt{\mathbb{E}[X^2]}$. Consequently, it remains to look at the scenario obeying $\mathbb{E}\big[ \big( V_{T}(2) \big)^2 \big] \geq \frac{1}{4(1-\gamma)^2}$, towards which we propose to set $\tau$ as follows
\begin{align}
	\label{eq:defn-tau-Case2-LB}
	\tau \defn \min\left\{0\leq \tau' \leq T-1 \,\,\Big|\,\, \mathbb{E}\big[ \big( V_{t}(2) \big)^2 \big] 
	\ge \frac{1}{4(1-\gamma)^2} ~\text{for all } \tau'+1 \leq t \leq  T \right\} .
\end{align}
Clearly, $\tau$ is well-defined in this scenario and obeys (in view of both \eqref{eq:defn-tau-Case2-LB} and the initialization $V_0=0$)
\begin{align}
	\label{eq:E-Vtau-UB-case2}
	\mathbb{E}\big[ \big( V_{\tau}(2) \big)^2 \big] 
	< \frac{1}{4(1-\gamma)^2} .
\end{align}
Our analysis for this scenario is divided into three subcases based on the size of the learning rates.

\paragraph{\em Case 2.1.} Consider the case where
\begin{align}
	\label{eq:case-21-condition}
\prod_{i = \tau+1}^T \big(1-\eta_i(1-\gamma p)\big) \geq \frac{1}{2}.
\end{align}
Invoke \eqref{eq:Vstar2-V2-expansion-Case2} to deduce that
\begin{align*}
	V^{\star}(2) - \mathbb{E}\big[V_{T}(2)\big] &= \prod_{i = \tau+1}^T \big(1-\eta_i(1-\gamma p)\big)\Big[\frac{1}{1-\gamma p} - \mathbb{E}\big[V_{\tau}(2)\big]\Big] \\
	&\ge \prod_{i = \tau+1}^T \big(1-\eta_i(1-\gamma p)\big)\Big[\frac{3}{4(1-\gamma) } - \sqrt{\mathbb{E}\big[ \big( V_{\tau}(2) \big)^2 \big] }\Big] \\
&\ge \prod_{i = \tau+1}^T \big(1-\eta_i(1-\gamma p)\big)\frac{1}{4(1-\gamma)} \ge \frac{1}{8(1-\gamma)}, 
\end{align*}
where the second line makes use of the definition \eqref{defn-parameter-p} and the elementary fact $\mathbb{E}[X]\leq \sqrt{\mathbb{E}[X^2]}$, 
and the last line relies on the inequalities \eqref{eq:E-Vtau-UB-case2} and \eqref{eq:case-21-condition}.

\paragraph{\em Case 2.2.} We now move on to the case where 
\begin{align}
0 \le \prod_{i = \tau+1}^T \big(1-\eta_i(1-\gamma p)\big) \le \frac{1}{2}.
	\label{eq:prod-etai-UB-lower-bound}
\end{align}
We intend to demonstrate that the variance of $V_T(2)$---and hence the typical size of its fluctuation---is too large. 
In view of the observation \eqref{eq:Var-VT-2-LB-123}, it boils down to lower bounding $\sum_{k=\tau+2}^{T}\big(\eta_{k}^{(T)}\big)^{2}$, which we accomplish as follows. 
\begin{itemize}
	\item Consider constant learning rates $\eta_k=\eta$, and suppose that $\eta$ obeys $ \frac{1}{(1-\gamma)^2T} < \eta \leq 1< \frac{1}{1-\gamma p}$. 
		It is readily seen that $\eta_{k}^{(T)}=\eta\big(1-\eta(1-\gamma p)\big)^{T-k}$ for any $k\geq 1$. 
We claim that it suffices to focus on the scenario where
\begin{align}
	\tau \leq T-2. 
	\label{eq:assumption-tau-T-minus-2}
\end{align}
In fact, if $\tau \geq T-1$, then the definition \eqref{eq:defn-tau-Case2-LB} of $\tau$ necessarily requires that
\[
	\mathbb{E}\big[V_{T-1}(2)\big] \leq \sqrt{ \mathbb{E}\big[ \big( V_{T-1}(2) \big)^2 \big] } < \frac{1}{2(1-\gamma)}.
\]
In view of \eqref{eq:EV2} (with $T$ replaced by $T-1$), a little algebra shows that this is equivalent to $\big(1-\eta(1-\gamma p)\big)^{T-1} \ge 1/3$,
and hence $\big(1-\eta(1-\gamma p)\big)^{T} \ge 1/9$.
In turn, this combined with \eqref{eq:EV2} leads to
	\begin{align}
		V^{\star}(2) - \mathbb{E}\big[V_{T}(2)\big] 
		= \frac{ \big(1-\eta(1-\gamma p)\big)^{T}}{1-\gamma p} 
		= \frac{3 \big(1-\eta(1-\gamma p)\big)^{T}}{4( 1-\gamma )} 
		\ge \frac{1}{12(1-\gamma)},
		\label{eq:large-tau-case-mean-diff}
	\end{align}
which already suffices for our purpose. 

Next, assuming that \eqref{eq:assumption-tau-T-minus-2} holds, one can derive 
\begin{align}
\sum_{k=\tau+2}^{T}\big(\eta_{k}^{(T)}\big)^{2} & =\sum_{k=\tau+2}^{T}\eta^{2}\big(1-\eta(1-\gamma p)\big)^{2(T-k)}=\frac{\eta^{2}\big[1-\big(1-\eta(1-\gamma p)\big)^{2(T-\tau-1)}\big]}{1-\big(1-\eta(1-\gamma p)\big)^{2}} \notag\\
	& \ge\frac{\eta^{2}/2}{1-\big( 1- \eta(1-\gamma p)\big)^{2}}\ge\frac{3\eta}{16(1-\gamma)},
	\label{eq:eta-k-prod-case-1234}
\end{align}
where the first inequality holds since (from the assumptions \eqref{eq:prod-etai-UB-lower-bound} and $\tau\leq T-2$)
\[
	0\leq \big(1-\eta(1-\gamma p)\big)^{2(T-\tau-1)} \le \big(1-\eta(1-\gamma p)\big)^{T-\tau} 
		= \prod_{i = \tau+1}^T \big(1-\eta_i(1-\gamma p)\big) \le \frac{1}{2}, 
\]
and the last inequality follows since 
\[
	0 \le 1-\big( 1- \eta(1-\gamma p)\big)^{2} =  1-\Big(1-\frac{4\eta(1-\gamma)}{3}\Big)^{2} \le \frac{8\eta(1-\gamma)}{3}.
\]
Substituting \eqref{eq:eta-k-prod-case-1234} into \eqref{eq:Var-VT-2-LB-123}, we obtain
\begin{align}
\mathsf{Var}\big(V_{T}(2)\big) & \ge\frac{4\gamma-1}{36(1-\gamma)}\sum_{k=\tau+1}^{T}\big(\eta_{k}^{(T)}\big)^{2}\ge\frac{2}{36(1-\gamma)}\cdot\frac{3\eta}{16(1-\gamma)} \notag\\
 & =\frac{\eta}{96(1-\gamma)^{2}}\geq\frac{1}{96(1-\gamma)^{4}T} ,\label{eq:case2_constant_stepsize}
\end{align}
 provided that $\gamma \geq 3/4$ (so that $4\gamma - 1 \geq 2$). 
Here, the last inequality is valid since either $\eta\geq \frac{1}{(1-\gamma)^2T}$.

\item We then move on to linearly rescaled learning rates with $\eta_t = \frac{1}{1+\ceta(1-\gamma)t}$ for some $0\leq \ceta<1-\gamma$. 
Towards this, we first make the observation that 
\begin{align}
\frac{\eta_{k-1}^{(T)}}{\eta_{k}^{(T)}} & =\frac{\eta_{k-1}\big(1-\eta_{k}(1-\gamma p)\big)}{\eta_{k}}=\frac{1-\frac{4}{3}(1-\gamma)\eta_{k}}{1-\big(\frac{1}{\eta_{k}}-\frac{1}{\eta_{k-1}}\big)\eta_{k}}=\frac{1-\frac{4}{3}(1-\gamma)\eta_{k}}{1-\ceta(1-\gamma)\eta_{k}}=1-\frac{\big(\frac{4}{3}-\ceta\big)(1-\gamma)\eta_{k}}{1-\ceta(1-\gamma)\eta_{k}} \notag\\
 & \le1-(1-\gamma)\eta_{k}\le1-(1-\gamma)\eta_{T},
 \label{eq:ratio-etakT-eta-kminus1}
\end{align}	
with the proviso that $\ceta < 1-\gamma \leq 1/3$ (as long as $\gamma \geq 2/3$).  
By defining $\tau' \defn T-\frac{1}{(1-\gamma)\eta_{T}}$, one can deduce that
\begin{align}
\sum_{k=\tau+2}^{T}\big(\eta_{k}^{(T)}\big)^{2} & \ge\sum_{k=\max\{\tau+2,\tau'+1\}}^{T}\big(\eta_{k}^{(T)}\big)^{2}\ge\frac{1}{T-\max\{\tau+1,\tau'\}}\left[\sum_{k=\max\{\tau+2,\tau'+1\}}^{T}\eta_{k}^{(T)}\right]^{2} \notag\\
 & \ge(1-\gamma)\eta_{T}\left[\sum_{k=\max\{\tau+2,\tau'+1\}}^{T}\eta_{k}^{(T)}\right]^{2} ,
	\label{eq:sum-eta-kT-lower-bound-234}
\end{align}
where the penultimate inequality comes from the Cauchy-Schwarz inequality. 
In addition, recognizing that $\eta_{k_1}^{(T)} \leq \big(1-(1-\gamma)\eta_{T}\big)^{k_2-k_1}\eta_{k_2}^{(T)}$ for any $k_2 \ge k_1$ (see \eqref{eq:ratio-etakT-eta-kminus1}), one has 
\begin{align*}
	\sum_{k=\tau'+1}^{T}\eta_{k}^{(T)} & = \sum_{k=\tau'+1}^{T}\eta_{k}^{(T)}, \notag\\
\sum_{k=\max\{2\tau'-T+1,1\}}^{\tau'}\eta_{k}^{(T)} & \leq\big(1-(1-\gamma)\eta_{T}\big)^{T-\tau'}\sum_{k=\tau'+1}^{T}\eta_{k}^{(T)}, \notag\\
\sum_{k=\max\{3\tau'-2T+1,1\}}^{2\tau'-T}\eta_{k}^{(T)} & \leq \big(1-(1-\gamma)\eta_{T}\big)^{2(T-\tau')}\sum_{k=\tau'+1}^{T}\eta_{k}^{(T)} , \notag\\
 & \cdots
\end{align*}
Summing these inequalities up and rearranging terms, we reach
\begin{align*}
\sum_{k=\tau'+1}^{T}\eta_{k}^{(T)} & \ge\frac{\sum_{k=1}^{T}\eta_{k}^{(T)}}{1+\big(1-(1-\gamma)\eta_{T}\big)^{T-\tau'}+\big(1-(1-\gamma)\eta_{T}\big)^{2(T-\tau')}+\ldots}\geq\frac{\sum_{k=1}^{T}\eta_{k}^{(T)}}{\frac{1}{1-\big(1-(1-\gamma)\eta_{T}\big)^{T-\tau'}}}\\
 & =\left(1-\big(1-(1-\gamma)\eta_{T}\big)^{T-\tau'}\right)\sum_{k=1}^{T}\eta_{k}^{(T)}\geq (1-e^{-1})\sum_{k=1}^{T}\eta_{k}^{(T)},
\end{align*}
which relies on the fact $\big(1-(1-\gamma)\eta_{T}\big)^{T-\tau'} = \big(1-1/(T-\tau')\big)^{T-\tau'} \le e^{-1}$ (using the definition of $\tau'$).
Consequently, it is easily seen that
\begin{align*}
\sum_{k=\max\{\tau+2,\tau'+1\}}^{T}\eta_{k}^{(T)} 
	& = \min \left\{ \sum_{k=\tau+2}^{T}\eta_{k}^{(T)} , \sum_{k=\tau'+1}^{T}\eta_{k}^{(T)} \right\} 
	\ge (1-e^{-1})\sum_{k = \tau+2}^T \eta_k^{(T)} \\
	& \overset{\text{(i)}}{=} (1-e^{-1})\left[1-\prod_{i = \tau+2}^t \big(1-\eta_i(1-\gamma p)\big)\right]\frac{1}{1-\gamma p} \\
	&\overset{\text{(ii)}}{\ge} \left[1-\frac{1}{2\big(1-\eta_{\tau+1}(1-\gamma p)\big)}\right]\frac{1-e^{-1}}{1-\gamma p} 
	\overset{\text{(iii)}}{\ge} \frac{1-e^{-1}}{4(1-\gamma p)} \ge \frac{3}{32(1-\gamma )} .
\end{align*}
Here, (i) and (ii) follow from \eqref{eq:eta-sum-hat-tau} and \eqref{eq:prod-etai-UB-lower-bound}, respectively, while (iii) holds since 
$$\eta_{\tau+1}(1-\gamma p) \le  1-\gamma p = \frac{4(1-\gamma)}{3} \le \frac{1}{3}$$ as long as $\gamma \ge 3/4$. 
Substitution into \eqref{eq:sum-eta-kT-lower-bound-234} yields
\begin{align}
\sum_{k=\tau+2}^{T}\big(\eta_{k}^{(T)}\big)^{2} & 
	\ge\frac{9\eta_{T}}{1024(1-\gamma)}.
	\label{eq:sum-eta-kT-lower-bound-567}
\end{align}

Substituting the above bound into \eqref{eq:Var-VT-2-LB-123}, we obtain
\begin{align}
\mathsf{Var}\big(V_{T}(2)\big) & \ge\frac{4\gamma-1}{36(1-\gamma)}\sum_{k=\tau+1}^{T}\big(\eta_{k}^{(T)}\big)^{2}\ge\frac{2}{36(1-\gamma)}\cdot\frac{9\eta_{T}}{1024(1-\gamma)} \notag\\
 & =\frac{\eta_{T}}{2048(1-\gamma)^{2}}\geq\frac{1}{4096(1-\gamma)^{4}T} , \label{eq:case2_rescaled_stepsize}
\end{align}
 provided that $\gamma \geq 3/4$ (so that $4\gamma - 1 \geq 2$). 
Here, the last inequality is valid since $\eta_T=\frac{1}{1+\ceta (1-\gamma) T} \geq \frac{1}{1+(1-\gamma)^2 T}\geq \frac{1}{2(1-\gamma)^2 T}$ as long as $T\geq \frac{1}{(1-\gamma)^2}$.

\end{itemize}

\paragraph{Putting all this together.}

With the above bounds in place, it is readily seen that either the bias is too large (see \eqref{eq:large-tau-case-mean-diff})
or the variance is too large (see \eqref{eq:case2_constant_stepsize} and \eqref{eq:case2_rescaled_stepsize}). These bounds taken collectively with \eqref{eq:bias-variance-decomposition} yield
\begin{align}
	\mathbb{E}\big[ \big( V^{\star}(2) - V_T(2) \big)^2 \big] 
	&\geq \big( V^{\star}(2) - \mathbb{E} \big[V_T(2)\big] \big)^2 + \mathsf{Var}\big(V_{T}(2)\big)  \notag\\
	&\geq \min \left\{   \frac{1}{144(1-\gamma)^2}, \frac{1}{96(1-\gamma)^4T} , \frac{1}{4096(1-\gamma)^4T}  \right\} = \frac{1}{4096(1-\gamma)^4T} ,
	\label{eq:MSE-V2-case1}
\end{align}
provided $T\geq \frac{1}{(1-\gamma)^2}$.

\input{case3.tex}

\subsection{Proof of Lemma~\ref{lem:optimal-V-Q-hardMDP}}
\label{sec:proof-lem:optimal-V-Q-hardMDP}

Given that state $0$ is an absorbing state with zero immediate reward, it is easily seen that
\[
V^{\pi}(0)=0\qquad\text{for all }\pi
\qquad \Longrightarrow \qquad V^{\star}(0) = Q^{\star}(0,1)  =0.
\]
%
Moreover, by construction, taking action 1 and taking action 2 in state 1 result
in the same behavior (in terms of both the reward function and the
associated transition probability), and as a consequence,
\begin{equation}
Q^{\star}(1,1)=Q^{\star}(1,2)=V^{\star}(1).\label{eq:Q1-Q2-Vstar}
\end{equation}
From Bellman's equation, we can thus deduce that
\[
	Q^{\star}(1,1)=r(1,1)+\gamma P(0\mymid1,1)V^{\star}(0)+\gamma P(1\mymid1,1)V^{\star}(1),
\]
which in conjunction with (\ref{eq:Q1-Q2-Vstar}) and a little algebra leads to
\[
	V^{\star}(1)=\frac{r(1,1)+\gamma P(0\mymid1,1)V^{\star}(0)}{1-\gamma P(1\mymid1,1)}=\frac{1}{1-\gamma p}=\frac{3}{4(1-\gamma)}.
\]
Here, the second identity follows since $V^{\star}(0)=0$, and the
third identity makes use of \eqref{defn-parameter-p}. 
The calculation for $V^{\star}(2)$ and $Q^{\star}(2,1)$  follows from an identical argument and is hence omitted. 
%

Turning to state $3$, by Bellman's equation, we have
\[
V^{\star}(3)  = Q^{\star}(3,1)=r(3,1)+\gamma P(3\mymid 3,1)V^{\star}(3) = 1 + \gamma V^{\star}(3)  ,
\]
which leads to $V^{\star}(3) = \frac{1}{1-\gamma}$.

%% file: case3.tex
\subsubsection{Case 3: medium learning rates ($ 1-\gamma < \ceta < \log T$ or $ \frac{1}{(1-\gamma)T} \leq \eta \leq \frac{1}{(1-\gamma)^2T}$)}

Throughout this case, we assume that 
\begin{equation}
	\label{eq:eta0_lb_case3}
	\eta_{0}^{(T)}\leq \frac{1}{75}.
\end{equation}
In fact, if $\eta_{0}^{(T)}>1/75$, then the scenario becomes much easier to cope with. 
To see this, applying the previous result~\eqref{eqn:Vstar-lb-violin} and recalling the choice \eqref{defn-parameter-p} of $p$ immediately yield
\begin{align}
	V^{\star}(2) - \mathbb{E}\big[V_T(2)\big] \ge \frac{\eta_0^{(T)}}{1-\gamma p} > \frac{1}{100(1-\gamma)},
\end{align}
which together with \eqref{eq:bias-variance-decomposition} and the assumption $T\geq \frac{1}{(1-\gamma)^2}$ yields
\begin{align}
	\mathbb{E}\big[ \big( V^{\star}(2) - V_T(2) \big)^2 \big] 
	&\geq \big( V^{\star}(2) - \mathbb{E} \big[V_T(2)\big] \big)^2 
	 \geq \frac{1}{10000(1-\gamma)^2}
	 \geq \frac{1}{10000(1-\gamma)^4 T} .
	\label{eq:MSE-V2-case3-456}
\end{align}

We now turn our attention to the dynamics w.r.t.~state $1$
and its associated value function $V_{t}(1)$ under the condition \eqref{eq:eta0_lb_case3}.

\paragraph{Two auxiliary sequences.}
Towards this, we first eliminate the effect of initialization on $Q_t(1, a)$ by introducing the following auxiliary sequence
\begin{align}
	\label{eqn:Qhat-schumann}
	\widehat{Q}_t(a) = (1-\eta_t)\widehat{Q}_{t-1}(a) + \eta_t\big\{ 1 + \gamma P_{t}(1\mymid 1, a)\widehat{V}_{t-1}\big\},
\end{align}
with 
\begin{align*}
	\widehat{V}_{t-1} \coloneqq \max_a \widehat{Q}_{t-1}(a) \qquad \text{and} \qquad \widehat{Q}_0(a) \coloneqq Q^{\star}(1, a) = \frac{1}{1-\gamma p},
\end{align*}
where we recall the value of $ Q^{\star}(1, a)$ from Lemma~\ref{lem:optimal-V-Q-hardMDP}.
In other words, $\{\widehat{Q}_t(a)\}$ is essentially a Q-learning sequence when initialized at the ground truth. 
Despite the difference in initialization, 
we claim that the discrepancy between $\widehat{Q}_t(a)$ and $Q_t(1, a)$ can be well controlled in the following sense:
\begin{align} 
\label{eq:Qt}
	Q_t(1, a) \ge \widehat{Q}_t(a) - \frac{1}{1-\gamma}\prod_{i = 1}^t \big(1-\eta_i(1-\gamma)\big) ,
	\qquad  a \in \{1,2\},
\end{align}
which shall be justified in Section~\ref{sec:proof-auxiliary-results-Case3}. As we shall discuss momentarily, the gap $\frac{1}{1-\gamma}\prod_{i = 1}^t \big(1-\eta_i(1-\gamma)\big)$ is sufficiently small for this case.

Further, in order to control $\widehat{Q}_t(a)$, we find it convenient to introduce another auxiliary sequence as follows
\begin{align}
	\overline{Q}_t = (1-\eta_t)\overline{Q}_{t-1} + \eta_t\big\{ 1 + \gamma P_{t}(1\mymid 1, 1)\overline{Q}_{t-1}\big\} 
	\qquad \text{and} \qquad 
	\overline{Q}_0 = V^{\star}(1) = \frac{1}{1-\gamma p},
	\label{eq:sequence-overline-Qt}
\end{align}
which can be interpreted as a Q-learning sequence when there is only a single action (so that there is no max operator involved). 
In view of the basic fact that $\widehat{V}_t =\max_a \widehat{Q}_t(a) \ge \widehat{Q}_t(1)$, we can easily verify that
\begin{align}
	\widehat{Q}_t(1) \ge (1-\eta_t)\widehat{Q}_{t-1}(1) + \eta_t\big\{1 + \gamma P_{t}(1\mymid 1, 1)\widehat{Q}_{t-1}(1)\big\}
	\ge \overline{Q}_t,
	\label{eq:ordering-Qt-hat-overline}
\end{align}
allowing one to lower bound $\widehat{V}_t$ by controlling $\overline{Q}_t$.

\paragraph{A useful lower bound on the auxiliary sequence~\eqref{eqn:Qhat-schumann}.}

In what follows, let us establish a useful lower bound on the sequence~\eqref{eqn:Qhat-schumann} introduced above. 
 Then we claim that there exists some $\tau \leq T$ (see \eqref{eq:defn-tau-case31} and \eqref{eq:defn-tau-case32}) such that  
\begin{align}
	\label{eqn:Tailbound-vhat-rolland}
	\mathbb{P}\Big\{\widehat{V}_t \ge \frac{1}{4(1-\gamma)} \Big\} \ge \frac{1}{2},\qquad \text{for}~t \ge \tau.
\end{align}
The auxiliary sequence constructed in \eqref{eq:sequence-overline-Qt} plays a crucial role in establishing this claim.

\begin{proof}[Proof of the claim \eqref{eqn:Tailbound-vhat-rolland}]

We intend to employ the sequence $\overline{Q}_{t}$ (cf.~\eqref{eq:sequence-overline-Qt}) to help control $\widehat{V}_t$. 
It is first observed that the sequence $\overline{Q}_t$ admits the following decomposition (akin to the derivation in \eqref{eq:Vt2-expansion-case1})
\begin{align}
\notag \overline{Q}_{t} & =\big(1-\eta_{t}(1-\gamma p)\big)\overline{Q}_{t-1}+\eta_{t}\big\{1+\gamma\big(P_{t}(1\mymid1,1)-p\big)\overline{Q}_{t-1}\big\}\\
\notag & =\prod_{i=1}^{t}\big(1-\eta_{i}(1-\gamma p)\big)\overline{Q}_{0}+\sum_{k=1}^{t}\eta_{k}\prod_{i=k+1}^{t}\big(1-\eta_{i}(1-\gamma p)\big)\big\{1+\gamma\big(P_{k}(1\mymid1,1)-p\big)\overline{Q}_{k-1}\big\}\\
 \notag & =\eta_{0}^{(t)}  \frac{1}{1-\gamma p} +\sum_{k=1}^{t}\eta_{k}^{(t)}+\sum_{k=1}^{t}\eta_{k}^{(t)}\gamma\big(P_{k}(1\mymid1,1)-p\big)\overline{Q}_{k-1}\\
 & =\frac{1}{1-\gamma p}+\sum_{k=1}^{t}\underset{\eqqcolon\,z_{k}}{\underbrace{\eta_{k}^{(t)}\gamma\big(P_{k}(1\mymid1,1)-p\big)\overline{Q}_{k-1}}},
	\label{eqn:qbar-up-decomp}
\end{align}
where the last line results from \eqref{eq:eta-sum-hat}. In order to lower bound $\overline{Q}_{t}$, it boils down to controlling $\sum_k z_k$.

Note that the sequence $\{z_k\}$ defined above is a martingale satisfying
\begin{align*}
	\mathbb{E}\big[ z_k\mymid &P_{k-1}(1\mymid 1, 1), \ldots, P_{1}(1\mymid 1, 1)\big] = 0 \\
	\text{and} &\qquad |z_k| \le \max_{1 \le k \le t} \eta_k^{(t)} \cdot \frac{\gamma p}{1-\gamma},
\end{align*}
where the last inequality follows from the basic property $0\leq \overline{Q}_{k-1}\leq \frac{1}{1-\gamma}$ (akin to Lemma~\ref{lemma:non-negativity-Qt-Vt}) and the fact that $\big| P_{k}(1\mymid1,1)-p\big|\leq \max \{p, 1-p\}=p$ since $p = (4\gamma-1)/(3\gamma)$ and $\gamma \geq 3/4$. We intend to invoke Freedman's inequality to control \eqref{eqn:qbar-up-decomp}. 
Armed with these properties and the fact that $P_k(1\mymid 1,1)$ is a Bernoulli random variable with mean $p$, we obtain
\begin{align*}
\sum_{k = 1}^t \mathsf{Var}\Big(z_k\mymid P_{k-1}(1\mymid 1, 1), \ldots, P_{1}(1\mymid 1, 1)\Big) &= \sum_{k = 1}^t \big( \eta_k^{(t)}\big)^2p(1-p)\big(\gamma\overline{Q}_{k-1}\big)^2 \\
&\le \max_{1 \le k \le t} \eta_k^{(t)} \cdot \sum_{k = 1}^t \eta_k^{(t)} \cdot \frac{1}{3(1-\gamma)} \le \frac{\max_{1 \le k \le t} \eta_k^{(t)}}{4(1-\gamma)^2}.
\end{align*}
Here, the penultimate inequality relies on the fact $0\leq \overline{Q}_{k-1}\leq \frac{1}{1-\gamma}$ (akin to Lemma~\ref{lemma:non-negativity-Qt-Vt}) and the choice of $p$ (see definition~\eqref{defn-parameter-p}), 
whereas the last inequality results from the following condition (derived through \eqref{eq:eta-sum-hat})
\[
	\sum_{k = 1}^t \eta_k^{(t)} = \big(1 - \eta_0^{(t)}\big)\frac{1}{1-\gamma p} \le \frac{1}{1-\gamma p} = \frac{3}{4(1-\gamma)} .
\]
Applying Freedman's inequality (see \eqref{eq:Freedman-special}) then yields 
\begin{align} 
	\label{eq:variance}
\mathbb{P}\left\{\bigg|\sum_{k = 1}^t z_k\bigg| \ge \sqrt{\frac{4\max_{1 \le k \le t} \eta_k^{(t)}}{(1-\gamma)^2}\log\frac{2}{\delta}} + \frac{4\max_{1 \le k \le t} \eta_k^{(t)}}{3(1-\gamma)}\log\frac{2}{\delta} \right\} \le \delta.
\end{align}

As an implication of the preceding result, a key ingredient towards bounding $\sum_{k = 1}^t z_k$ lies in controlling the quantity $\max_{1 \le k \le t} \eta_k^{(t)}$. To do so, we claim for the moment that there exists some $\tau\leq T$ such that 
\begin{align}
	\label{eqn:schubert-winterreise}
	\max_{1 \le k \le t} \eta_k^{(t)} \le \frac{1}{50},\qquad\text{for}~t \ge \tau,
\end{align} 
whose proof is postponed to Section~\ref{sec:proof-auxiliary-results-Case3}. 
In light of this claim, setting $\delta = 1/2$ in the expression~\eqref{eq:variance} yields $$\sum_{k = 1}^t z_k \geq -\frac{1}{2(1-\gamma)}$$ with probably at least $1/2$. Combining this with the decomposition~\eqref{eqn:qbar-up-decomp} and  the property \eqref{eq:ordering-Qt-hat-overline}, we arrive at 
\begin{align*}
	\widehat{V}_t \geq \widehat{Q}_t(1) \geq \overline{Q}_t \geq \frac{1}{1-\gamma p} -\frac{1}{2(1-\gamma)}
	=  \frac{1}{4(1-\gamma)}
\end{align*}
with probability at least $1/2$, where the last identity relies on the choice of $p$ (see the definition~\eqref{defn-parameter-p}).
This establishes the advertised claim \eqref{eqn:Tailbound-vhat-rolland}. 
\end{proof}

\paragraph{Main proof.}
With the property \eqref{eqn:Tailbound-vhat-rolland} in place, we are positioned to prove our main result. 
Towards this, we find it convenient to define
\begin{subequations}
\begin{align}
	\Delta_t(a) &\defn \widehat{Q}_t(a) - Q^{\star}(1, a), \qquad a = 1, 2; \\
	\Delta_{t,\mathsf{max}} &\defn \max_a \Delta_t(a). 
\end{align}
\end{subequations}
The goal is thus to control $\Delta_{T,\mathsf{max}}$; in fact, we intend to show that $\Delta_{T,\mathsf{max}}$ is in expectation excessively large, resulting in an ``over-estimation'' issue that hinders convergence.
Towards this, it follows from the iterative update rule \eqref{eqn:Qhat-schumann} that 
\begin{align*}
\Delta_t(a) &= (1-\eta_t)\Delta_{t-1}(a) + \eta_t\big(1 + \gamma P_{t}(1\mymid 1, a)\widehat{V}_{t-1} - Q^{\star}(1, a)\big) \\
&= (1-\eta_t)\Delta_{t-1}(a) + \eta_t\gamma\big(P_{t}(1\mymid 1, a)\widehat{V}_{t-1} - pV^{\star}(1)\big) \\
	&= (1-\eta_t)\Delta_{t-1}(a) + \eta_t\gamma\big(p \big( \widehat{V}_{t-1} - V^{\star}(1) \big) + \big( P_{t}(1\mymid 1, a) - p \big) \widehat{V}_{t-1}\big) \\
&= (1-\eta_t)\Delta_{t-1}(a) + \eta_t\gamma\Big(p\Delta_{t-1,\mathsf{max}} + \big(P_{t}(1\mymid 1, a) - p \big)\widehat{V}_{t-1}\Big).
\end{align*}
Here, the second line comes from the Bellman equation $Q^{\star}(1, a) = 1 + \gamma pV^{\star}(1)$,
whereas the last line holds since $\widehat{V}_{t-1} - V^{\star}(1) = \max_a \big(\widehat{Q}_{t-1}(a) - V^{\star}(1)\big) = \max_a\Delta_{t-1}(a)$ (as a consequence of the relation~\eqref{eq:optimal-V-Q-hardMDP}). 
Applying the above relation recursively leads to
\begin{align}
\Delta_t(a) = \sum_{k = 1}^t\eta_k\prod_{i = k+1}^t \big(1-\eta_i\big)\gamma\Big(p\Delta_{k-1,\mathsf{max}} + \big(P_{k}(1\mymid 1, a) - p \big)\widehat{V}_{k-1}\Big),
\end{align}
where we have used the initialization $\Delta_0(a) = 0$. Letting 
\begin{subequations}
	\label{eqn:defn-xi-t}
\begin{align}
	\xi_t(a) &\defn \sum_{k = 1}^t\eta_k\prod_{i = k+1}^t \big(1-\eta_i\big)\gamma \big( P_{k}(1\mymid 1, a) - p \big)\widehat{V}_{k-1}, \\
	\xi_{t,\mathsf{max}} &\defn \max_a \xi_t(a),
\end{align}
\end{subequations}
%
one can express the above relation as follows 
\begin{align*}
\Delta_{t,\mathsf{max}} = \sum_{k = 1}^t\eta_k\prod_{i = k+1}^t \big(1-\eta_i\big)\gamma p \Delta_{k-1,\mathsf{max}} + \xi_{t,\mathsf{max}}.
\end{align*}

Next, we claim that $\mathbb{E} [ \xi_{t,\mathsf{max}} ]$ satisfies the following property 
\begin{align}
\label{eqn:lb-xi}
	\mathbb{E} [ \xi_{t,\mathsf{max}} ] &\ge \frac{c}{\sqrt{(1-\gamma)^2T}\log T}
	\qquad \text{for all}~t \ge \widehat{\tau}
\end{align}
for some universal constant $c>0$,  where
\begin{align}
\label{eq:tau-hat-byproduct-duplicate}
\widehat{\tau} \defn \max\Bigg\{\tau' \,\,\Big|\,\, \prod_{i = \tau'}^{T} \big(1-\eta_i(1-\gamma p)\big) \le \frac{6}{7}\Bigg\}	,
\end{align}
whose existence is ensured under the condition \eqref{eq:eta0_lb_case3}.
Given the validity of this claim (which we shall justify in Section~\ref{sec:proof-auxiliary-results-Case3}), we immediately arrive at
\begin{align}
\label{eqn:Delta-iterates-exp}
	\mathbb{E} [ \Delta_{t,\mathsf{max}} ] &\ge \sum_{k = 1}^t\eta_k\prod_{i = k+1}^t \big(1-\eta_i\big)\gamma p \mathbb{E} [ \Delta_{k-1,\mathsf{max}} ] + \frac{c}{\sqrt{(1-\gamma)^2T}\log T}\qquad \text{for all}~t \ge \widehat{\tau}.
\end{align}

In order to study the above recursion, it is helpful to look at the following sequence
\begin{align}
\label{eqn:new-sequence}
x_t = (1-\eta_t)x_{t-1} + \eta_t \bigg(\gamma p x_{t-1} + \frac{c}{\sqrt{(1-\gamma)^2T}\log T} \bigg)
\end{align}
with $x_{\widehat{\tau}} = 0$, where we recall the definition of $\widehat{\tau}$ in \eqref{eq:tau-hat-byproduct-duplicate}.  
In comparison to the iterative relation~\eqref{eqn:Delta-iterates-exp} which starts from $\mathbb{E} [ \Delta_{0,\mathsf{max}} ] = 0$ (and hence $\mathbb{E} [ \Delta_{t,\mathsf{max}} ] \geq 0$), we let the sequence $x_t$ start from $x_{\widehat{\tau}} = 0$, where $\widehat{\tau}$ is defined in \eqref{eq:tau-hat-byproduct-duplicate}.
It is straightforward to verify that 
\begin{equation}
	\mathbb{E} \big[ \Delta_{T,\mathsf{max}} \big] \ge x_T,
	\label{eq:E-Delta-Tmax-xT-order}
\end{equation}
recognizing that
\begin{align*}
x_t = \sum_{k = \widehat{\tau}}^t\eta_k\prod_{i = k+1}^t \big(1-\eta_i\big)\gamma p x_{k-1} + 
 \sum_{k = \widehat{\tau}}^T \eta_k\prod_{i = k+1}^T
\big(1- \eta_i)\frac{c}{\sqrt{(1-\gamma)^2T}\log T}. 
\end{align*}
A little algebra reveals that the sequence~\eqref{eqn:new-sequence} obeys
\begin{align*}
	x_T &= \frac{c}{\sqrt{(1-\gamma)^2T}\log T} \sum_{k = \widehat{\tau}}^T \eta_k\prod_{i = k+1}^T
\big(1-\eta_i(1-\gamma p) \big) 
= \frac{c}{\sqrt{(1-\gamma)^2T}\log T}\frac{1}{1-\gamma p}\left[1-\prod_{i = \widehat{\tau}}^T
\big(1-\eta_i(1-\gamma p) \big)\right] \\
&= \frac{3c}{4\sqrt{(1-\gamma)^4 T}\log T}\left[1-\prod_{i = \widehat{\tau}}^T
\big(1-\eta_i(1-\gamma p) \big)\right]
\ge \frac{3c}{28\sqrt{(1-\gamma)^4T}\log T},
\end{align*}
where the second equality arises from \eqref{eq:eta-sum-hat-tau},
and the last inequality holds as long as $\prod_{i = \widehat{\tau}}^T \big(1-\eta_i(1-\gamma p)\big) \le 6/7$ (see \eqref{eq:tau-hat-byproduct}). 
This taken together with \eqref{eq:E-Delta-Tmax-xT-order} leads to 
\begin{align*}
	\mathbb{E} \big[ \Delta_{T,\mathsf{max}} \big] \ge x_T \ge \frac{3c}{28\sqrt{(1-\gamma)^4T}\log T} .
\end{align*}

Combining the above bound with \eqref{eq:Qt} leads to
\begin{align*} 
	\mathbb{E} \Big[ V_T(1) - V^{\star}(1) \Big] &\ge \mathbb{E} \Big[  \Delta_{T,\mathsf{max}} - \frac{1}{1-\gamma}\prod_{i = 1}^T \big(1-\eta_i(1-\gamma)\big) \Big] \\
	&\ge \frac{3c}{28\sqrt{(1-\gamma)^4T}\log T} - \frac{1}{1-\gamma}\prod_{i = 1}^T \big(1-\eta_i(1-\gamma)\big).
\end{align*}
Taking this together with \eqref{eq:Vt3-expansion-case3}, we arrive at
\begin{align*} 
\max \Big\{	&\mathbb{E} \Big[ \big|V_T(3) - V^{\star}(3) \big| \Big],\;  \mathbb{E} \Big[ \big|V_T(1) - V^{\star}(1) \big| \Big] \Big\} \\
	&\ge \max\left\{\frac{1}{1-\gamma}\prod_{i = 1}^T \big(1-\eta_i(1-\gamma)\big), ~\frac{3c}{28\sqrt{(1-\gamma)^4T}\log T} - \frac{1}{1-\gamma}\prod_{i = 1}^T \big(1-\eta_i(1-\gamma)\big) \right\} \\
	&\ge \frac{1}{2} \cdot \frac{1}{1-\gamma}\prod_{i = 1}^T \big(1-\eta_i(1-\gamma)\big) + \frac{1}{2}  \left[\frac{3c}{28\sqrt{(1-\gamma)^4T}\log T} - \frac{1}{1-\gamma}\prod_{i = 1}^T \big(1-\eta_i(1-\gamma)\big)\right] \\
	&= \frac{3c}{56\sqrt{(1-\gamma)^4T}\log T}.
\end{align*}
This combined with \eqref{eq:bias-variance-decomposition}  establishes the following desired lower bound: 
\begin{align*} 
	\max_s \mathbb{E} \Big[ \big|V_T(s) - V^{\star}(s)\big|^2 \Big] 
	\geq \left( \frac{3c}{56\sqrt{(1-\gamma)^4T}\log T} \right)^2 = \frac{9c^2}{56^2 (1-\gamma)^4T \log^2 T}.
\end{align*}

\subsubsection{Proofs of auxiliary results}
\label{sec:proof-auxiliary-results-Case3}

\paragraph{Proof of the inequality \eqref{eq:Qt}.} 

We shall establish this claim by induction. 
To begin with, the inequality \eqref{eq:Qt} holds trivially for the base case with $t=0$.
Now, let us assume that the claim holds up to the $(t-1)$-th iteration, and we would like to justify it for the $t$-th iteration.
As an immediate consequence of the claim \eqref{eq:Qt} for the $(t-1)$-th iteration and the definitions of $V_{t-1}$ and  $\widehat{V}_{t-1}$, we have 
\begin{align*} 
	V_{t-1}(1) = \max_{a} Q_{t-1}(1, a) 
	&\ge \max_{a}  \widehat{Q}_{t-1}(a) - \frac{1}{1-\gamma}\prod_{i = 1}^{t-1} \big(1-\eta_i(1-\gamma)\big) \\
	&= \widehat{V}_{t-1} - \frac{1}{1-\gamma}\prod_{i = 1}^{t-1} \big(1-\eta_i(1-\gamma)\big) .
\end{align*}
By virtue of the respective update rules of $Q_t(1, a)$ and $\widehat{Q}_t(a)$, we can express their difference as follows:
\begin{align*}
Q_t(1, a) - \widehat{Q}_t(a) &= (1-\eta_t)\big(Q_{t-1}(1, a) - \widehat{Q}_{t-1}(a)\big) + \eta_t\gamma P_{t}(1\mymid 1, a)\big(V_{t-1}(1) - \widehat{V}_{t-1}\big) \\
&\ge -(1-\eta_t)\frac{1}{1-\gamma}\prod_{i = 1}^{t-1} \big(1-\eta_i(1-\gamma)\big) - \eta_t\gamma P_{t}(1\mymid 1, a) \frac{1}{1-\gamma}\prod_{i = 1}^{t-1} \big(1-\eta_i(1-\gamma)\big) \\
&\ge -(1-\eta_t)\frac{1}{1-\gamma}\prod_{i = 1}^{t-1} \big(1-\eta_i(1-\gamma)\big) - \eta_t\gamma  \frac{1}{1-\gamma}\prod_{i = 1}^{t-1} \big(1-\eta_i(1-\gamma)\big) \\
&= -\frac{1}{1-\gamma}\prod_{i = 1}^t \big(1-\eta_i(1-\gamma)\big), 
\end{align*}
where the first inequality invokes the induction hypothesis for the $(t-1)$-th iteration. 
This establishes \eqref{eq:Qt} for the $t$-th iteration, and hence the proof is complete via an induction argument.

\paragraph{Proof of the claim~\eqref{eqn:schubert-winterreise}.}
When taking the constant learning rates $\eta_t \equiv \eta \le \frac{1}{(1-\gamma)^2T} \le \frac{1}{50}$ (under the condition $T \ge \frac{50}{(1-\gamma)^2}$), one has
\begin{align*}
\max_{1 \le k \le t} \eta_k^{(t)} \le \eta_t =\eta \le \frac{1}{50},
\end{align*} 
thus allowing us to take $\tau=1$ for this case.

It then suffices to look at rescaled linear learning rates (i.e., $\eta_t = \frac{1}{1+\ceta(1-\gamma)t}$). 
As already calculated in the expression~\eqref{eq:ratio-etakT-eta-kminus1}, the ratio of two consecutive quantities obeys 
\begin{align}
\label{eqn:ratio-schubert}
\frac{\eta_{k-1}^{(t)}}{\eta_k^{(t)}} 
= \frac{1-\frac{4}{3}(1-\gamma)\eta_k}{1-c_\eta(1-\gamma)\eta_k}.
\end{align}
%
In what follows, we divide into two cases, depending on whether this sequence is decreasing or increasing.

\begin{itemize}

\item {\em The case with $4/3 \le \ceta < \log T$.} 
In this scenario, the ratio~in \eqref{eqn:ratio-schubert} is larger than 1, and hence the sequence $\{\eta_{k}^{(t)}\}$ decreases with $k$. 
Let us define 
\begin{align}
	\tau \defn \min\Bigg\{\tau' \,\,\Big|\,\, \prod_{i = 1}^{\tau'} \big(1-\eta_i(1-\gamma p)\big) \le \frac{1}{50}\Bigg\},	
	\label{eq:defn-tau-case31}
\end{align}
which clearly satisfies $\tau \leq T$ (in view of \eqref{eq:eta0_lb_case3}). 
For all $t \ge \tau$, one has
\begin{align*}
\max_{1 \le k \le t} \eta_k^{(t)} \le \prod_{i = 1}^{\tau} \big(1-\eta_i(1-\gamma p)\big) \le \frac{1}{50}.
\end{align*} 
At the same time, we claim that one must have 
\begin{align}
	\prod_{i = \tau}^{T} \big(1-\eta_i(1-\gamma p)\big) \le \frac{2}{3}.
	\label{eq:prod-etai-tau-T-23}
\end{align}
Otherwise, recalling $\eta_0^{(T)} = \prod_{i = 1}^{T} \big(1-\eta_i(1-\gamma p)\big)$, we have 
$$
\eta_0^{(T)} = \left\{ \prod_{i = 1}^{\tau-1} \big(1-\eta_i(1-\gamma p)\big) \right\}
\left\{ \prod_{i = \tau}^{T} \big(1-\eta_i(1-\gamma p)\big)   \right\} > \frac{1}{50} \cdot \frac{2}{3}=\frac{1}{75}, 
$$ 
which contradicts our assumption that $\eta_0^{(T)} > 1/75 $ (cf.~\eqref{eq:eta0_lb_case3}).

\item {\em The case with $1-\gamma < \ceta < 4/3$.}
In this case, the sequence $\eta_{k}^{(t)}$ increases with $k$. 
If we set  
\begin{align}
	\tau \defn \Big\lceil\frac{49}{\ceta(1-\gamma)} \Big\rceil
	< \frac{50}{(1-\gamma)^2}< T,
	\label{eq:defn-tau-case32}
\end{align}
then for all $t \ge \tau$ we have 
\begin{align*}
\max_{1 \le k \le t} \eta_k^{(t)} = \eta_t^{(t)} = \eta_t \le \eta_{\tau} 
\leq 
\frac{1}{1+\ceta(1-\gamma)\frac{49}{\ceta(1-\gamma)}}
= \frac{1}{50}. 
\end{align*} 
Under the condition $T \ge \frac{150}{(1-\gamma)^2}\geq \frac{150}{\ceta(1-\gamma)}$ (so that $T-\tau +1 \geq \frac{100}{\ceta(1-\gamma)} \geq \frac{100}{(1-\gamma)4/3}$), 
one can show that
\begin{align}
	\prod_{i = \tau}^{T} \big(1-\eta_i(1-\gamma p)\big) \le \Big( 1-\frac{1-\gamma}{100} \Big)^{T-\tau+1}
\leq \Big( 1-\frac{1-\gamma}{100} \Big)^{\frac{100}{(1-\gamma)4/3}} \le \frac{3}{4}.
	\label{eq:prod-etai-tau-T-34}
\end{align}

\item Putting these two cases together (with $\tau$ specified in \eqref{eq:defn-tau-case31} and \eqref{eq:defn-tau-case32}), we obtain  
\begin{align} 
\max_{1 \le k \le t} \eta_k^{(t)} \le \frac{1}{50}
\end{align}
for all $t \ge \tau$, thus establishing the desired inequality~\eqref{eqn:schubert-winterreise}.

\end{itemize}

\paragraph{Proof of the inequality~\eqref{eqn:lb-xi}.} 
For every $t$, recalling the definition~\eqref{eqn:defn-xi-t}, it is convenient to write 
\begin{align*}
	\mathbb{E} [\xi_{t,\mathsf{max}}] &= \mathbb{E} \left[ \frac{\xi_t(1)+\xi_t(2)+|\xi_t(1)-\xi_t(2)|}{2} \right] 
	= \mathbb{E} \left[ \frac{|\xi_t(1)-\xi_t(2)|}{2}  \right] \\
	&= \frac{1}{2}\mathbb{E} \left[ \, \left|\sum_{k = 1}^t\eta_k\prod_{i = k+1}^t \big(1-\eta_i\big)\gamma \big( P_{k}(1\mymid 1, 1) -  P_{k}(1\mymid 1, 2) \big)\widehat{V}_{k-1}) \,\right| \right] ,
\end{align*}
where we have used the fact that $\mathbb{E}[\xi_t(a)] = 0$. 
To control the right-hand side of the above equation, let us define  
\begin{align*}
	\zeta_t &\defn \sum_{k = 1}^t z_k ,
	\qquad z_k \defn \eta_k\prod_{i = k+1}^t \big(1-\eta_i\big)\gamma \big(P_{k}(1\mymid 1, 1) - P_{k}(1\mymid 1, 2) \big)\widehat{V}_{k-1}
\end{align*}
for any $k\geq 1$, 
where $\{z_k\}$ also forms a martingale sequence since
\begin{align*}
	\mathbb{E}\Big[ z_k\mymid \big\{ P_{j}(1\mymid 1, 1), P_{j}(1\mymid 1, 2) \big\}_{1\leq j< k} \Big] = 0.
\end{align*}

As a consequence of Freedman's inequality, we claim that $\zeta_t $ satisfies
\begin{align}
\label{eqn:tail-bunny}
\mathbb{P}\left\{\big|\zeta_t\big| \ge \sqrt{\frac{8\log\frac{2}{\delta}}{3(1-\gamma)}\sum_{k = 1}^t\eta_k^2\Big[\prod_{i = k+1}^t \big(1-\eta_i\big)\Big]^2} + \frac{4\eta_t \log\frac{2}{\delta}}{3(1-\gamma)} \right\} \le \delta. 
\end{align}
To verify this relation, we first notice that 
\begin{align}
\label{eqn:unif-bound}
|z_k| \le \max_{1 \le k \le t} \eta_k\prod_{i = k+1}^t \big(1-\eta_i\big) \cdot \frac{1}{1-\gamma} \le \frac{\eta_t}{1-\gamma} ,
\end{align}
provided that $\max_{k} \eta_k\prod_{i = k+1}^t \big(1-\eta_i\big) \le \eta_t$.
To verify the condition $\max_{k} \eta_k\prod_{i = k+1}^t \big(1-\eta_i\big) \le \eta_t$,  one can check---similar to \eqref{eq:ratio-etakT-eta-kminus1}---that
\begin{equation}\label{eq:ratio-check_eta_000}
\frac{\eta_{k-1}\prod_{i = k}^t \big(1-\eta_i\big)}{\eta_k\prod_{i = k+1}^t \big(1-\eta_i\big)} = 1 - \frac{\big(1 - \ceta(1-\gamma)\big)\eta_k}{1 - \ceta(1-\gamma)\eta_k}
	\leq 1,
\end{equation}
which indicates that $\eta_k\prod_{i = k+1}^t \big(1-\eta_i\big)$ is an increasing sequence as long as $\ceta \le \log T \le \frac{1}{1-\gamma}$ (see \eqref{eq:logT_ub}).  
In addition to the boundedness condition \eqref{eqn:unif-bound}, we can further calculate 
\begin{align*}
	& \sum_{k = 1}^t \mathsf{Var}\big(z_k\mymid P_{k-1}(1\mymid 1, 1), P_{k-1}(1\mymid 1, 2), \ldots, P_{1}(1\mymid 1, 1), P_{1}(1\mymid 1, 2)\big) \\
	&\qquad = \sum_{k = 1}^t\eta_k^2\left[\prod_{i = k+1}^t \big(1-\eta_i\big)\right]^2\cdot2p(1-p)\cdot\big(\gamma\widehat{V}_{k-1}\big)^2   \le \sum_{k = 1}^t\eta_k^2\left[\prod_{i = k+1}^t \big(1-\eta_i\big)\right]^2\cdot \frac{2}{3(1-\gamma)}, 
\end{align*}
where the last inequality comes from the facts that $\widehat{V}_{k-1} \le \frac{1}{1-\gamma}$ and the choice $p = \frac{4\gamma - 1}{3\gamma}$.
These bounds taken together with Freedman's inequality (see \eqref{eq:Freedman-special}) validate \eqref{eqn:tail-bunny}.

By virtue of \eqref{eqn:tail-bunny}, setting $\delta = \frac{(1-\gamma)^2}{2}\mathbb{E} \big[ |\zeta_t|^2 \big]$ yields that with probability at least $1-\delta$, 
\begin{align} \label{eq:B-delta}
|\zeta_t| \le B	\defn 
\sqrt{\frac{8\log\frac{2}{\delta}}{3(1-\gamma)} \sum_{k = 1}^t\eta_k^2\Big[\prod_{i = k+1}^t \big(1-\eta_i\big)\Big]^2} + \frac{4\eta_t \log\frac{2}{\delta}}{3(1-\gamma)} ~~
\text{ with } \delta = \frac{(1-\gamma)^2}{2}\mathbb{E} \big[ |\zeta_t|^2 \big]. 
\end{align}
When $T \ge \frac{1}{(1-\gamma)^2}$, one can ensure that 
\begin{align}
	\mathbb{E} [\xi_{t,\mathsf{max}}]  
	&= \frac{1}{2}\mathbb{E} \big[ |\zeta_t| \big]
	\ge \frac{1}{2}\mathbb{E} \big[ |\zeta_t|\mathds{1}\big(|\zeta_t| \le B\big) \big]  
	\ge \frac{1}{2B}\mathbb{E} \big[ |\zeta_t|^2\mathds{1}\big(|\zeta_t| \le B\big) \big] \nonumber \\
	&= \frac{1}{2B}\Big\{ \mathbb{E} [ |\zeta_t|^2 ] - \mathbb{E} \big[ |\zeta_t|^2\mathds{1}\big(|\zeta_t| > B\big) \big] \Big\} \nonumber \\
	&\overset{(\mathrm{i})}{\geq} \frac{1}{2B}\Big\{ \mathbb{E} [ |\zeta_t|^2 ] - \frac{1}{(1-\gamma)^2} \mathbb{P}\big\{|\zeta_t| > B\big\} \Big\} \nonumber \\
	&\ge \frac{1}{2B}\Big\{ \mathbb{E} [ |\zeta_t|^2 ] - \frac{\delta}{(1-\gamma)^2} \Big\} 
	\overset{(\mathrm{ii})}{\geq} \frac{1}{4B}\mathbb{E} [ |\zeta_t|^2 ]. \label{eq:emax_intermediate_lb}
\end{align}
Here, (i) holds since 
\begin{align*}
\big|\zeta_t\big| \le  \sum_{k=1}^t |z_k| \leq  \left[ \sum_{k=1}^t \eta_k\prod_{i = k+1}^t \big(1-\eta_i\big) \right] \cdot \frac{1}{1-\gamma} \leq \frac{1}{1-\gamma}
\end{align*} 
as a consequence of \eqref{eqn:unif-bound} and \eqref{eq:sum-eta-i-t-infinite}; (ii) holds by the choice of $\delta$.
It is thus sufficient to lower bound $\mathbb{E} [|\zeta_t|^2]$.
Towards this,  let us  define
\begin{align}
\label{eq:tau-hat-byproduct}
\widehat{\tau} \defn \max\Bigg\{\tau' \,\,\Big|\,\, \prod_{i = \tau'}^{T} \big(1-\eta_i(1-\gamma p)\big) \le \frac{6}{7}\Bigg\}	,
\end{align}
which clearly satisfies $\tau\leq \widehat{\tau}\leq T$ (in view of \eqref{eq:prod-etai-tau-T-23} and \eqref{eq:prod-etai-tau-T-34}).
Then, for all $t \ge \widehat{\tau}$ one has 
(which shall be proved towards the end of this subsection)
\begin{align}
	\sum_{k = \tau}^{t} \eta_k^2 \Big[\prod_{i = k+1}^t \big(1-\eta_i\big)\Big] ^2
	\ge \frac{1}{8}\sum_{k = 1}^{t} \eta_k^2 \Big[\prod_{i = k+1}^t \big(1-\eta_i\big)\Big]^2.
	\label{eq:sum-LB-1/8-4567}
\end{align}

We now proceed to lower bound $\mathbb{E} [|\zeta_t|^2]$ for $t \ge \widehat{\tau}$. We first observe that for any $t \ge \widehat{\tau}$, 
\begin{align*}
	\mathbb{E} \big[ |\zeta_t|^2 \big] 
	&\ge \sum_{k = 1}^t \mathbb{E} \left[ \mathsf{Var}\Big(\eta_k\prod_{i = k+1}^t \big(1-\eta_i\big)\gamma\big(P_{k}(1\mymid 1, 1) - P_{k}(1\mymid 1, 2) \big)\widehat{V}_{k-1} \mymid \widehat{V}_{k-1}\Big) \right] \\
	&\ge  \frac{1}{2}\sum_{k = \tau}^t \mathbb{E} \left[ \mathsf{Var}\Big(\eta_k\prod_{i = k+1}^t \big(1-\eta_i\big)\gamma \big(P_{k}(1\mymid 1, 1) - P_{k}(1\mymid 1, 2) \big)\widehat{V}_{k-1} \mymid \widehat{V}_{k-1} \ge \frac{1}{4(1-\gamma)} \Big) \right],
\end{align*}
where the first line relies on \eqref{eq:variance-expansion-martingale}, 
and the last step makes use of the fact \eqref{eqn:Tailbound-vhat-rolland}. 
To further control the right-hand side of the above inequality, we take $\tau' \defn \max\left\{t - \frac{1}{\eta_{t/2}}, 1\right\} $  and show that 
\begin{align}
	\mathbb{E} \big[ |\zeta_t|^2 \big]  
	&\overset{\text{(i)}}{\ge} \frac{1}{2}\sum_{k = \tau}^t\eta_k^2\Big[\prod_{i = k+1}^t \big(1-\eta_i\big)\Big]^2 \gamma^2\cdot 2p(1-p) \frac{1}{16(1-\gamma)^2}
	\nonumber\\
	&\ge \frac{1}{48(1-\gamma)}\sum_{k = \tau}^t\eta_k^2\Big[\prod_{i = k+1}^t \big(1-\eta_i\big)\Big]^2 \overset{\text{(ii)}}{\ge} \frac{1}{400(1-\gamma)}\sum_{k = 1}^t\eta_k^2\Big[\prod_{i = k+1}^t \big(1-\eta_i\big)\Big]^2 \nonumber\\
&\ge \frac{1}{400(1-\gamma)}
	\sum_{k = \tau'}^t\eta_k^2\Big[\prod_{i = k+1}^t \big(1-\eta_i\big)\Big]^2 \overset{\text{(iii)}}{\ge} \frac{\eta_t}{9600(1-\gamma)}. \label{eq:square-zeta}
\end{align}
Here, (i) makes use of the constraint $\widehat{V}_{k-1} \ge \frac{1}{4(1-\gamma)}$, 
while (ii) makes use of \eqref{eq:sum-LB-1/8-4567}, and (iii) are valid if the following property holds (which shall be proved towards the end of this subsection)
%
%
\begin{align}
	\sum_{k = \tau'}^t\eta_k^2\Big[\prod_{i = k+1}^t \big(1-\eta_i\big)\Big]^2
	\ge \frac{1}{24} \eta_t.
	\label{eq:sum-LB-1/8-890}
\end{align}

We are now well-equipped to control $\mathbb{E} [\xi_{t,\mathsf{max}}] $
using the property \eqref{eq:emax_intermediate_lb}. Recall the expression of $B$ in \eqref{eq:B-delta}, we know that bounding $\mathbb{E} [ |\zeta_t|^2 ] / B$ boils down to  controlling
\begin{align}
\frac{\mathbb{E} [ |\zeta_t|^2 ]}{\frac{\eta_t}{1-\gamma}\log\frac{2}{\delta}} \qquad \mbox{and}\qquad \frac{\mathbb{E} [ |\zeta_t|^2 ]}{\sqrt{\frac{\log\frac{2}{\delta}}{1-\gamma}\sum_{k = 1}^t\eta_k^2\Big[\prod_{i = k+1}^t \big(1-\eta_i\big)\Big]^2}} . 
	\label{eq:two-terms-LB-123}
\end{align}
\begin{itemize}
\item
For the first term in \eqref{eq:two-terms-LB-123}, recalling that $\delta = \frac{(1-\gamma)^2}{2}\mathbb{E} \big[ |\zeta_t|^2 \big]$, we can demonstrate that
\begin{align}
\log \frac{1}{\delta} = -\log \frac{(1-\gamma)^2}{2}\mathbb{E} \big[ |\zeta_t|^2 \big] \le -\log \frac{(1-\gamma)\eta_t}{19200} \le \log \frac{19200(1+(1-\gamma)T\log T)}{1-\gamma} \lesssim \log T ,
	\label{eq:log-delta-T-relation-123}
\end{align}
where the first inequality makes use of the bound \eqref{eq:square-zeta}, and the second inequality arises from the fact $\eta_t \geq \frac{1}{1+(1-\gamma)T\log T}$ (given the range of the learning rates in this case). 
Combining this with \eqref{eq:square-zeta}, 
we can guarantee that 
\[
\frac{\mathbb{E} [ |\zeta_t|^2 ]}{\frac{\eta_t}{1-\gamma}\log\frac{2}{\delta}} \gtrsim \frac{1}{\log T}. 
\]

\item
Moving to the second term in \eqref{eq:two-terms-LB-123}, one can ensure that
\begin{align*}
\frac{\mathbb{E} \big[ |\zeta_t|^2 \big]}{\sqrt{\frac{\log\frac{2}{\delta}}{1-\gamma}\sum_{k = 1}^t\eta_k^2\Big[\prod_{i = k+1}^t \big(1-\eta_i\big)\Big]^2}} & \overset{\mathrm{(i)}}{\gtrsim} \sqrt{\frac{1}{(1-\gamma)\log T}\sum_{k = 1}^t\eta_k^2\Big[\prod_{i = k+1}^t \big(1-\eta_i\big)\Big]^2} \\
& \overset{\mathrm{(ii)}}{\gtrsim}  \sqrt{\frac{\eta_t}{(1-\gamma)\log T}} \overset{\mathrm{(iii)}}{\gtrsim} \frac{1}{\sqrt{(1-\gamma)^2T}\log T}.
\end{align*}
Here, (i) follows from \eqref{eq:square-zeta} and \eqref{eq:log-delta-T-relation-123} since
\[
	\mathbb{E} \big[ |\zeta_t|^2 \big] \gtrsim \frac{1}{1-\gamma}\sum_{k = 1}^t\eta_k^2\Big[\prod_{i = k+1}^t \big(1-\eta_i\big)\Big]^2\qquad \text{and} \qquad \log\frac{2}{\delta} \lesssim \log T;
\]
(ii) arises from \eqref{eq:sum-LB-1/8-890}; and (iii) relies on the fact $\eta_t \gtrsim \frac{1}{(1-\gamma)T\log T}$ (given the range of the learning rates in this case).
\end{itemize}

Substituting the above relations into  \eqref{eq:emax_intermediate_lb} and using the expression of $B$ in \eqref{eq:B-delta}, we reach at
\begin{align*}
	\mathbb{E} [ \xi_{t, \mathsf{max}} ] \ge \frac{1}{4B}\mathbb{E} \big[ |\zeta_t|^2 \big]
	&\ge \frac{c}{\sqrt{(1-\gamma)^2T}\log T},
\end{align*}
for some constant $c > 0$. Thus, this validates the inequality~\eqref{eqn:lb-xi}.

\begin{proof}[Proof of the claim \eqref{eq:sum-LB-1/8-4567}] By the definition of $\widehat{\tau}$ in \eqref{eq:tau-hat-byproduct}, we have 
$\prod_{i = \widehat{\tau}}^{T} \big(1-\eta_i(1-\gamma p)\big) \le 6/7$. An important observation is that
\begin{align}\label{eqn:death-maiden}
\sum_{k = \tau}^{t} \eta_k\left[\prod_{i = k+1}^t \big(1-\eta_i\big)\right] 
	\stackrel{\mathrm{(i)}}{=} 1-\prod_{i = \tau}^t \big(1-\eta_i\big) 
	\stackrel{\mathrm{(ii)}}{\ge} \frac{1}{8}
	\stackrel{\mathrm{(iii)}}{\geq} \frac{1}{8}\sum_{k = 1}^{t} \eta_k\left[\prod_{i = k+1}^t \big(1-\eta_i\big)\right]. 
\end{align}
Here, the relations (i) and (iii) arise from \eqref{eq:sum-eta-i-t-infinite-tau}, and the inequality~(ii) follows since 
\begin{align}
\prod_{i = \tau}^t \big(1-\eta_i\big) \le \prod_{i = \tau}^{\widehat{\tau}} \big(1-\eta_i\big) \le \prod_{i = \tau}^{\widehat{\tau}} \big(1-\eta_i(1-\gamma p)\big) = \frac{\prod_{i = \tau}^{T} \big(1-\eta_i(1-\gamma p)\big)}{\prod_{i = \widehat{\tau}+1}^{T} \big(1-\eta_i(1-\gamma p)\big)} 
\leq \frac{\, 3/4\,}{6/7} \le \frac{7}{8},
\label{eq:prod-etai-tau-T-2345}
\end{align}
where $\tau$ is defined in \eqref{eq:defn-tau-case31} and \eqref{eq:defn-tau-case32} for linearly rescaled learning rates and 
 $\tau = 1$ for constant learning rates, and 
 we have also made use of \eqref{eq:eta0_lb_case3}, \eqref{eq:prod-etai-tau-T-23} and \eqref{eq:prod-etai-tau-T-34} in the penultimate inequality in \eqref{eq:prod-etai-tau-T-2345}. 

With \eqref{eqn:death-maiden} in place, we can continue to prove the claim \eqref{eq:sum-LB-1/8-4567}. Recognizing that $\eta_k\big[\prod_{i = k+1}^t \big(1-\eta_i\big)\big]$ is increasing in $k$ (see \eqref{eq:ratio-check_eta_000}),
we can obtain
\begin{align}
	\sum_{k = 1}^{\tau-1} \eta_k^2 \Big[\prod_{i = k+1}^t \big(1-\eta_i\big)\Big]^2 
	&\le \max_{1\leq k < \tau} \left\{ \eta_{k} \Big[\prod_{i = k +1}^t \big(1-\eta_i\big)\Big] \right\} \sum_{k = 1}^{\tau-1} \eta_k \Big[  \prod_{i = k+1}^t \big(1-\eta_i\big) \Big] \notag\\
	&\le \eta_{\tau} \Big[\prod_{i = \tau+1}^t \big(1-\eta_i\big)\Big] \sum_{k = 1}^{\tau-1} \eta_k \Big[  \prod_{i = k+1}^t \big(1-\eta_i\big) \Big] \notag\\
	& \le 7\eta_{\tau} \Big[\prod_{i = \tau+1}^t \big(1-\eta_i\big)\Big] \sum_{k = \tau}^{t} \eta_k\Big[\prod_{i = k+1}^t \big(1-\eta_i\big)\Big],
\end{align}
where the last inequality comes from \eqref{eqn:death-maiden}.
With the preceding inequality in place, the claim \eqref{eq:sum-LB-1/8-4567} then follows by observing that 
\begin{align*}
	\sum_{k = \tau}^{t} \eta_k^2 \Big[\prod_{i = k+1}^t \big(1-\eta_i\big)\Big] ^2
	&\ge \min_{\tau\leq k \leq t} \left\{ \eta_{k} \Big[\prod_{i = k+1}^t \big(1-\eta_i\big)\Big] \right\} \sum_{k = \tau}^{t} \eta_k\Big[\prod_{i = k+1}^t \big(1-\eta_i\big)\Big] \\
	&\ge \eta_{\tau} \Big[\prod_{i = \tau+1}^t \big(1-\eta_i\big)\Big] \sum_{k = \tau}^{t} \eta_k\Big[\prod_{i = k+1}^t \big(1-\eta_i\big)\Big], 
\end{align*}
where we make use of the monotonicity of $\eta_k\big[\prod_{i = k+1}^t \big(1-\eta_i\big)\big]$ again.
\end{proof}

\begin{proof}[Proof of the claim \eqref{eq:sum-LB-1/8-890}]
Note that for $\tau' \defn \max\left\{t - \frac{1}{\eta_{t/2}}, 1\right\}$, one has 
\[
\eta_k\Big[\prod_{i = k+1}^t \big(1-\eta_i\big)\Big] \ge \eta_t\big(1-\eta_{t/2}\big)^{t-\tau'} \ge
\eta_t\big(1-\eta_{t/2}\big)^{1/\eta_{t/2}} 
\ge \frac{1}{3}\eta_t,\qquad\text{for all}~\tau'\le k \le t,
\]
as long as the following condition holds (recalling the definition of $\widehat{\tau}$ in \eqref{eq:tau-hat-byproduct})
\begin{equation}
	\label{eq:eta-t2-t-rank-1/10}
	\eta_{t/2} \le 2\eta_t \le 2\eta_{\widehat{\tau}} \le 1/10.
\end{equation}
In addition, 
similar to \eqref{eq:eta-sum-hat-tau}, we can derive
\begin{align*}
	\sum_{k = \tau'}^t\eta_k\prod_{i = k+1}^t \big(1-\eta_i\big) &= 1 - \prod_{i = \tau'}^t \big(1-\eta_i\big) 
	\ge 1-\max\left\{\big(1-\eta_t\big)^{1/\eta_{t/2}+1}, \, \prod_{i = 1}^t \big(1-\eta_i\big)\right\} \\
	& \ge 1-\max\left\{e^{-1/2}, \, \prod_{i = 1}^{\widehat{\tau}} \big(1-\eta_i\big)\right\} 
	\ge \frac{1}{8},
\end{align*}
where we once again use the condition \eqref{eq:eta-t2-t-rank-1/10}, and the last inequality comes from the derivation in \eqref{eq:prod-etai-tau-T-2345}. 
Putting these two bounds together yields
\begin{align*}
\sum_{k=\tau'}^{t}\eta_{k}^{2}\Big[\prod_{i=k+1}^{t}\big(1-\eta_{i}\big)\Big]^{2} & \geq\min_{k:\,\tau'\leq k\leq t}\left\{ \eta_{k}\Big[\prod_{i=k+1}^{t}\big(1-\eta_{i}\big)\Big]\right\} \sum_{k=\tau'}^{t}\eta_{k}\Big[\prod_{i=k+1}^{t}\big(1-\eta_{i}\big)\Big]\\
 & \geq\frac{1}{3}\eta_{t}\cdot\frac{1}{8}\geq\frac{1}{24}\eta_{t}.
\end{align*}

To finish up, it remains to justify \eqref{eq:eta-t2-t-rank-1/10}. 
This condition is obvious for constant learning rates. As for rescaled learning rates, one can see that 
\[
	\eta_i = \frac{1}{1+(1-\gamma)\ceta i} \ge \frac{19}{20\ceta(1-\gamma)i}  \qquad \text{for all }i \ge \overline{\tau},
\]
where $\overline{\tau} \defn \lceil \frac{19}{\ceta(1-\gamma)} \rceil$. This allows one to obtain 
\[
\log\Big[\prod_{i = \overline{\tau}}^{T} \big(1-\eta_i(1-\gamma p)\big)\Big] \le -\sum_{i = \overline{\tau}}^{T} \eta_i(1-\gamma p) \le -\sum_{i = \overline{\tau}}^{T}\frac{19}{15\ceta i} \le -\frac{19\log\frac{\,T\,}{\overline{\tau}}}{15\ceta} \le -\frac{19\log\frac{\ceta(1-\gamma)T}{20}}{15\ceta} \le -\frac{1}{5},
\]
provided that $T\geq \frac{c_1}{(1-\gamma)^2}$ for some sufficiently large constant $c_1>0$ and $1-\gamma < \ceta < \log T$. 
Taking this together with \eqref{eq:tau-hat-byproduct} implies that $\widehat{\tau} \ge \overline{\tau}$ and hence 
	$\eta_{\widehat{\tau}} \le \eta_{\overline{\tau}} = \frac{1}{1+ (1-\gamma)\ceta \overline{\tau}\,} = 1/20$. 
\end{proof}

%% file: AsynQ_analysis.tex
\section{Analysis for asynchronous Q-learning (Theorem~\ref{thm:infinite-horizon-asyn})}
\label{sec:AsynQ_analysis}

\subsection{Notation and preliminary facts}
\label{sec:notation-AsynQ-analysis}

\paragraph{Vector and matrix notation.}

We shall adopt the vector notation $\bm{Q}_t \in \mathbb{R}^{|\cS||\cA| }$, $\bm{V}_t\in \mathbb{R}^{|\cS|}$, $\bm{r}\in \mathbb{R}^{|\cS||\cA| }$ in the same way as in Section~\ref{subsec:matrix-notation}. 
The sample transition matrix $\bm{P}_t \in \mathbb{R}^{|\cS||\cA| \times |\cS|}$ in the asynchronous case is defined such that
\begin{equation}
	\bm{P}_t\big( (s,a), s^{\prime} \big) = \begin{cases} 1, \quad &\text{if } (s,a,s^{\prime}) = (s_{t-1},a_{t-1}, s_t); \\ 
						0, \quad & \text{else}.\end{cases}
	\label{eq:defn-Pt-async-Q}
\end{equation}
It is also handy to introduce the diagonal matrix 
$\bm{\Lambda}_t \in \mathbb{R}^{|\cS||\cA| \times |\cS||\cA|}$ such that
\begin{equation}
	\label{eq:defn-Lambda-t-async}
	\bm{\Lambda}_t\big( (s,a), (s,a) \big) = \begin{cases} \eta, \qquad & \text{if }(s, a) = (s_{t-1}, a_{t-1}); \\
							0, & \text{otherwise.}
	\end{cases}
\end{equation}
%
%
Armed with the above notation, the asynchronous Q-learning update rule \eqref{eq:async-Q-update-rule} can be conveniently expressed as follows:
\begin{align}
	\bm{Q}_{t} =\big(\bm{I}-\bm{\Lambda}_{t}\big)\bm{Q}_{t-1}+\bm{\Lambda}_{t}\big(\bm{r}+\gamma\bm{P}_{t}\bm{V}_{t-1}\big). 
	\label{eq:asynQ}
\end{align}
%


\paragraph{Range of $V_t$ and $Q_t$.}

Similar to the synchronous counterpart, we have the following elementary properties.
\begin{lemma}
Suppose that $0<\eta_t\leq 1$ for all $t\geq 0$. Assume that $\bm{0}\leq \bm{Q}_0\leq \frac{1}{1-\gamma} \bm{1}$. 
Then for any $t \ge 0$, one has
\begin{align}
	\bm{0} \le \bm{Q}_t \le \frac{1}{1-\gamma} \bm{1} \qquad \text{and} \qquad \bm{0}\leq \bm{V}_t \le \frac{1}{1-\gamma} \bm{1}.
\end{align}
\end{lemma}
\begin{proof} The proof is the same as that of Lemma~\ref{lemma:non-negativity-Qt-Vt}, and is hence omitted for brevity. \end{proof}

\subsection{Main steps for proving Theorem~\ref{thm:infinite-horizon-asyn}}

We are now in a position to outline the main steps for the proof of Theorem~\ref{thm:infinite-horizon-asyn}.

\paragraph{Step 1: deriving basic recursions.}

According to the update rule~\eqref{eq:asynQ}, we can derive the following elementary decomposition
\begin{align}
\bm{\Delta}_{t} \defn \bm{Q}_{t}-\bm{Q}^{\star} & =\big(\bm{I}-\bm{\Lambda}_{t}\big)\bm{Q}_{t-1}+\bm{\Lambda}_{t}\big(\bm{r}+\gamma\bm{P}_{t}\bm{V}_{t-1}\big)-\bm{Q}^{\star}\nonumber \\
 & =\big(\bm{I}-\bm{\Lambda}_{t}\big)\big(\bm{Q}_{t-1}-\bm{Q}^{\star}\big)+\bm{\Lambda}_{t}\big(\bm{r}+\gamma\bm{P}_{t}\bm{V}_{t-1}-\bm{Q}^{\star}\big)\nonumber \\
 & =\big(\bm{I}-\bm{\Lambda}_{t}\big)\big(\bm{Q}_{t-1}-\bm{Q}^{\star}\big)+\gamma\bm{\Lambda}_{t}\big(\bm{P}_{t}\bm{V}_{t-1}-\bm{P}\bm{V}^{\star}\big)\nonumber \\
 & =\big(\bm{I}-\bm{\Lambda}_{t}\big)\bm{\Delta}_{t-1}+\gamma\bm{\Lambda}_{t}\big(\bm{P}_{t}-\bm{P}\big)\bm{V}_{t-1}+\gamma\bm{\Lambda}_{t}\bm{P}\big(\bm{V}_{t-1}-\bm{V}^{\star}\big), \label{eq:iteration-asynQ}
\end{align}
where the penultimate identity follows from the Bellman optimality equation $\bm{Q}^{\star} = \bm{r} + \bm{P}\bm{V}^{\star}$.
Combining \eqref{eq:iteration-asynQ} with the inequalities \eqref{eq:V2Q-infinite} and using the definition \eqref{eq:defn-pit-s} of $\pi_t$ result in
\begin{subequations}
\begin{align}
\bm{\Delta}_{t} &\le \big(\bm{I}-\bm{\Lambda}_{t}\big)\bm{\Delta}_{t-1}+\gamma\bm{\Lambda}_{t}\big(\bm{P}_{t}-\bm{P}\big)\bm{V}_{t-1}+\gamma\bm{\Lambda}_{t}\bm{P}^{\pi_{t-1}}\bm{\Delta}_{t-1}, \\
\bm{\Delta}_{t} &\ge \big(\bm{I}-\bm{\Lambda}_{t}\big)\bm{\Delta}_{t-1}+\gamma\bm{\Lambda}_{t}\big(\bm{P}_{t}-\bm{P}\big)\bm{V}_{t-1}+\gamma\bm{\Lambda}_{t}\bm{P}^{\pi^{\star}}\bm{\Delta}_{t-1}.
\end{align}
\end{subequations}
Apply the above two relations recursively to obtain
\begin{subequations}
\label{eq:Delta-allstep-UB-LB-asynQ}
\begin{align}
\bm{\Delta}_{t} &\le \gamma\sum_{i=1}^{t}\prod_{j=i+1}^{t}\big(\bm{I}-\bm{\Lambda}_{j}\big)\bm{\Lambda}_{i}\big(\bm{P}_{i}-\bm{P}\big)\bm{V}_{i-1}
	+ \gamma\sum_{i=1}^{t}\prod_{j=i+1}^{t}\big(\bm{I}-\bm{\Lambda}_{j}\big)\bm{\Lambda}_{i}\bm{P}^{\pi_{i-1}}\bm{\Delta}_{i-1}+
	\prod_{j=1}^{t}\big(\bm{I}-\bm{\Lambda}_{j}\big)\bm{\Delta}_{0}, \label{eq:Delta-allstep-UB-asynQ}\\
\bm{\Delta}_{t} &\ge \gamma\sum_{i=1}^{t}\prod_{j=i+1}^{t}\big(\bm{I}-\bm{\Lambda}_{j}\big)\bm{\Lambda}_{i}\big(\bm{P}_{i}-\bm{P}\big)\bm{V}_{i-1}
	+ \gamma\sum_{i=1}^{t}\prod_{j=i+1}^{t}\big(\bm{I}-\bm{\Lambda}_{j}\big)\bm{\Lambda}_{i}\bm{P}^{\pi^{\star}}\bm{\Delta}_{i-1}+
	\prod_{j=1}^{t}\big(\bm{I}-\bm{\Lambda}_{j}\big)\bm{\Delta}_{0}.
	\label{eq:Delta-allstep-LB-asynQ}
\end{align}
\end{subequations}

By defining the following diagonal matrices
\begin{equation}
\bm \Lambda_{i}^{(t)}\coloneqq\begin{cases}
\prod_{j=1}^{t}(\bm I-\bm \Lambda_{j}), & \text{if }i=0 ,\\
\bm \Lambda_{i}\prod_{j=i+1}^{t}(\bm I-\bm \Lambda_{j}), & \text{if }0<i<t ,\\
\bm \Lambda_{t}, & \text{if }i=t ,
\end{cases}\label{def:Lambda-i-t}
\end{equation}
and setting 
\begin{equation}
	\beta = \frac{c_3(1-\gamma)}{\log T}
	\label{eq:defn-beta-async-Q}
\end{equation}
for some constant $c_3 > 0$, 
we can rearrange terms in the upper bound \eqref{eq:Delta-allstep-UB-asynQ} to reach
\begin{align}
	\bm{\Delta}_{t} &\le
	\underset{\eqqcolon\,\bm{\zeta}_{t}}{\underbrace{\bm \Lambda_{0}^{(t)}\bm{\Delta}_{0} + \gamma\sum_{i=1}^{(1-\beta)t}\bm \Lambda_{i}^{(t)}\Big[\big(\bm{P}_{i}-\bm{P}\big)\bm{V}_{i-1} + \bm{P}^{\pi_{i-1}}\bm{\Delta}_{i-1}\Big]}} \nonumber\\
	&\qquad+ \underset{\eqqcolon\,\bm{\xi}_{t}}{\underbrace{\gamma\sum_{i=(1-\beta)t+1}^{t}\bm \Lambda_{i}^{(t)}\big(\bm{P}_{i}-\bm{P}\big)\bm{V}_{i-1}}} + \gamma\sum_{i=(1-\beta)t+1}^{t}\bm \Lambda_{i}^{(t)}\bm{P}^{\pi_{i-1}}\bm{\Delta}_{i-1}. 
	\label{eq:defn-main}
\end{align}
%
In the subsequent steps, we shall first develop bounds on the sizes of the terms $\bm{\zeta}_{t}$ and $\bm{\xi}_{t}$ in \eqref{eq:defn-main} separately, and then combine these bounds with \eqref{eq:defn-main} recursively in order to derive the advertised upper bound on $\bm{\Delta}_{t}$.

\paragraph{Step 2: bounding the terms $\bm{\zeta}_{t}$ and $\bm{\xi}_{t}$.}

The terms $\bm{\zeta}_{t}$ and $\bm{\xi}_{t}$ defined in \eqref{eq:defn-main} can be bounded with high probability by the following lemmas.

\begin{lemma} \label{lem:zeta-async} 
With probability at least $1-\delta$, we have
\begin{align}
\|\bm{\zeta}_{t}\|_{\infty} \leq \frac{4}{(1-\gamma)T}
\end{align}
for all  $t$ obeying $\frac{T}{c_{4}\log T}\leq t\leq T$. Here, $c_4>0$ is some constant obeying $c_4 \le c_1c_3/4$, 
where the constants $c_1$ and $c_3$ appear in \eqref{eq:thm:infinite-horizaon-eta-asyn} and \eqref{eq:defn-beta-async-Q}, respectively. 
\end{lemma}
\begin{proof} See Section~\ref{sec:proof-lem:zeta-async}. \end{proof}
\begin{lemma} \label{lem:xi-async}
Suppose that $0<\eta \le \frac{\log^{3}T}{(1-\gamma)T\mumin}$. With probability at least $1-\delta$, one has
\begin{align}
	|\bm{\xi}_{t}| \leq \sqrt{\frac{16\big(\log^{3}T\big)\big(\log\frac{|\mathcal{S}||\mathcal{A}|T}{\delta}\big)}{(1-\gamma)T\mumin}\Big(\max_{(1-\beta)t\leq i<t}\mathsf{Var}_{\bm{P}}(\bm{V}_{i})+\bm{1}\Big)}+\frac{6\big(\log^{3}T\big)\big(\log\frac{|\mathcal{S}||\mathcal{A}|T}{\delta}\big)}{(1-\gamma)^{2}T\mumin}\bm{1}
\end{align}
for all  $t$ obeying $\frac{T}{c_{4}\log T}\leq t\leq T$ for some constant $c_4>0$.  
\end{lemma}
\begin{proof} See Section~\ref{sec:proof-lem:xi-async}. \end{proof}

\paragraph{Step 3: controlling $\bm{\Delta}_{t}$.}
Consider any $t$ obeying $\frac{T}{c_{4}\log T}\leq t\leq T$ and any $k$ obeying $2t/3 \leq k\leq t$.   
Under the sample size condition \eqref{eq:thm:infinite-horizon-T-asyn},  Lemmas~\ref{lem:zeta-async}-\ref{lem:xi-async} together with a little algebra lead to
\[
	|\bm{\zeta}_{k}| + |\bm{\xi}_{k}|  \leq 
\sqrt{\frac{32\big(\log^{3}T\big)\big(\log\frac{|\mathcal{S}||\mathcal{A}|T}{\delta}\big)}{(1-\gamma)T\mumin}\bigg(\max_{(1-\beta)k\le i\leq k}\mathsf{Var}_{\bm{P}}(\bm{V}_{i})+\bm{1}\bigg)} 
	\leq \sqrt{\bm{\varphi}_{t}}, 
\]
where we define
\begin{equation}
	\bm{\varphi}_{t}\coloneqq \frac{ 32 \big(\log^{3}T \big)\big( \log\frac{|\mathcal{S}||\mathcal{A}|T}{\delta} \big) }{(1-\gamma)T\mumin}\bigg(\max_{\frac{t}{2}\le i\leq t}\mathsf{Var}_{\bm{P}}(\bm{V}_{i})+\bm{1}\bigg). \label{def:phi-asynQ}
\end{equation}
Combining this inequality with \eqref{eq:defn-main} allows us to obtain
\begin{align}
\bm{\Delta}_{k} & \le\sqrt{\bm{\varphi}_{t}}+\sum_{i=(1-\beta)k+1}^{k}\bm \Lambda_{i}^{(k)}\gamma\bm{P}^{\pi_{i-1}}\bm{\Delta}_{i-1}=\sqrt{\bm{\varphi}_{t}}+\sum_{i=(1-\beta)k}^{k-1}\bm \Lambda_{i+1}^{(k)}\gamma\bm{P}^{\pi_{i}}\bm{\Delta}_{i}\qquad\text{for all }~2t/3\le k\le t.\label{eq:Deltak-UB-phit-infinite-asynQ}
\end{align}

Similar to the quantity $\alpha_{i}^{(t)}$ defined in \eqref{eq:defn-alpha-it-sync}, let us define
\begin{align}
	\bm D_{i}^{(t)} \defn \bigg( \sum_{j=(1-\beta)t}^{t-1}\bm \Lambda_{j+1}^{(t)}\bigg)^{-1}\bm \Lambda_{i+1}^{(t)} ,
	\label{eq:defn-Dit-async-Q}
\end{align}
which, according to \eqref{eq:sum-eta-i-t-infinite} and the definition \eqref{eq:defn-Lambda-t-async}, clearly satisfies
\begin{equation}
	\bm D_{i}^{(t)} \geq \bm{\Lambda}_{i+1}^{(t)} \geq \bm{0} \qquad \text{and} \qquad 
	\sum_{i=(1-\beta)t}^{t-1}\bm D_{i}^{(t)} = \bm{I}. 
	\label{eq:Dit-property-async}
\end{equation}
Set $i_0=t$ for notational convenience. With this set of notation and the property \eqref{eq:Dit-property-async} in mind, we can derive the following bound
\begin{align}
\bm{\Delta}_{t} & \le\sum_{i_{1}=(1-\beta)t}^{t-1}\Big(\bm D_{i_{1}}^{(t)}\sqrt{\bm{\varphi}_{t}}+\bm \Lambda_{i_{1}+1}^{(t)}\gamma\bm{P}^{\pi_{i_{1}}}\bm{\Delta}_{i_{1}}\Big)\nonumber \\
 & \leq\sum_{i_{1}=(1-\beta)t}^{t-1}\Bigg[\bm D_{i_{1}}^{(t)}\sqrt{\bm{\varphi}_{t}}+\bm \Lambda_{i_{1}+1}^{(t)}\gamma\bm{P}^{\pi_{i_{1}}}\sum_{i_{2}=(1-\beta)i_{1}}^{i_{1}-1}\Big(\bm D_{i_{2}}^{(i_{1})}\sqrt{\bm{\varphi}_{t}}+\bm \Lambda_{i_{2}+1}^{(i_{1})}\gamma\bm{P}^{\pi_{i_{2}}}\bm{\Delta}_{i_{2}}\Big)\Bigg]\nonumber \\
 & \leq\sum_{i_{1}=(1-\beta)t}^{t-1}\bm D_{i_{1}}^{(t)}\sqrt{\bm{\varphi}_{t}}+\sum_{i_{1}=(1-\beta)t}^{t-1}\bm D_{i_{1}}^{(t)}\sum_{i_{2}=(1-\beta)i_{1}}^{i_{1}-1}\big(\gamma\bm{P}^{\pi_{i_{1}}}\bm D_{i_{2}}^{(i_{1})}\big)\sqrt{\bm{\varphi}_{t}} \nonumber \\
 & \qquad\qquad +\sum_{i_{1}=(1-\beta)t}^{t-1}\sum_{i_{2}=(1-\beta)i_{1}}^{i_{1}-1}\prod_{k=1}^{2}\big(\bm \Lambda_{i_{k}+1}^{(i_{k-1})}\gamma\bm{P}^{\pi_{i_{k}}}\big)\bm{\Delta}_{i_{2}}\nonumber \\
 & =\sum_{i_{1}=(1-\beta)t}^{t-1}\bm D_{i_{1}}^{(t)}\Bigg\{\bm I + \sum_{i_{2}=(1-\beta)i_{1}}^{i_{1}-1} \gamma\bm{P}^{\pi_{i_{1}}}\bm D_{i_{2}}^{(i_{1})} \Bigg\}\sqrt{\bm{\varphi}_{t}}+\sum_{i_{1}=(1-\beta)t}^{t-1}\sum_{i_{2}=(1-\beta)i_{1}}^{i_{1}-1}\prod_{k=1}^{2}\big(\bm \Lambda_{i_{k}+1}^{(i_{k-1})}\gamma\bm{P}^{\pi_{i_{k}}}\big)\bm{\Delta}_{i_{2}}.\label{eq:Delta-recursion-inf-asynQ}
\end{align}
Here, the first relation makes use of the second property in \eqref{eq:Dit-property-async}, 
the second relation further expands $\bm{\Delta}_{i_{1}}$ in the same way as in the first line of \eqref{eq:Delta-recursion-inf-asynQ}, 
whereas the third inequality relies on the first property in \eqref{eq:Dit-property-async}.

Next, we intend to invoke the above relation multiple times to reach a simpler relation. 
Set
\begin{align}
	H \defn \frac{\log T}{1-\gamma}. 
\end{align}
Similar to the way we derive \eqref{eq:Delta-UB-temp-inf}, we can apply the relation \eqref{eq:Delta-recursion-inf-asynQ} recursively and use the basic relation $|\bm{\Delta}_{k}|\leq \frac{1}{1-\gamma}\bm{1}$ for any $k$ to show that
\begin{align}
\bm{\Delta}_{t} & \le\sum_{\left(i_{1},\cdots,i_{H}\right)\in\mathcal{I}_{t}}\Bigg\{\bm{D}_{i_{1}}^{(t)}\bigg(\bm{I}+\sum_{h=1}^{H-1}\gamma^{h}\prod_{k=1}^{h}\big(\bm{P}^{\pi_{i_{k}}}\bm{D}_{i_{k+1}}^{(i_{k})}\big)\bigg)\sqrt{\bm{\varphi}_{t}}+\gamma^{H}\prod_{k=1}^{H}\big(\bm{\Lambda}_{i_{k}+1}^{(i_{k-1})}\bm{P}^{\pi_{i_{k}}}\big)\big|\bm{\Delta}_{i_{H}}\big|\Bigg\},
	\label{eq:Delta-UB-temp-inf-asynQ-1}
\end{align}
where $\mathcal{I}_{t}$ has been defined in~\eqref{eq:defn-It-inf}. To further simplify \eqref{eq:Delta-UB-temp-inf-asynQ-1}, 
we need to control the two terms on the right-hand side of \eqref{eq:Delta-UB-temp-inf-asynQ-1} separately. 
\begin{itemize}

	\item We shall begin with the first term on the right-hand side of \eqref{eq:Delta-UB-temp-inf-asynQ-1}. Towards this end, let us define a collection of policies $\{\widehat{\pi}_k\}$ recursively and backward as follows: 
\begin{align}
\widehat{\pi}_k \defn
	\begin{cases} \arg\max_{\pi \in \Pi} \bm{\Pi}^{\pi}\sqrt{\bm{\varphi}_{t}}, & \text{if }k=H-1, \\
		\arg\max_{\pi \in \Pi} \bm \Pi^{\pi}\Big(\bm{I}+\sum_{h=k+1}^{H-1}\gamma^{h}\prod_{j=1}^{h}\bm{P}^{\widehat{\pi}_j}\Big)\sqrt{\bm{\varphi}_{t}} , \quad & \text{if }k= H-2, \cdots, 1,
	\end{cases}
	\label{eq:defn-hat-pi-Pi}
\end{align}
or alternatively (in view of the definition \eqref{eqn:ppivq} of $\bm{P}^{\pi}$), 		
\begin{align}
\widehat{\pi}_k \defn
	\begin{cases} \arg\max_{\pi \in \Pi} \bm P^{\pi}\sqrt{\bm{\varphi}_{t}}, & \text{if }k=H-1; \\
		\arg\max_{\pi \in \Pi} \bm P^{\pi}\Big(\bm{I}+\sum_{h=k+1}^{H-1}\gamma^{h}\prod_{j=1}^{h}\bm{P}^{\widehat{\pi}_j}\Big)\sqrt{\bm{\varphi}_{t}} , \quad & \text{if }k= H-2, \cdots, 1. 
	\end{cases}
	\label{eq:defn-hat-pi-P}
\end{align}
Here, $\Pi$ is a policy set satisfying
\begin{eqnarray}
\label{defn:Pi-asyncQ}
\begin{aligned}
	\Pi  \defn\big\{\pi=[\pi(s)]_{s\in\cS}\, \big| \, \pi(s) \in\Pi_{s}, \forall s\in\cS\big\}, 
	\qquad \Pi_{s}  \defn\big\{ \pi_{i}(s) \mid i\in[t/2,t)\big\};	 
\end{aligned}
\end{eqnarray}
in words, for any policy $\pi$ belonging to $\Pi$, each $\pi(s)$ coincides with one of the policy iterates $\pi_i(s)$ during the latest $\beta t$ iterations, although we do not require all $\{\pi(s)\}$ across different states to be associated with the same time stamp $i$.  
%
With this collection of policies in place, we can deduce that
\begin{align*}
 & \sum_{i_{1},\cdots,i_{H}}\bm{D}_{i_{1}}^{(t)}\bigg(\bm{I}+\sum_{h=1}^{H-1}\prod_{k=1}^{h}\big(\gamma\bm{P}^{\pi_{i_{k}}}\bm{D}_{i_{k+1}}^{(i_{k})}\big)\bigg)\sqrt{\bm{\varphi}_{t}}\\
 & =\sum_{i_{1},\cdots,i_{H-1}}\sum_{i_{H}}\bm{D}_{i_{1}}^{(t)}\bigg(\bm{I}+\sum_{h=1}^{H-2}\prod_{k=1}^{h}\big(\gamma\bm{P}^{\pi_{i_{k}}}\bm{D}_{i_{k+1}}^{(i_{k})}\big)+\prod_{k=1}^{H-2}\big(\gamma\bm{P}^{\pi_{i_{k}}}\bm{D}_{i_{k+1}}^{(i_{k})}\big)\Big(\gamma\bm{P}^{\pi_{i_{H-1}}}\bm{D}_{i_{H}}^{(i_{H-1})}\Big)\bigg)\sqrt{\bm{\varphi}_{t}}\\
 & =\sum_{i_{1},\cdots,i_{H-1}}\bm{D}_{i_{1}}^{(t)}\bigg(\bm{I}+\sum_{h=1}^{H-2}\prod_{k=1}^{h}\big(\gamma\bm{P}^{\pi_{i_{k}}}\bm{D}_{i_{k+1}}^{(i_{k})}\big)+\prod_{k=1}^{H-2}\big(\gamma\bm{P}^{\pi_{i_{k}}}\bm{D}_{i_{k+1}}^{(i_{k})}\big)\Big(\gamma\bm{P}^{\pi_{i_{H-1}}}\sum_{i_{H}}\bm{D}_{i_{H}}^{(i_{H-1})}\Big)\bigg)\sqrt{\bm{\varphi}_{t}}\\
 & \overset{(\mathrm{i})}{=}\sum_{i_{1},\cdots,i_{H-1}}\bm{D}_{i_{1}}^{(t)}\bigg(\bm{I}+\sum_{h=1}^{H-2}\prod_{k=1}^{h}\big(\gamma\bm{P}^{\pi_{i_{k}}}\bm{D}_{i_{k+1}}^{(i_{k})}\big)+\prod_{k=1}^{H-2}\big(\gamma\bm{P}^{\pi_{i_{k}}}\bm{D}_{i_{k+1}}^{(i_{k})}\big)\gamma\bm{P}^{\pi_{i_{H-1}}}\bigg)\sqrt{\bm{\varphi}_{t}}\\
 & \overset{(\mathrm{ii})}{\leq}\sum_{i_{1},\cdots,i_{H-1}}\bm{D}_{i_{1}}^{(t)}\bigg(\bm{I}+\sum_{h=1}^{H-2}\prod_{k=1}^{h}\big(\gamma\bm{P}^{\pi_{i_{k}}}\bm{D}_{i_{k+1}}^{(i_{k})}\big)+\prod_{k=1}^{H-2}\big(\gamma\bm{P}^{\pi_{i_{k}}}\bm{D}_{i_{k+1}}^{(i_{k})}\big)\gamma\bm{P}^{\widehat{\pi}_{H-1}}\bigg)\sqrt{\bm{\varphi}_{t}}\\
 & =\sum_{i_{1},\cdots,i_{H-2}}\bm{D}_{i_{1}}^{(t)}\bigg(\bm{I}+\sum_{h=1}^{H-3}\prod_{k=1}^{h}\big(\gamma\bm{P}^{\pi_{i_{k}}}\bm{D}_{i_{k+1}}^{(i_{k})}\big)+\prod_{k=1}^{H-3}\big(\gamma\bm{P}^{\pi_{i_{k}}}\bm{D}_{i_{k+1}}^{(i_{k})}\big)\gamma\bm{P}^{\pi_{i_{H-2}}}\bigg(\sum_{i_{H-1}}\bm{D}_{i_{H-1}}^{(i_{H-2})}\bigg)\big(\bm{I}+\gamma\bm{P}^{\widehat{\pi}_{H-1}}\big)\bigg)\sqrt{\bm{\varphi}_{t}}\\
 & \overset{(\mathrm{iii})}{=}\sum_{i_{1},\cdots,i_{H-2}}\bm{D}_{i_{1}}^{(t)}\bigg(\bm{I}+\sum_{h=1}^{H-3}\prod_{k=1}^{h}\big(\gamma\bm{P}^{\pi_{i_{k}}}\bm{D}_{i_{k+1}}^{(i_{k})}\big)+\prod_{k=1}^{H-3}\big(\gamma\bm{P}^{\pi_{i_{k}}}\bm{D}_{i_{k+1}}^{(i_{k})}\big)\gamma\bm{P}^{\pi_{i_{H-2}}}\big(\bm{I}+\gamma\bm{P}^{\widehat{\pi}_{H-1}}\big)\bigg)\sqrt{\bm{\varphi}_{t}}\\
 & \overset{(\mathrm{iv})}{\leq}\sum_{i_{1},\cdots,i_{H-2}}\bm{D}_{i_{1}}^{(t)}\bigg(\bm{I}+\sum_{h=1}^{H-3}\prod_{k=1}^{h}\big(\gamma\bm{P}^{\pi_{i_{k}}}\bm{D}_{i_{k+1}}^{(i_{k})}\big)+\prod_{k=1}^{H-3}\big(\gamma\bm{P}^{\pi_{i_{k}}}\bm{D}_{i_{k+1}}^{(i_{k})}\big)\gamma\bm{P}^{\widehat{\pi}_{H-2}}\big(\bm{I}+\gamma\bm{P}^{\widehat{\pi}_{H-1}}\big)\bigg)\sqrt{\bm{\varphi}_{t}}, 
\end{align*}
where we abbreviate $\sum_{(i_{1},\cdots,i_{H})\in \mathcal{I}_t}$ as $\sum_{i_{1},\cdots,i_{H}}$ 
as long as it is clear from the context.  Here, (i) and (iii) arise from the second property in \eqref{eq:Dit-property-async}, 
while (ii) and (iv) are due to the construction \eqref{eq:defn-hat-pi-P}. 
Continuing the derivation of the above inequality recursively, we arrive at  
\begin{align*}
\sum_{i_{1},\cdots,i_{H}}\bm{D}_{i_{1}}^{(t)}\bigg(\bm{I}+\sum_{h=1}^{H-1}\prod_{k=1}^{h}\big(\gamma\bm{P}^{\pi_{i_{k}}}\bm{D}_{i_{k+1}}^{(i_{k})}\big)\bigg)\sqrt{\bm{\varphi}_{t}} & \le\bigg(\bm{I}+\sum_{h=1}^{H-1}\gamma^{h}\prod_{k=1}^{h}\bm{P}^{\widehat{\pi}_{k}}\bigg)\sqrt{\bm{\varphi}_{t}}.
\end{align*}

\item

We now turn attention to the second term on the right-hand side of \eqref{eq:Delta-UB-temp-inf-asynQ-1}. It is seen that
\begin{align*}
\sum_{i_{1},\cdots,i_{H}}\prod_{k=1}^{H} \gamma^H \big(\bm{\Lambda}_{i_{k}+1}^{(i_{k-1})}\bm{P}^{\pi_{i_{k}}}\big)\big|\bm{\Delta}_{i_{H}}\big| & \leq\frac{\gamma^H}{1-\gamma}\sum_{i_{1},\cdots,i_{H}}\prod_{k=1}^{H}\big(\bm{D}_{i_{k}}^{(i_{k-1})}\bm{P}^{\pi_{i_{k}}}\big)\bm{1}\\
& = \frac{\gamma^H}{1-\gamma}\sum_{i_{1},\cdots,i_{H-1}}\prod_{k=1}^{H-1}\big(\bm{D}_{i_{k}}^{(i_{k-1})}\bm{P}^{\pi_{i_{k}}}\big)\bigg(\sum_{i_{H}}\bm{D}_{i_{H}}^{(i_{H-1})} \big( \bm{P}^{\pi_{i_{H}}}\bm{1} \big)\bigg)\\
 & = \frac{\gamma^H}{1-\gamma}\sum_{i_{1},\cdots,i_{H-1}}\prod_{k=1}^{H-1}\big(\bm{D}_{i_{k}}^{(i_{k-1})}\bm{P}^{\pi_{i_{k}}}\big)\bigg(\sum_{i_{H}}\bm{D}_{i_{H}}^{(i_{H-1})}\bigg) \bm{1} \\
 & =\frac{\gamma^H}{1-\gamma}\sum_{i_{1},\cdots,i_{H-1}}\prod_{k=1}^{H-1}\big(\bm{D}_{i_{k}}^{(i_{k-1})}\bm{P}^{\pi_{i_{k}}}\big)\bm{1}\\
	& =\cdots =\frac{\gamma^H}{1-\gamma}\bm{1} = \frac{\gamma^{\frac{\log T}{1-\gamma}}}{1-\gamma}\bm{1} \\
	& \leq \frac{1}{(1-\gamma)T} \bm{1},
\end{align*}
where the first line follows from the first property in \eqref{eq:Dit-property-async}, 
the third line is due to the fact $\bm{P}^{\pi}\bm{1}=\bm{1}$ for any $\pi$,  
and the fourth line arises from the second property in \eqref{eq:Dit-property-async}.

\end{itemize}
Substituting the above two bounds into \eqref{eq:Delta-UB-temp-inf-asynQ-1} yields
\begin{align}
\bm{\Delta}_{t} 
	& \leq\underbrace{\bigg(\bm{I}+\sum_{h=1}^{H-1}\gamma^{h}\prod_{k=1}^{h}\bm{P}^{\widehat{\pi}_{k}}\bigg)\sqrt{\bm{\varphi}_{t}}}_{\eqqcolon \,\bm{\beta}}+\frac{1}{(1-\gamma)T}\bm{1}. 
	\label{eq:Delta-UB-temp-inf-asynQ}
\end{align}

\paragraph{Step 4: putting all pieces together.}
Repeating our analysis for the term $\bm{\beta}_1$ in Section~\ref{sec:proof_inf_upper_bound} (i.e., Step 5 of Section~\ref{sec:proof_inf_upper_bound} with Lemma~\ref{lemma:inf-alpha-claim} replaced by Lemma~\ref{lemma:inf-alpha-claim-async}), 
we arrive at 
\begin{align*}
|\bm{\beta}|^{2} 
& \leq 
\frac{320\big(\log^{4}T\big)\big(\log\frac{|\mathcal{S}||\mathcal{A}|T}{\delta}\big)}{\gamma^2 (1-\gamma)^{4}T}\Big(1+2\max_{\frac{t}{2}\le i<t}\|\bm{\Delta}_{i}\|_{\infty}\Big)\bm{1} 
\end{align*}
with probability at least $1-\delta$. 
Substitution into \eqref{eq:Delta-UB-temp-inf-asynQ} then yields
\begin{align}
\notag \bm{\Delta}_{t} &\leq 
 \sqrt{\frac{320\big(\log^{3}T\big)\big(\log\frac{|\mathcal{S}||\mathcal{A}|T}{\delta}\big)}{\gamma^2 (1-\gamma)^{4}T\mumin}\Big(1+\max_{\frac{t}{2}\le i<t}\|\bm{\Delta}_{i}\|_{\infty}\Big)}\ \bm{1} + \frac{1}{(1-\gamma)T}\bm{1} \\ 
 &\leq 30 \sqrt{\frac{\big(\log^{3}T\big)\big(\log\frac{|\mathcal{S}||\mathcal{A}|T}{\delta}\big)}{\gamma^2 (1-\gamma)^{4}T\mumin}\Big(1+\max_{\frac{t}{2}\le i<t}\|\bm{\Delta}_{i}\|_{\infty}\Big)}\ \bm{1}
	\label{eq:Delta-UB-final-asynQ}
\end{align}
holds simultaneously for all $t\ge\frac{T}{c_{4}\log T}$, provided that the sample size condition \eqref{eq:thm:infinite-horizon-T-asyn} is satisfied. 
Similarly, we can also establish the following lower bound on $\bm{\Delta}_{t}$ (which we omit the details for the sake of brevity)
\begin{align*}
\bm{\Delta}_{t} \geq -30 \sqrt{\frac{\big(\log^{4}T\big)\big(\log\frac{|\mathcal{S}||\mathcal{A}|T}{\delta}\big)}{\gamma^2 (1-\gamma)^{4}T\mumin}\Big(1+\max_{\frac{t}{2}\le i<t}\|\bm{\Delta}_{i}\|_{\infty}\Big)}\ \bm{1} 
\end{align*}
with probability at least $1-\delta$. To summarize, it is seen that with probability exceeding $1-2\delta$, 
\begin{align}
	\|\bm{\Delta}_{t} \|_{\infty} \leq 30 \sqrt{\frac{\big(\log^{4}T\big)\big(\log\frac{|\mathcal{S}||\mathcal{A}|T}{\delta}\big)}{\gamma^2 (1-\gamma)^{4}T\mumin}\Big(1+\max_{\frac{t}{2}\le i<t}\|\bm{\Delta}_{i}\|_{\infty}\Big)}. 
	\label{eq:Delta-UB-LB-final-asynQ}
\end{align}
This resembles the relation \eqref{eq:Delta-t-inf-recursion-inf-outline} derived for the synchronous case, 
except that $T$ in the denominator is replaced with $\mumin T$. 
As a result, we can readily repeat the argument in Appendix~\ref{sec:combine_ub_lb} to reach
\begin{align}
	\|\bm{\Delta}_{T}\|_{\infty} \leq O\bigg( \sqrt{\frac{\big(\log^{4}T  \big)\big(\log\frac{|\mathcal{S}||\mathcal{A}|T}{\delta}\big)}{(1-\gamma)^{4}T\mumin}} +\frac{\big(\log^{4}T\big)\big(\log\frac{|\mathcal{S}||\mathcal{A}|T}{\delta}\big)}{(1-\gamma)^{4}T\mumin} 
	\bigg), \label{eq:Delta-T-UB-final-inf-async}
\end{align}
which in turn establishes the claimed result in Theorem~\ref{thm:infinite-horizon-asyn}. 


\subsection{Proofs of technical lemmas}

\subsubsection{Proof of Lemma~\ref{lem:zeta-async}}
\label{sec:proof-lem:zeta-async}

In view of the definition of $\bm{\zeta}_{t}$ in \eqref{eq:defn-main} and the fact that $\bm \Lambda_{0}^{(t)}$ is a diagonal matrix, 
we can deduce that
\begin{align*}
\|\bm{\zeta}_{t}\|_{\infty} 
& \leq \|\bm \Lambda_{0}^{(t)}\|\|\bm{\Delta}_{0}\|_{\infty}+t\max_{1\leq i\leq(1-\beta)t}\|\bm \Lambda_{i}^{(t)}\|\max_{1\leq i\leq(1-\beta)t}\big(\|\bm{P}^{\pi_{i-1}}\bm{\Delta}_{i-1}\|_{\infty}+\|\bm{P}_{i}\bm{V}_{i-1}\|_{\infty}+\|\bm{P}\bm{V}_{i-1}\|_{\infty}\big)\\
 & \leq\|\bm \Lambda_{0}^{(t)}\|\|\bm{\Delta}_{0}\|_{\infty}+t\max_{1\leq i\leq(1-\beta)t}\|\bm \Lambda_{i}^{(t)}\|\max_{1\leq i\leq(1-\beta)t}\Big\{\|\bm{P}^{\pi_{i-1}}\|_{1}\|\bm{\Delta}_{i-1}\|_{\infty}+\big(\|\bm{P}_{i}\|_{1}+\|\bm{P}\|_{1}\big)\|\bm{V}_{i-1}\|_{\infty}\Big\}\\
 & \overset{(\mathrm{i})}{=}\|\bm \Lambda_{0}^{(t)}\|\|\bm{\Delta}_{0}\|_{\infty}+t\max_{1\leq i\leq(1-\beta)t}\|\bm \Lambda_{i}^{(t)}\|\max_{1\leq i\leq(1-\beta)t}\big(\|\bm{\Delta}_{i-1}\|_{\infty}+2\,\|\bm{V}_{i-1}\|_{\infty}\big)\\
 & \overset{(\mathrm{ii})}{\leq}\frac{1}{T^{2}}\cdot\frac{1}{1-\gamma}+\frac{1}{T^{2}}\cdot t\cdot\frac{3}{1-\gamma}\\
 & \le\frac{4}{(1-\gamma)T}.
\end{align*}
Here, (i) holds true since $\|\bm{P}^{\pi_{i-1}}\|_{1}=\|\bm{P}_{i}\|_{1}=\|\bm{P}\|_{1}=1$. To verify (ii), we first define
\begin{equation}\label{eq:defn-tk-sa}
	t_{k}(s,a) \defn \text{the time stamp when }\text{the trajectory visits }(s,a)\text{ for the }k\text{-th time}
\end{equation}
and
\begin{equation}
	K_{t}(s,a) \defn  \Big| \big\{ k \geq 1\mid t_{k}(s,a) < t\big\} \Big|,  \label{eq:defn-Kt}
\end{equation}
namely, the total number of times --- before the $t$-th iteration ---
that the sample trajectory visits $(s,a)$.
Then \citet[Lemma 8]{li2020sample} tells us that with probability at least $1-\delta$, 
\begin{align}
K_{t_1}(s,a)-K_{t_2}(s, a) \ge \frac{1}{2}(t_1 - t_2)\mumin,
\end{align}
holds uniformly for all $(s, a) \in \cS\times\cA$ and $0 \le t_2 \le t_1 \le T$ obeying 
$$
	t_1 - t_2 \ge \frac{886t_{\mathsf{mix}}}{\mumin}\log\frac{|\cS||\cA|T}{\delta}.
$$
This in turn implies that: 
if $\beta t \ge \frac{886t_{\mathsf{mix}}}{\mumin}\log\frac{|\cS||\cA|T}{\delta}$ and $i\leq (1-\beta)t$, then one has
\begin{subequations}
	\label{eq:Lambda-0t-Lambda-it-async-Q}
\begin{align}
	\|\bm \Lambda_{0}^{(t)}\| &= \bigg\| \prod_{j=1}^{t}\big(\bm{I}-\bm{\Lambda}_{j}\big)\bigg\| 
	= \max_{(s, a)\in \cS\times \cA}(1-\eta)^{K_{t}(s,a)} \leq (1-\eta)^{\frac{1}{2}t\mumin} \le \frac{1}{T^{2}} \\
\|\bm \Lambda_{i}^{(t)}\| &= \bigg \| \bm \Lambda_{i}\prod_{j=i+1}^{t}(\bm I-\bm \Lambda_{j})\bigg\|  \leq \max_{(s, a)} (1-\eta)^{K_{t}(s,a)-K_{i}(s, a)} \leq (1-\eta)^{\frac{1}{2}\beta t\mumin} \le \frac{1}{T^{2}}
\end{align}
\end{subequations}
with probability at least $1-\delta$, 
provided that $\eta\beta t\mumin > 4\log T$. 
In other words, \eqref{eq:Lambda-0t-Lambda-it-async-Q} holds with probability at least $1-\delta$, as long as
\[
t>\max\left\{ \frac{4\log T}{\eta\beta\mumin}, \, \frac{886\tmix}{\beta\mumin}\log\frac{|\cS||\cA|T}{\delta}\right\} =\max\left\{ \frac{4T}{c_{1}c_{3}\log T}, \, \frac{886\tmix}{c_{3}(1-\gamma)\mumin}\log\frac{|\cS||\cA|T}{\delta}\log T\right\} .  	
\]
This taken together with the sample size assumption \eqref{eq:thm:infinite-horizon-T-asyn} concludes the proof of Lemma~\ref{lem:zeta-async}.

\subsubsection{Proof of Lemma~\ref{lem:xi-async}}
\label{sec:proof-lem:xi-async}

Fix any state-action pair $(s,a)\in\mathcal{S}\times\mathcal{A}$, and let us look at the $(s,a)$-th entry of $\bm{\xi}_{t}$, i.e., 
$\xi_{t}(s,a)$.
For notational simplicity, let $\bm{\Lambda}_{j}(s,a)$ denote the $(s,a)$-th diagonal entry
of the diagonal matrix $\bm{\Lambda}_{j}$, and $\bm{P}_{t}(s,a)$
(resp.~$\bm{P}(s,a)$) the $(s,a)$-th row of $\bm{P}_{t}$ (resp.~$\bm{P}$).

Using the definition of $\bm{\xi}_{t}$ in \eqref{eq:defn-main} and the above notation, we can derive
\begin{align}
	\xi_{t}(s,a) & = \gamma \sum_{i=(1-\beta)t+1}^{t}\prod_{j=i+1}^{t}\big(1-\bm{\Lambda}_{j}(s,a)\big)\bm{\Lambda}_{i}(s,a)\big(\bm{P}_{i}(s,a)-\bm{P}(s,a)\big)\bm{V}_{i-1}.
	\label{eq:expression-xi-sa}
\end{align}
Equipped with the definitions of $t_k(s,a)$ (cf.~\eqref{eq:defn-tk-sa}) and $K_t(s,a)$ (cf.~\eqref{eq:defn-Kt}), 
we can further rewrite \eqref{eq:expression-xi-sa} as
\begin{align}
	\xi_{t}(s,a) & = \gamma \sum_{k=K_{(1-\beta)t+1}}^{K_{t}(s,a)}(1-\eta)^{K_{t}(s,a)-k}\eta\big(\bm{P}_{t_{k}+1}(s,a)-\bm{P}(s,a)\big)\bm{V}_{t_{k}}.\label{eq:expression-xi-sa-1}
\end{align}
In what follows, we shall suppress the notation and write $t_k = t_k(s,a)$ 
and $K_t = K_t(s,a)$  to streamline notation.

 
The main step thus boils down to controlling \eqref{eq:expression-xi-sa-1}. 
Towards this, we claim that: with probability at least $1-\delta$,
\begin{align}
	& \Bigg|\sum_{k=K_{\beta}}^{K}(1-\eta)^{K-k}\eta\big(\bm{P}_{t_{k}+1}(s,a)-\bm{P}(s,a)\big)\bm{V}_{t_{k}}\Bigg| \nonumber\\
	&\qquad \leq \sqrt{\frac{16\big(\log^{3}T\big)\big(\log\frac{|\mathcal{S}||\mathcal{A}|T}{\delta}\big)}{(1-\gamma)T\mumin}\Big(\max_{t_{K_{\beta}}\leq i \leq t_K}\mathsf{Var}_{\bm{P}(s, a)}(\bm{V}_{i})+1\Big)}+\frac{6\big(\log^{3}T\big)\big(\log\frac{|\mathcal{S}||\mathcal{A}|T}{\delta}\big)}{(1-\gamma)^{2}T\mumin} 
	\label{eq:claim-Bernstein-fixed-K}
\end{align}
holds simultaneously for all $(s,a)\in\mathcal{S}\times\mathcal{A}$ and
all $1 \le K_{\beta} \le K \le T$, provided that $0<\eta \le \frac{\log^{3}T}{(1-\gamma)T\mumin}$. 
If this claim were true, then 
taking $K_{\beta} = K_{ (1-\beta)t + 1 }$ and $K = K_t$ and substituting the bound \eqref{eq:claim-Bernstein-fixed-K} into the expression \eqref{eq:expression-xi-sa-1} would lead to
\begin{align}
	|\bm{\xi}_{t}| \leq \sqrt{\frac{16\big(\log^{3}T\big)\big(\log\frac{|\mathcal{S}||\mathcal{A}|T}{\delta}\big)}{(1-\gamma)T\mumin}\Big(\max_{(1-\beta)t\leq i<t}\mathsf{Var}_{\bm{P}}(\bm{V}_{i})+\bm{1}\Big)}+\frac{6\big(\log^{3}T\big)\big(\log\frac{|\mathcal{S}||\mathcal{A}|T}{\delta}\big)}{(1-\gamma)^{2}T\mumin}\bm{1}, 
\end{align}
thus concluding the proof of this lemma. To finish up, it is sufficient to justify the claim (\ref{eq:claim-Bernstein-fixed-K}), which forms the content of the remainder of this proof.

\begin{proof}[Proof of the claim (\ref{eq:claim-Bernstein-fixed-K})]
Let us use the notation in \eqref{def:eta-i-t} to express
$\eta_k^{(K)} = (1-\eta)^{K-k}\eta$.  For any fixed integer $K>0$, the following vectors
\[
\left\{ \bm{P}_{t_{k}+1}(s,a)\mid1\leq k\leq K\right\} 
\]
are identically and independently distributed; see \citet[Section~B.1]{li2020sample}.
We can then express the term
\[
X_K \defn \sum_{k=K_{\beta}}^{K}(1-\eta)^{K-k}\eta\big(\bm{P}_{t_{k}+1}(s,a)-\bm{P}(s,a)\big)\bm{V}_{t_{k}},
\]
as follows: 
\[
X_K=\sum_{k=K_{\beta}}^{K}z_{k}\qquad\text{with }z_{k}\coloneqq\eta_k^{(K)}\big(\bm{P}_{t_{k}+1}(s,a)-\bm{P}(s,a)\big)\bm{V}_{t_{k}},
\]
where the $z_{k}$'s satisfy 
\[
	\mathbb{E}\left[z_{k}\mymid t_{k},\cdots,t_{1}, \bm{V}_{t_{k}},\cdots,\bm{V}_{t_{1}}\right]=0.
\]

We intend to invoke the Freedman inequality 
to control $X_K$ for any $K$ obeying $K\leq T$.
Similar to the synchronous counterpart, we can see that
\begin{align*}
	B & \coloneqq \max_{1<k\leq K}\|z_{k}\|_{\infty} \leq \eta_k^K \big( \big\|\bm{P}_{t_k+1} \|_1 + \|\bm{P}\|_1 \big) \big\|\bm{V}_{t_k}\|_{\infty} \leq \frac{2\eta_k^K}{1-\gamma} \leq \frac{2\eta}{1-\gamma} , \\
W & \coloneqq\sum_{k=K_{\beta}}^{K}\mathsf{Var}\big(z_{k}\mymid t_{k},\cdots,t_{1},\bm{V}_{t_{k}},\cdots,\bm{V}_{t_{1}}\big)=\gamma^{2}\sum_{k=K_{\beta}}^{K}\big(\eta_{k}^{(K)}\big)^{2}\mathsf{Var}\big((\bm{P}_{t_{k}+1}-\bm{P})\bm{V}_{t_{k}}\mid\bm{V}_{t_{k}}\big)\\
 & \leq\sum_{k=K_{\beta}}^{K}\big(\eta_{k}^{(K)}\big)^{2}\mathsf{Var}_{\bm{P}(s,a)}\big(\bm{V}_{t_{k}}\big)\leq\Big(\max_{K_{\beta}\leq k\leq K}\eta_{k}^{(K)}\Big)\bigg(\sum_{k=K_{\beta}}^{K}\eta_{k}^{(K)}\bigg)\mathsf{Var}_{\bm{P}(s,a)}\big(\bm{V}_{t_{k}}\big)\\
 & \leq\eta\max_{K_{\beta}\leq k\leq K}\mathsf{Var}_{\bm{P}(s,a)}\big(\bm{V}_{t_{k}}\big)\leq\eta\max_{t_{K_{\beta}}\leq i\leq t_{K}}\mathsf{Var}_{\bm{P}(s,a)}\big(\bm{V}_{i}\big) ,
\end{align*}
where we have made use of \eqref{eq:sum-eta-i-t-infinite}. In addition, we make note of a trivial upper bound on $W$ as follows
	\begin{align*}
\sigma^{2} & :=\frac{\eta}{(1-\gamma)^2} \geq \eta\max_{t_{K_{\beta}}\leq i\leq t_{K}}\mathsf{Var}_{\bm{P}(s,a)}\big(\bm{V}_{i}\big) 
		\geq W_{t}.
\end{align*}
With the preceding bounds in place, applying the Freedman inequality in Theorem \ref{thm:Freedman} 
and taking $L=\log_2 \frac{1}{1-\gamma} $ imply that 
\begin{align*}
|X_{K}| & \le\sqrt{8\max\bigg\{ W,\frac{\sigma^{2}}{2^{L}}\bigg\}\log\frac{4|\mathcal{S}||\mathcal{A}|T^{2}\log_{2}\frac{1}{1-\gamma}}{\delta}}+\frac{8\eta}{3(1-\gamma)}\log\frac{4|\mathcal{S}||\mathcal{A}|T^{2}\log_{2}\frac{1}{1-\gamma}}{\delta}\\
 & \le\sqrt{8\eta\max\Big\{\max_{t_{K_{\beta}}\leq i\leq t_{K}}\mathsf{Var}_{\bm{P}(s,a)}(\bm{V}_{i}),1\Big\}\log\frac{4|\mathcal{S}||\mathcal{A}|T^{2}\log_{2}\frac{1}{1-\gamma}}{\delta}}+\frac{8\eta}{3(1-\gamma)}\log\frac{4|\mathcal{S}||\mathcal{A}|T^{2}\log_{2}\frac{1}{1-\gamma}}{\delta}\\
 & \le\sqrt{\frac{16\big(\log^{3}T\big)\big(\log\frac{|\mathcal{S}||\mathcal{A}|T}{\delta}\big)}{(1-\gamma)T\mumin}\Big(\max_{t_{K_{\beta}}\leq i\leq t_{K}}\mathsf{Var}_{\bm{P}(s,a)}(\bm{V}_{i})+1\Big)}+\frac{6\big(\log^{3}T\big)\big(\log\frac{|\mathcal{S}||\mathcal{A}|T}{\delta}\big)}{(1-\gamma)^{2}T\mumin}
\end{align*}
with probability at least $1 - \frac{\delta}{|\mathcal{S}||\mathcal{A}|T^2}$, provided that $\eta \le \frac{\log^{3}T}{(1-\gamma)T\mumin}$.
We can thus conclude the proof by taking the union bound over all $(s, a)\in \cS \times \cA$ and all $1 \le K_{\beta} \le K \le T$. 
\end{proof}

%% file: lower_asynQ.tex
\section{Lower bound for asynchronous Q-learning (Theorem~\ref{thm:LB-asynQ})}
\label{sec:lower-bound-async-Q}

This section establishes Theorem~\ref{thm:LB-asynQ} by identifying a hard MDP instance satisfying the assumed conditions.

\subsection{Construction of a hard instance and its values}
Let us construct an MDP $\mathcal{M}_{\mathsf{hard}}$ with state space $\cS=\{0,1,2,3\}$ as follows, 
which is partly inspired by the idea from~\citet{li2022settling}. We shall denote by $\cA_s$ the action space associated with state $s$. 
The probability transition kernel and the reward function of $\mathcal{M}_{\mathsf{hard}}$ are specified as follows:  
\begin{subequations}
\label{eq:hard-MDP-asynQ}
\begin{align}
	& \mathcal{A}_0=\{1,2\},  & P(\cdot \mymid 0, 1) &= [1, 0, 0, 0], & r(0, 1) = \frac{2}{3}, \\
	& \quad   & P(\cdot \mymid 0, 2) &= \big[1-2(1-\gamma)(1-\mu_0), 2(1-\gamma)\mu_1, 2(1-\gamma)\mu_2, 2(1-\gamma)\mu_3\big],\hspace{-0.5em}   & r(0, 2) = 0, \\
	& \mathcal{A}_1=\{1,2,3\}, \hspace{-0.5em} & P(\cdot \mymid 1, 1) &= \bigg[\frac{3}{2}(1-\gamma)\mu_0, 1 - \frac{3}{2}(1-\gamma)\mu_0, 0, 0\bigg],   & r(1, 1) = 1,  \\
	& \quad   & P(\cdot \mymid 1, 2) &= \bigg[\frac{3}{2}(1-\gamma)\mu_0, 1 - \frac{3}{2}(1-\gamma)\mu_0, 0, 0\bigg],  & r(1, 2) = 1, \\
	& \quad   & P(\cdot \mymid 1, 3) &= \big[0, 1 - 3(1-\gamma)(\mu_2 + \mu_3), 3(1-\gamma)\mu_2, 3(1-\gamma)\mu_3\big],  & r(1, 3) = 0, \\
	& \mathcal{A}_2=\{1,2\},  & P(\cdot \mymid 2, 1) &= \big[2(1-\gamma)\mu_0, 0, 1-2(1-\gamma)\mu_0, 0\big],  & r(2,1) = 1, \\
	& \quad  & P(\cdot \mymid 2, 2) &= \big[0, 2(1-\gamma)\mu_1, 1 - 2(1-\gamma)(\mu_1 + \mu_3), 2(1-\gamma)\mu_3\big],  & r(2,2) = 0, \\
	& \mathcal{A}_3=\{1,2\},  & P(\cdot \mymid 3, 1) &= [0, 0, 0, 1], & r(3, 1) = \frac{3}{4}, \\
	& \quad  & P(\cdot \mymid 3, 2) &= \big[2(1-\gamma)\mu_0, 2(1-\gamma)\mu_1, 2(1-\gamma)\mu_2, 1-2(1-\gamma)(1-\mu_3)\big], \hspace{-0.5em} & r(3,2) = 0, 
\end{align}
\end{subequations}
where the parameter $\mu = [\mu_0,\mu_1,\mu_2,\mu_3]\in \Delta(\cS)$  is set as 
\begin{align}
	\label{defn-parameter-mu}
	\mu = \Big[\frac{2}{5},\, \frac{2}{5} - \frac{c_{\mu}}{\log^2 T},\, \frac{1}{5},\, \frac{c_{\mu}}{\log^2 T}\Big],
\end{align}
for some sufficiently small quantity $0<c_{\mu} =O(1)$. In particular, actions $1$ and $2$ from state $1$ are identical, which will play a similar critical role as the case of synchronous Q-learning in pinpointing the ``over-estimation'' issue.

In addition, let the behavior policy $\pi_{\mathsf{b}}$ be uniform distributions such that
\begin{align*}
	\pi_{\mathsf{b}}(a \mymid s) = \frac{1}{|\cA_s|},\qquad\text{for all }s \in \cS\text{ and all }a\in \cA_s. 
\end{align*}
Then the transition probability from state $s$ to state $s'$ under the behavior policy $\pi_{\mathsf{b}}$ is given by
\begin{align*}
	P^{\pi_{\mathsf{b}}}(s'\mymid s) = \sum_{a\in \cA_s} \pi_{\mathsf{b}}(a\mymid s)P(s' \mymid s, a) = \gamma\mathds{1}(s' = s) + (1-\gamma)\mu(s').
\end{align*}
With this in mind, it can be easily verified that the stationary distribution under $\pib$ is given by 
\begin{align*}
	\mu_{\mathsf{b}}(s, a) = \frac{1}{|\cA_s|}\mu(s),
	\qquad \text{for all }s\in \cS \text{ and all }a\in \cA_s. 
\end{align*}
This together with \eqref{defn-parameter-mu} implies that
\begin{align}
	\mumin = \frac{c_{\mu}}{3\log^2 T}. 
	\label{eq:construction-mumin}
\end{align}
Moreover, if the sample trajectory is initialized with an initial state distribution $\mu_0$, 
then the marginal state distribution at time $t$ can be calculated as
\begin{align*}
\mu_t = (P^{\pi_{\mathsf{b}}})^t\mu_0 = (\gamma I + (1-\gamma)\mu1^{\top})^t\mu_0 = \gamma^t\mu_0 + (1-\gamma^t)\mu,
\end{align*} 
thus indicating that the total variation distance $d_{\mathsf{TV}}$ between $\mu_t$ and $\mu$ obeys 
\begin{align*}
	d_{\mathsf{TV}}(\mu_t, \mu) = \frac{1}{2} \| \mu_t - \mu \|_1 = \frac{1}{2} \gamma^t\| \mu_0 - \mu \|_1.
\end{align*}
Consequently, 
the mixing time of the sample trajectory~\citep{paulin2015concentration} obeys 
\begin{equation}
	t_{\mathsf{mix}} \in \Big[\frac{c_{\mathsf{mix},1}}{1-\gamma},\frac{c_{\mathsf{mix},2}}{1-\gamma} \Big]
	\label{eq:mixing-time-asyn-Q-LB-135}
\end{equation}
for some universal constants $c_{\mathsf{mix},1},c_{\mathsf{mix},2}>0$.

Before embarking on the analysis of the behavior of asynchronous Q-learning, 
let us first look at the optimal value function and Q-function of the constructed MDP. 
\begin{lemma}
\label{lem:optimal-V-Q-hardMDP-asynQ}
Consider the MDP $\mathcal{M}_{\mathsf{hard}}$ as constructed in \eqref{eq:hard-MDP-asynQ}. It holds that
\begin{subequations}
\label{eq:optimal-V-Q-hardMDP-asynQ}
\begin{align}
	V^{\star}(0) &= Q^{\star}(0,1) = \frac{2}{3(1-\gamma)}; \\
	Q^{\star}(0,2) &< V^{\star}(0) - \frac{4}{15}; \\
	V^{\star}(1) &= Q^{\star}(1, 1) = Q^{\star}(1, 2) = \frac{5+2\gamma}{(5+3\gamma)(1-\gamma)};\\
	Q^{\star}(1,3) &< V^{\star}(1) - \frac{7}{8}; \\
	V^{\star}(2) &= Q^{\star}(2,1) = \frac{15+8\gamma}{3(5+4\gamma)(1-\gamma)}; \\
	Q^{\star}(2,2) &< V^{\star}(2) - \frac{7}{9}; \\
	V^{\star}(3) &= Q^{\star}(3,1) = \frac{3}{4(1-\gamma)}; \\
	Q^{\star}(3,2) &< V^{\star}(3) - \frac{1}{4}.
\end{align}
\end{subequations}
\end{lemma}

\begin{proof}
In order to justify \eqref{eq:optimal-V-Q-hardMDP-asynQ}, 
let us begin by defining two vectors $\bm{V}$ and $\bm{Q}$ as follows:
\begin{align*}
	\bm{V} &= \Big[\frac{2}{3(1-\gamma)}, \frac{5+2\gamma}{(5+3\gamma)(1-\gamma)}, \frac{15+8\gamma}{3(5+4\gamma)(1-\gamma)}, \frac{3}{4(1-\gamma)}\Big]^{\top}, \\
	\bm{Q} &= \bm{r} + \gamma \bm{P} \bm{V}, 
\end{align*}
where $\bm{r} = [r(s,a)]$ denotes the reward vector. 
Then the claimed expressions of the value function in \eqref{eq:optimal-V-Q-hardMDP-asynQ} 
are valid as long as we can validate $V(s) = \max_a Q(s, a)$ (namely, they satisfy the Bellman equation). 
These are elementary calculations, which we omit for brevity. 
Once the expressions of both the value function and the Q-function are settled, 
the remaining set of advertised inequalities can be validated straightforwardly, 
which is omitted as well for conciseness. 
%
%
\end{proof}

\subsection{Analysis for the constructed MDP}

We now proceed to analyze the dynamics of asynchronous Q-learning when applied to the above MDP instance, 
for which we divide into three cases based on the magnitudes of the learning rates. 
Throughout the proof, 
we denote by $t_k(s,a)$ the iteration number corresponding to the $k$-th time the state-action pair $(s,a)$ is visited, 
and let $N_T(s,a)$ represent the total number of visits to $(s,a)$ up to time $T$. 
We shall also reuse the notation of the sample transition matrix $\bm{P}_t$ defined in \eqref{eq:defn-Pt-async-Q}, as well as the value estimate vector $\bm{V}_t=[V_t(s)]_{s\in \cS}$ as in Appendix~\ref{sec:notation-AsynQ-analysis}.

\subsubsection{Case 1: small learning rates ($\eta < \frac{1}{\mumin(1-\gamma)T}$)}

In this case, we focus attention on analyzing state $3$. 
To begin with, we claim that if $\eta \leq c_8\frac{(1-\gamma)^2}{\log T}$ for some sufficiently small constant $c_8>0$, 
then one has, for all $t\leq T$, 
\begin{align}
Q_t(3, 1) \le V^{\star}(3)\qquad\text{and}\qquad Q_t(3, 2) < V^{\star}(3) - \frac{1}{6} .
	\label{eq:claim-Qt-31-UB-async}
\end{align}
Given the assumption $\mumin T \geq \frac{c_3\log T}{(1-\gamma)^4}$, 
the regime $\eta < \frac{1}{\mumin(1-\gamma)T}$ clearly satisfies $\eta \leq c_8\frac{(1-\gamma)^2}{\log T}$.

\begin{proof}[Proof of Claim~\eqref{eq:claim-Qt-31-UB-async}]
We would like to prove the second inequality in \eqref{eq:claim-Qt-31-UB-async} by induction. 
Clearly, it suffices to look at those iterations where the value of $Q_t(3, 2)$ changes, namely, $\{t_k(3, 2)\}_{k\geq 1}$. 
Assume for the moment that \eqref{eq:claim-Qt-31-UB-async} holds for all $t < t_k(3, 2)$.  
Taking $\overline{\bm{V}} \coloneqq \big[\frac{1}{1-\gamma}, \frac{1}{1-\gamma}, \frac{1}{1-\gamma}, V^{\star}(3)\big]^{\top}$, 
we can invoke the asynchronous Q-learning update rule iteratively to deduce that
\begin{align}
	Q_{t_k(3,2)}(3, 2) &= (1 - \eta)^kQ_{0}(3, 2) + \sum_{i = 1}^k\gamma\eta(1-\eta)^{k - i}
	\bm{P}_{t_i(3,2)}\big((3,2),\cdot\big) \bm{V}_{t_i-1} \notag\\
	&\le \sum_{i = 1}^k\gamma\eta(1-\eta)^{k - i}\bm{P}_{t_i(3,2)}\big((3,2),\cdot\big) \overline{\bm{V}} \notag\\
	&\le \gamma \bm{P}(\cdot \mymid 3, 2)\overline{\bm{V}} + \sum_{i = 1}^k\gamma\eta(1-\eta)^{k - i}
	\big[\bm{P}_{t_i(3,2)}\big((3,2),\cdot\big) -  \bm{P}(\cdot \mymid 3, 2) \big] \overline{\bm{V}}    \notag\\
	&< \gamma \bm{P}(\cdot \mymid 3, 2)\overline{\bm{V}} + \frac{1}{12}    \notag\\ 	
	&\le V^{\star}(3) - \frac{1}{6}. 
	\label{eq:last-eq-LB-1579}
\end{align}
Here, the second line holds since $\overline{\bm{V}}$  serves as an entrywise upper bound on $\bm{V}_t$; 
the third line follows since $\sum_{i = 1}^k \eta(1-\eta)^{k - i}\leq 1$;    
the fourth line in \eqref{eq:last-eq-LB-1579} is valid since,   
according to the Bernstein inequality (see \citet[Lemma 1]{li2020sample}), 
\begin{align*}
	\bigg| \sum_{i = 1}^k \eta(1-\eta)^{k - i} \big[\bm{P}_{t_i(3,2)}\big((3,2),\cdot\big) -  \bm{P}(\cdot \mymid 3, 2) \big] \overline{\bm{V}} \bigg| 
	< \frac{1}{12}
\end{align*}
holds with probability at least $1 - 1/T$; 
and the validity of the last inequality can be shown by observing that
\begin{align*}
\gamma\bm{P}(\cdot\mymid3,2)\overline{\bm{V}} & =\gamma P(3\mymid3,2)\overline{V}(3)+\frac{\gamma}{1-\gamma}\big(1-P(3\mymid3,2)\big)\\
 & =\gamma\Big(1-2(1-\gamma)\Big)\overline{V}(3)+\gamma\big(2(1-\gamma)\mu_{3}\big)\overline{V}(3)+\frac{\gamma}{1-\gamma}\big(2(1-\gamma)\mu_{3}\big)\\
 & =\gamma\Big(1-2(1-\gamma)\Big)\overline{V}(3)+3.5\gamma\mu_{3}\le\gamma\big(1-2(1-\gamma)\big)\overline{V}(3)+2\gamma\\
 & =\overline{V}(3)-\frac{3}{4}(2\gamma+1)+2\gamma\le\overline{V}(3)-\frac{1}{4},
\end{align*}
where we have used the facts that $\overline{V}(3)=V^{\star}(3) = \frac{3}{4(1-\gamma)}$, 
$P(3\mymid3,2)= 1- 2(1-\gamma)(1-\mu_3)$  
and $\gamma <1$. 
Thus, standard induction arguments immediately validate the second inequality in the claim \eqref{eq:claim-Qt-31-UB-async} for all $t\leq T$.

Regarding the first inequality of \eqref{eq:claim-Qt-31-UB-async}, it is seen that for any $k\geq 1$, 
\begin{align*}
	Q_{t_k(3,1)}(3, 1) &= (1 - \eta)^kQ_{0}(3, 1) + \sum_{i = 1}^k\eta(1-\eta)^{k - i}\Big( r(3, 1) + 
	\gamma \bm{P}_{t_i(3,1)}\big((3,1),\cdot\big)\bm{V}_{t_i(3,1)-1} \Big) \\
&\le \sum_{i = 1}^k\eta(1-\eta)^{k - i}\big( r(3, 1) + \gamma\overline{V}(3) \big) \le V^{\star}(3), 
\end{align*}
where we have used the facts that $Q_0\equiv 0$ and $\overline{V}(3)=V^{\star}(3)=\frac{3}{4(1-\gamma)}$ 
and the elementary inequality $\sum_{i = 1}^k \eta(1-\eta)^{k - i}\leq 1$.  
Given that the above bound holds for all $k\geq 1$ (and $\{t_k(3,1)\}_{k\geq 1}$ correspond to all iterations when the value of $Q_{t}(3, 1)$ changes), we have established the first advertised inequality in \eqref{eq:claim-Qt-31-UB-async}. 
\end{proof}

Next, let us define
\begin{align}
s \coloneqq \max\Big\{ k : Q_{t_k(3, 1)}(3, 1) < V^{\star}(3) - \frac{1}{8} \Big\}. 
	\label{eq:defn-s-Qtk-1357}
\end{align}
From this definition and Claim~\eqref{eq:claim-Qt-31-UB-async}, we know that for any $k>s$, it holds that
\[
	Q_{t_k(3, 1)}(3, 1) \geq V^{\star}(3) - \frac{1}{8} .
\]
Recognizing that
\begin{align*}
\big|Q_{t_k(s, a)}(s, a) - Q_{t_{k-1}(s, a)}(s, a)\big|
	= \eta \big|Q_{t_k(s, a)}(s, a) - r(s,a) - \gamma \bm{P}_{t_k(s, a)}\big((s, a),\cdot\big) \bm{V}_{t_{k-1}(s,a)} \big| 
	\le \frac{\eta}{1-\gamma} < \frac{1}{24},
\end{align*}
we can readily obtain 
\[
Q_{t_{k-1}(3,1)}(3,1)\geq Q_{t_{k}(3,1)}(3,1)-\frac{1}{24}\geq V^{\star}(3)-\frac{1}{6}>Q_{t_{k-1}(3,1)}(3,2),
\]
\[
	\Longrightarrow \qquad V_{t_{k-1}(3, 1)}(3) = Q_{t_{k-1}(3, 1)}(3, 1), \qquad \forall k > s. 
\]
Combine this with the Q-learning update rule to arrive at
\begin{align*}
Q_{t_k(3,1)}(3, 1) &= (1 - \eta)Q_{t_{k-1}(3,1)}(3, 1) + \eta\big( r(3, 1) + \gamma Q_{t_{k-1}(3,1)}(3, 1) \big) \\
&= \big(1 - \eta(1-\gamma)\big) Q_{t_{k-1}(3,1)}(3, 1) + \frac{3}{4}\eta \\
	&= \cdots = \big(1 - \eta(1-\gamma)\big)^{k-s}Q_{t_{s}(3,1)}(3, 1) + \frac{3}{4}\eta\sum_{i = 0}^{k-s-1} \big(1 - \eta(1-\gamma)\big)^i \\
	&= V^{\star}(3) - \big(1 - \eta(1-\gamma)\big)^{k-s}\big[V^{\star}(3) - Q_{t_{s}(3,1)}(3, 1) \big], 
\end{align*}
where the last inequality holds since $V^{\star}(3) = \frac{3}{4(1-\gamma)} $. 
This taken together with \eqref{eq:defn-s-Qtk-1357} (so that $V^{\star}(3) - Q_{t_{s}(3,1)}(3, 1)\geq 1/8$) leads to
\begin{align}
V^{\star}(3) - Q_{T}(3, 1) = Q^{\star}(3, 1) - Q_{T}(3, 1) \geq \frac{1}{8}\big(1 - \eta(1-\gamma)\big)^{N_T(3, 1)} 
	\geq c_9 
	\label{eq:V3-Q31-gap-1579}
\end{align}
for some constant $c_9>0$;  
here, the last inequality holds since, according to \citet[Lemma 8]{li2020sample},  
\begin{align*}
	N_T(3, 1) \leq \frac{3}{2}\mu_{\mathsf{b}}(3, 1) T = \frac{9}{4}\mumin T \le O\Big( \frac{1}{\eta(1-\gamma)} \Big)
\end{align*}
occurs with probability at least $1-1/T$, provided that
$T\geq \frac{443\tmix \log(10T)}{\mumin}$.  
Note that according to \eqref{eq:construction-mumin} and  \eqref{eq:mixing-time-asyn-Q-LB-135}, 
$\frac{\tmix }{\mumin}$ is on the order of $\frac{\log^2T}{c_{\mu}(1-\gamma)}$.

Putting \eqref{eq:claim-Qt-31-UB-async} and \eqref{eq:V3-Q31-gap-1579} together then reveals that 
with probability at least $1-2/T$, 
\begin{equation} \label{eq:case1}
	\max_a \big| Q^{\star}(3,a) - Q_{T}(3,a) \big|
	\geq V^{\star}(3) - V_{T}(3) 
	= V^{\star}(3) - \max\big\{ Q_{T}(3, 1), Q_{T}(3, 2) \big\} 
	\geq \min\Big\{ \frac{1}{6}, c_9\Big\}. 
\end{equation}

\subsubsection{Case 2: large learning rates ($\eta > \frac{\log T}{\mumin(1-\gamma)^2T}$)}

\paragraph{Case 2.1: $\eta \geq \frac{c_8(1-\gamma)^2}{\log T}$ for some small enough constant $c_8>0$.} 
Under the condition that $\eta \geq \frac{c_8(1-\gamma)^2}{\log T}$, 
we claim that with probability at least $\frac{1-\gamma}{50}$, 
%
\begin{equation}
	\exists (s,a)\qquad \text{s.t.}\quad \big|Q_T(s, a) - Q^{\star}(s, a)\big| \geq c_{10}\frac{1-\gamma}{\log T}
	\label{eq:claim-Q-gap-2489}
\end{equation}
for some universal constant $c_{10}>0$. 
We shall first prove this claim.

\begin{proof}[Proof of Claim~\eqref{eq:claim-Q-gap-2489}]
We shall focus attention on the case where $(s_{T-1}, a_{T-1}) = (2, 1)$.  
Given that the stationary distribution obeys $\mu_{\mathsf{b}}(2,1)=1/10$ and that $T$ is sufficiently large (so that the empirical distribution approaches the stationary distribution),  
we know that 
\[
	\mathbb{P}\big( (s_{T-1}, a_{T-1}) = (2, 1) \big) \geq \mu_{\mathsf{b}}(2,1)/2= 1/20.
\]
%

Let us first look at the case where $\big|V_{T-1}(0) - V^{\star}(0) \big| > \frac{1}{27(1-\gamma)}$. It follows from $P(\cdot\mymid 0, 1)=[1,0,0,0]$ that
\[
	\big|Q_{T}(0,1) - Q^{\star}(0,1) \big| = 
	\big|r(0,1) + \gamma V_{T-1}(0) - V^{\star}(0) \big|
	= \gamma \big| V_{T-1}(0) - V^{\star}(0) \big| > \frac{0.03}{1-\gamma}
\]
as long as $\gamma \geq 0.95$, 
which clearly satisfies \eqref{eq:claim-Q-gap-2489}. 
When it comes to the complement case where $\big|V_{T-1}(0) - V^{\star}(0) \big| \le \frac{1}{27(1-\gamma)}$, 
 either of the following two scenarios will happen: 
 \begin{itemize}
		\item If $|V_{T-1}(2) - V_{T-1}(0)| \le \frac{1}{27(1-\gamma)}$, then the assumption  $\big|V_{T-1}(0) - V^{\star}(0) \big| \le \frac{1}{27(1-\gamma)}$ yields
\begin{align}
Q_T(2, 2) &\le V_{T-1}(2) \le V_{T-1}(0) + \frac{1}{27(1-\gamma)}
	\le V^{\star}(0) + \frac{2}{27(1-\gamma)} \notag\\
&\le \gamma \big[1 - 2(1-\gamma)(\mu_1 + \mu_3) \big]V^{\star}(2) - \frac{1}{27(1-\gamma)} < Q^{\star}(2, 2) - \frac{1}{27(1-\gamma)},
	\label{eq:QT-22-UB-Qstar-22-135}
\end{align}
where the second line holds since, for $\gamma \ge 0.95$,
\begin{align*}
\gamma \big[1 - 2(1-\gamma)(\mu_1 + \mu_3) \big]V^{\star}(2) &= \gamma\Big[1-\frac{4}{5}(1-\gamma)\Big]\frac{15+8\gamma}{3(5+4\gamma)(1-\gamma)} \\
	&\ge \frac{7}{9(1-\gamma)} 
	= V^{\star}(0) + \frac{1}{9(1-\gamma)},
\end{align*}
and the last inequality holds since $Q^{\star}(2,2) > r(2,2)+\gamma P(2\mymid 2,2) V^{\star}(2)$ (from the Bellman equation). 
		 
	 \item  Consider instead the scenario with $|V_{T-1}(2) - V_{T-1}(0)| > \frac{1}{27(1-\gamma)}$.  
		 For notational convenience, define 
		 $$
		 	Q_T^s(2, 1) \coloneqq (1 - \eta)Q_{T-1}(2, 1) + \eta\big(r(2, 1) + \gamma V_{T-1}(s) \big),
			\quad \forall s\in \cS. 
		 $$
		 %
		 Recognizing that $\min\big\{ P(0\mymid 2,1), P(2\mymid 2,1) \big\} \geq 2(1-\gamma)\mu_0$,  
		 we can show that
		 with probability at least $(1-\gamma)\mu_0=\frac{2(1-\gamma)}{5}$, 
\begin{align*}
\big|Q_T(2, 1) - Q^{\star}(2, 1)\big| 
	&= \max\Big\{\big| Q_T^0(2, 1) - Q^{\star}(2, 1) \big|,\, \big|Q_T^2(2, 1) - Q^{\star}(2, 1)\big| \Big\} \\
&\ge \frac{1}{2} \big| Q_T^0(2, 1) - Q_T^2(2, 1) \big| 
	= \frac{1}{2}\eta\gamma \big|V_{T-1}(0) - V_{T-1}(2) \big| \\
	& > \frac{\eta\gamma}{54(1-\gamma)}  \geq \frac{0.95c_8(1-\gamma)}{54\log T}, 
\end{align*}
where the last inequality holds since  $\eta \geq \frac{c_8(1-\gamma)^2}{\log T}$  and $\gamma \geq 0.95$. 
 \end{itemize}
We can thus conclude that for all the above scenarios, the claim \eqref{eq:claim-Q-gap-2489} holds with probability at least 
$\frac{1-\gamma}{50}$. 
\end{proof}

\paragraph{Case 2.2: $\frac{\log T}{\mumin(1-\gamma)^2T} < \eta \leq \frac{c_8(1-\gamma)^2}{\log T}$.}
Recall from Claim~\eqref{eq:claim-Qt-31-UB-async} that: when $\eta \leq \frac{c_8(1-\gamma)^2}{\log T}$, one has
\begin{align}
Q_t(3, 1) \le V^{\star}(3)\qquad\text{and}\qquad Q_t(3, 2) \le V^{\star}(3) - \frac{1}{6},
	\qquad \forall t\leq T.
	\label{eq:claim-Qt-31-UB-async-repeat}
\end{align}
In addition, for any $k \geq \frac{\log T}{\eta(1-\gamma)}$, we can use $P(\cdot\mymid3,1)=[0,0,0,1]$ and the Bellman equation to derive
\begin{align*}
Q_{t_k(3, 1)}(3, 1) &= (1 - \eta)Q_{t_{k-1}(3, 1)}(3, 1) + \eta\big( r(3, 1) + \gamma V_{t_k(3, 1)-1}(3) \big) \\
	&\geq (1 - \eta)Q_{t_{k-1}(3, 1)}(3, 1) + \eta\Big( \frac{3}{4} + \gamma Q_{t_{k-1}(3, 1)}(3, 1) \Big) \\ 
&= \big(1 - \eta(1-\gamma) \big)Q_{t_{k-1}}(3, 1) + \frac{3}{4}\eta \\
&= \big(1 - \eta(1-\gamma)\big)^kQ_{0}(3, 1) + \frac{3}{4}\sum_{i = 1}^k\eta \big(1-\eta(1-\gamma)\big)^{k - i} \\
&= \Big[1 - \big(1 - \eta(1-\gamma)\big)^k \Big]V^{\star}(3) \ge V^{\star}(3) - \frac{1}{T(1-\gamma)}, 
\end{align*}
where the last line is valid since $V^{\star}(3)=\frac{3}{4(1-\gamma)}$ and $Q_0(3,1)=0$. 
It is also seen from \citet[Lemma 8]{li2020sample} that,   with probability at least $1-1/T$, 
\begin{align*}
	N_{T/3}(3, 1) \geq \frac{1}{2}\mu_{\mathsf{b}}(3, 1) T = \frac{3}{4}\mumin T 
	\geq \frac{\log T}{\eta(1-\gamma)} 
\end{align*}
as long as $\eta \ge \frac{\log T}{\mumin(1-\gamma)T}$. 
The above two results taken collectively yield  
\begin{align}
0 \le V^{\star}(3) - V_t(3) \leq V^{\star}(3)-Q_{t}(3,1) \le \frac{1}{T(1-\gamma)},\qquad\text{for all }\frac{T}{3} \le t \le T. 
\end{align}
%

Next, we move on to analyze state $2$. 
Towards this end, let us define
\begin{align}
s \coloneqq \max\Big\{ k : Q_{t_k(2, 1)}(2, 2) > Q_{t_k(2, 1)}(2, 1) - \frac{1}{2} ~~\text{or}~~ \big|Q_{t_k(2, 1)}(2, 1) - V^{\star}(2)\big| > \frac{1}{4} \Big\}.
	\label{eq:defn-s-case2-async-Q-LB}
\end{align}
Note that when $\eta \leq \frac{c_8(1-\gamma)^2}{\log T}$, one has
\begin{align*}
	\big|Q_{t_k(s,a)}(s, a) - Q_{t_{k-1}(s,a)}(s, a)\big| 
	= \eta \big|Q_{t_k(s,a)}(s, a) - r(s, a) - \gamma \bm{P}_{t_{k}(s,a)}\big((s,a),\cdot\big)\bm{V}_{t_k(s,a)-1} \big| 
	\le \frac{\eta}{1-\gamma} < \frac{1}{4},
\end{align*}
which combined with \eqref{eq:defn-s-case2-async-Q-LB} implies that $Q_{t_{k-1}(2,1)}(2, 2) < Q_{t_{k-1}(2,1)}(2, 1)$ for any $k>s$ and hence
\begin{align} \label{eq:V-Q-2}
	V_{t_{k-1}(2,1)}(2) = Q_{t_{k-1}(2,1)}(2, 1) \qquad \text{ for } k > s.
\end{align}
This crucial identity together with the construction of $P(\cdot\mymid 2,1)$ in turn allows one to derive, for any $k>s$,
\begin{align}
	Q_{t_k(2, 1)}(2, 1)&= (1 - \eta)Q_{t_{k-1}(2, 1)}(2, 1) + \eta\Big(r(2, 1) + \gamma \bm{P}_{t_k(2, 1)} \big((2,1),\cdot\big) 
	\bm{V}_{t_{k}(2,1)-1}\Big) \notag\\
&= \big(1 - \eta(1-\gamma)(1+2\gamma\mu_0)\big)Q_{t_{k-1}(2, 1)}(2, 1) \notag\\
& \qquad + \eta\Big(1 + 2\gamma(1-\gamma)\mu_0 V_{t_{k}(2, 1)-1}(0) + \gamma \Big(\bm{P}_{t_k(2, 1)} \big((2,1),\cdot\big) - \bm{P}(\cdot\mymid 2, 1)\Big)\bm{V}_{t_{k}(2, 1)-1}\Big) \notag\\
&= \big(1 - \eta(1-\gamma)(1+2\gamma\mu_0)\big)^{k-s}Q_{t_{s}(2, 1)}(2, 1) 
	+ \eta\sum_{i = s+1}^{k} \big(1 - \eta(1-\gamma)(1+2\gamma\mu_0)\big)^{k-i} \notag\\
	&\qquad\cdot\Big(1 + 2\gamma(1-\gamma)\mu_0V_{t_{i}(2, 1)-1}(0) + \gamma \Big(\bm{P}_{t_i(2,1)}\big((2,1),\cdot\big) - \bm{P}(\cdot\mymid 2, 1)\Big)\bm{V}_{t_{i}(2, 1)-1}\Big), \label{eq:expression-Qtk-21-async} 
\end{align}
%
thus leading to
%
\begin{align*}
\mathbb{E}\big[Q_{t_k(2, 1)}(2, 1) \mid Q_{t_s(2, 1)}(2, 1) \big] &\le \big(1 - \eta(1-\gamma)(1+2\gamma\mu_0)\big)^{k-s}Q_{t_s(2, 1)}(2, 1) \\
&\qquad+ \eta\sum_{i = s+1}^{k} \big(1 - \eta(1-\gamma)(1+2\gamma\mu_0)\big)^{k-i}\big(1 + 2\gamma(1-\gamma)\mu_0V^{\star}(0)\big) \\
&\le \big(1 - \eta(1-\gamma)(1+2\gamma\mu_0)\big)^{k-s}Q_{t_s(2, 1)}(2, 1) \\
&\qquad+ \eta\sum_{i = s+1}^{k} \big(1 - \eta(1-\gamma)(1+2\gamma\mu_0)\big)^{k-i} V^{\star}(2) \\
&= V^{\star}(2) - \big(1 - \eta(1-\gamma)(1+2\gamma\mu_0)\big)^{k-s}\big(V^{\star}(2) - Q_{t_s(2, 1)}(2, 1) \big) \\
&\le V^{\star}(2) - \frac{1}{4}\big(1 - \eta(1-\gamma)(1+2\gamma\mu_0)\big)^{k-s},
\end{align*}
where the second inequality arises from the Bellman equation (so that $V^{\star}(2)\geq 1 + 2\gamma(1-\gamma)\mu_0V^{\star}(0)$). 
\begin{itemize}
	\item If $\big(1 - \eta(1-\gamma)(1+2\gamma\mu_0)\big)^{k-s} \ge \frac{1}{2}$, then we can readily see that
\begin{align}
\mathbb{E}\Big[\big|Q_{T}(2, 1) - Q^{\star}(2, 1)\big|^2\Big] \ge \Big(\mathbb{E}\big[Q_{t_k(2, 1)}(2, 1) - V^{\star}(2) \mid Q_{t_s(2, 1)}(2, 1) \big]\Big)^2 \ge \frac{1}{64}.
\end{align}
\item
Otherwise, consider the case where $\big(1 - \eta(1-\gamma)(1+2\gamma\mu_0)\big)^{k-s} < \frac{1}{2}$. 
The expression \eqref{eq:expression-Qtk-21-async} allows one to control the variance as follows: 
\begin{align*}
 &\mathsf{Var}\big(Q_{t_{k}(2,1)}(2,1)\mid Q_{t_{s}(2,1)}(2,1)\big) =\eta^{2}\gamma^{2}\sum_{i=s+1}^{k}\big(1-\eta(1-\gamma)(1+2\gamma\mu_{0})\big)^{2(k-i)}\\
 & \qquad\qquad\qquad\cdot\mathbb{E}\left[\mathbb{E}\bigg[\Big(\Big(\bm{P}_{t_{i}(2,1)}\big((2,1),\cdot\big)-\bm{P}(\cdot\mymid2,1)\Big)\bm{V}_{t_{i}(2,1)-1}\Big)^{2}\mid\bm{V}_{t_{i}(2,1)-1}\bigg]\right]\\
 & \qquad=\eta^{2}\gamma^{2}\sum_{i=s+1}^{k}\big(1-\eta(1-\gamma)(1+2\gamma\mu_{0})\big)^{2(k-i)}\\
 & \qquad\qquad\cdot\mathbb{E}\left[\mathbb{E}\bigg[2(1-\gamma)\mu_{0}\big(1-2(1-\gamma)\mu_{0}\big)\Big\{ V_{t_{i}(2,1)-1}(2)-V_{t_{i}(2,1)-1}(0)\Big\}^{2}\mid\bm{V}_{t_{i}(2,1)-1}\bigg]\right]\\
 & \qquad\geq\eta^{2}\gamma^{2}\sum_{i=s+1}^{k}\big(1-\eta(1-\gamma)(1+2\gamma\mu_{0})\big)^{2(k-i)}\left\{ (1-\gamma)\mu_{0}\cdot\frac{1}{36(1-\gamma)^{2}}\right\} \\
	& \qquad= \frac{\eta^{2}\gamma^{2}\mu_0}{36(1-\gamma)}
	\frac{1-\big(1-\eta(1-\gamma)(1+2\gamma\mu_{0})\big)^{k-s}}{\eta(1-\gamma)(1+2\gamma\mu_{0})}
	\ge \frac{\eta}{200(1+2\gamma\mu_0)(1-\gamma)^2} \geq \frac{\eta}{400(1-\gamma)^2} , 
\end{align*}
where the third inequality holds since 
$V_{t_{i}(2,1)-1}(2) - V_{t_{i}(2,1)-1}(0) \ge V^{\star}(2) - \frac{1}{4} - V^{\star}(0) > \frac{5}{27(1-\gamma)}$ (using the definition of $s$ in \eqref{eq:defn-s-case2-async-Q-LB}), 
		and the last line uses $\gamma \geq 0.95$ and $\mu_0=2/5$. 
%
As a result, 
\begin{align}
\mathbb{E}\Big[\big|Q_{T}(2, 1) - Q^{\star}(2, 1)\big|^2\Big] \ge 
	\mathbb{E}\Big[ \mathsf{Var}\big(Q_{t_{k}(2,1)}(2,1)\mid Q_{t_{s}(2,1)}(2,1)\big) \Big]
	\geq \frac{\eta}{400(1-\gamma)^2}.
\end{align}
\end{itemize}

Taking the above two results together reveals that
\begin{align}
\mathbb{E}\Big[\big|Q_{T}(2, 1) - Q^{\star}(2, 1)\big|^2\Big] 
	\geq  \min \bigg\{ \frac{\eta}{400(1-\gamma)^2}, \frac{1}{64} \bigg\}
	\geq   \min \bigg\{ \frac{\log T}{400\mumin (1-\gamma)^4 T}, \frac{1}{64} \bigg\}.
	\label{eq:claim-Q-gap-248976}
\end{align}

\paragraph{Combining Case 2.1 and Case 2.2.}

Putting together \eqref{eq:claim-Q-gap-2489} and \eqref{eq:claim-Q-gap-248976} together leads to
\begin{equation} \label{eq:case2}
	\max_{s, a} \mathbb{E} \Big[ \big| Q_T(s, a) - Q^{\star}(s, a) \big|^2 \Big]  
	\geq \min \bigg\{ \frac{\log T}{400\mumin (1-\gamma)^4 T}, \frac{1}{64}, \frac{c_{10}^2(1-\gamma)^3}{50\log^2 T}\bigg\},
\end{equation}
provided that $\eta > \frac{\log T}{\mumin(1-\gamma)^2T}$.

\subsubsection{Case 3: medium learning rates ($ \frac{1}{\mumin(1-\gamma)T} \leq \eta \leq \frac{\log T}{\mumin(1-\gamma)^2T}$)}

We now shift attention to the dynamics underlying state $1$, and look at 
 its associated value function $V_{t}(1)$.  

\paragraph{Two auxiliary sequences.} 
Before proceeding, we abuse notation by defining 
\begin{align*}
t_i \coloneqq \min\Big\{ t > t_{i-1} : \text{both pairs $(1, 1)$ and $(1, 2)$ have been visited during $(t_{i-1}, t]$} \Big\},
\end{align*}
and letting $N_{(t_{k-1}, t_k]}(1, a, 1)$ (resp.~$N_{(t_{k-1}, t_k]}(1, a)$) denote the number of times the sample trajectory 
visits $(s,a,s')=(1,a,1)$ (resp.~$(s,a)=(1,a)$) 
within the time interval $(t_{k-1}, t_k]$. 
From standard Bernstein's inequality (see \citet[Lemma 8]{li2020sample}) and the definition of $t_i$, 
one can easily see that with probability at least $1-1/T$, 
\begin{align}
	1 \le N_{(t_{i-1}, t_i]}(1, a) \leq c_{11}\log T 
	\label{eq:interval-count-asyncQ}
\end{align}
holds simultaneously for all $t_i\leq T$,  where $c_{11}>0$ is some suitable constant.  
In turn,  
this bound \eqref{eq:interval-count-asyncQ} also implies that 
\begin{equation}
	K \coloneqq  \max\big\{k : t_k \le T \big\} \geq c_{12} \frac{T}{\log T}
	\label{eq:bound-K-async-Q-876}
\end{equation}
for some constant $c_{12}>0$. 
These follow from fairly standard concentration arguments and are hence omitted.

Under our assumption on $T$, the sample trajectory mixes well after $T/3$ iterations. 
In order to further remove the effect of $Q_{T/3}(1, a)$, let us introduce the following auxiliary sequence
\begin{align}
	\label{eqn:Qhat-asynQ}
	\widehat{Q}_k(a) = (1-\eta)\widehat{Q}_{k-1}(a) + \eta\Big\{ \frac{1}{3} + \gamma P_{k}(a)\widehat{V}_{k-1}\Big\},
\end{align}
where $\widehat{V}_{k-1} \coloneqq \max_a \widehat{Q}_{k-1}(a)$, 
\begin{align*}
	\widehat{Q}_0(a) &\coloneqq Q^{\star}(1, a) - V^{\star}(0) = \frac{5}{3(5+3\gamma)(1-\gamma)}, 
	 \qquad a=1,2 \\
	P_k(a) &\coloneqq \frac{N_{(t_{k-1}, t_k]}(1, a, 1)}{N_{(t_{k-1}, t_k]}(1, a)}. 
\end{align*}
Repeating the proof of \eqref{eq:Qt} (which we omit here for brevity), we arrive at the following relation:
\begin{align} 
\label{eq:Qt-asynQ}
	Q_{t_k(1,a)}(1, a) - V^{\star}(0) \ge \widehat{Q}_k(a) - \frac{1}{1-\gamma} \big(1-\eta(1-\gamma)\big)^{k} - \frac{1}{T(1-\gamma)^2},
	\qquad  a \in \{1,2\}. 
\end{align}
Furthermore, in order to control $\widehat{Q}_k(a)$, we construct an additional auxiliary sequence as follows 
\begin{align}
	\overline{Q}_k = (1-\eta)\overline{Q}_{k-1} + \eta\Big\{ \frac{1}{3} + \gamma P_k(1)\overline{Q}_{k-1}\Big\} 
	\qquad \text{and} \qquad 
	\overline{Q}_0 = V^{\star}(1) - V^{\star}(0) = \frac{5}{3(5+3\gamma)(1-\gamma)}.
	\label{eq:sequence-overline-Qt-asynQ}
\end{align}
From the basic fact that $\widehat{V}_k =\max_a \widehat{Q}_k(a) \ge \widehat{Q}_k(1)$, it can be easily verified that 
%
\begin{align}
	\widehat{Q}_k(1) \ge (1-\eta)\widehat{Q}_{k-1}(1) + \eta\Big\{ \frac{1}{3} + \gamma P_{k}(1)\widehat{Q}_{k-1}(1)\Big\}
	\ge \overline{Q}_k,
	\label{eq:ordering-Qt-hat-overline-asynQ}
\end{align}
which in turn motivates us to lower bound $\widehat{V}_k$ by controlling $\overline{Q}_k$. 
Using similar analysis for~\eqref{eqn:Tailbound-vhat-rolland}, we reach
\begin{align}
	\label{eqn:Tailbound-vhat-rolland-asynQ}
	\mathbb{P}\Big\{\widehat{V}_k \ge \frac{1}{5(1-\gamma)} \Big\} \ge \frac{1}{2},\qquad \text{for any }k\text{ with}~t_k \ge \frac{2T}{3}.
\end{align}

\paragraph{Main proof.}
With the preceding auxiliary sequences in place, let us define (akin to the synchronous case) 
\begin{subequations}
\begin{align}
	\Delta_k(a) &\defn \widehat{Q}_k(a) + V^{\star}(0) - Q^{\star}(1, a), \qquad a = 1, 2; \\
	\Delta_{k,\mathsf{max}} &\defn \max_a \Delta_k(a). 
\end{align}
\end{subequations}
Based on the iterative update rule \eqref{eqn:Qhat-asynQ}, we can once again derive
\begin{align}
\Delta_k(a) = \sum_{i = 1}^k\eta \big(1-\eta\big)^{k - i}\gamma\Big(p\Delta_{k-1,\mathsf{max}} + \big(P_{k}(a) - p \big)\widehat{V}_{k-1}\Big), \qquad a = 1, 2, 
\end{align}
where $p = 1 - \frac{3}{2}(1-\gamma)\mu_0$.
Then adopting the same analysis as for the synchronous case, we arrive at
\begin{align} \label{eq:case3}
	\mathbb{E} \Big[ \max_a \big| Q_T(1,a) - Q^{\star}(1,a) \big| \Big] \geq 
	\mathbb{E} \big[ \big| V_T(1) - V^{\star}(1) \big| \big] &\geq \frac{c_{13}}{\sqrt{\mumin(1-\gamma)^4T\log^3 T}},
\end{align}
for some constant $c_{13}>0$. 
Here, we have made use of the fact that
\begin{align*}
\big(1-\eta(1-\gamma)\big)^{K} 
	\leq  \bigg(1-\frac{3\log^2 T}{c_{\mu} T}\bigg)^{\frac{c_{12}T}{\log T}}  \le \frac{1}{T},
\end{align*}
an immediate consequence of \eqref{eq:bound-K-async-Q-876} and the assumption that $\eta(1-\gamma) \ge \frac{1}{\mumin T} = \frac{3\log^2 T}{c_{\mu} T}$.

\subsubsection{Putting all this together}
Combining~\eqref{eq:case1},~\eqref{eq:case2}, and~\eqref{eq:case3} leads to
\begin{equation} 
	\max_{s, a} \mathbb{E} \Big[ \big| Q_T(s, a) - Q^{\star}(s, a) \big|^2 \Big]  
	\geq \min \bigg\{c_9^2, \frac{\log T}{400\mumin (1-\gamma)^4 T}, \frac{1}{64}, \frac{c_{10}^2(1-\gamma)^3}{\log^2 T}, \frac{c_{13}^2}{\mumin(1-\gamma)^4T\log^3 T}\bigg\},
\end{equation}
for any $0 < \eta < 1$.
Then the conclusion is handy under the proviso that $T \ge \frac{c_3}{\mumin(1-\gamma)^7\log T}$.